\documentclass[twoside,11pt]{article}

\usepackage{blindtext}

\usepackage[preprint]{jmlr2e}

\usepackage{lastpage}
\jmlrheading{23}{2022}{1-\pageref{LastPage}}{1/21; Revised 5/22}{9/22}{21-0000}{Yuling Jiao, Yanming Lai, Huazhen Lin, Wensen Ma, Houduo Qi, Defeng Sun}

\firstpageno{1}

\usepackage{enumerate, amsmath, bm, bbm, enumitem, mathrsfs, accents}

\def\P{\mathbb{P}}
\def\E{\mathbb{E}}
\def\R{\mathbb{R}}
\def\N{\mathbb{N}}
\def\eps{\varepsilon}
\def\1{\mathbbm{1}}

\newcommand{\argmin}{\mathop{\mathrm{argmin}}}
\newcommand{\argmax}{\mathop{\mathrm{argmax}}}

\usepackage{xcolor}

\newtheorem{innercustompro}{Porposition}
\newenvironment{mypro}[1]
  {\renewcommand\theinnercustompro{#1}\innercustompro}
  {\endinnercustompro}
\newtheorem{innercustomthm}{Theorem}
\newenvironment{mythm}[1]
  {\renewcommand\theinnercustomthm{#1}\innercustomthm}
  {\endinnercustomthm}
\newtheorem{innercustomcor}{Corollary}
\newenvironment{mycor}[1]
{\renewcommand\theinnercustomcor{#1}\innercustomcor}
{\endinnercustomcor}
\begin{document}

\title{Beyond the Prompt in Large Language Models: Comprehension, In-Context Learning, and Chain-of-Thought}

\author{
    \name Yuling Jiao\textsuperscript{1,2,3,4}, Yanming Lai\textsuperscript{5}, Huazhen Lin\textsuperscript{6,7,8}, 
    Wensen Ma\textsuperscript{5}, Houduo Qi\textsuperscript{9,5}, Defeng Sun\textsuperscript{5} \\
    \addr \textsuperscript{1}School of Artificial Intelligence, Wuhan University \\
    \textsuperscript{2}National Center for Applied Mathematics in Hubei, Wuhan University \\
    \textsuperscript{3}Hubei Key Laboratory of Computational Science, Wuhan University \\
    \textsuperscript{4}Institute for Math \& AI, Wuhan University \\
    \textsuperscript{5}Department of Applied Mathematics, The Hong Kong Polytechnic University \\
    \textsuperscript{6}Center of Statistical Research, Southwestern University of Finance and Economics \\
    \textsuperscript{7}School of Statistics and Data Science, Southwestern University of Finance and Economics \\
    \textsuperscript{8}New Cornerstone Science Laboratory, Southwestern University of Finance and Economics\\
    \textsuperscript{9}Department of Data Science and Artificial Intelligence, The Hong Kong Polytechnic University
}


\maketitle

\begin{abstract}
    Large Language Models (LLMs) have demonstrated remarkable proficiency across diverse tasks, exhibiting emergent properties such as semantic prompt comprehension, In-Context Learning (ICL), and Chain-of-Thought (CoT) reasoning. Despite their empirical success, the theoretical mechanisms driving these phenomena remain poorly understood. This study dives into the foundations of these observations by addressing three critical questions: (1) \textit{How do LLMs accurately decode prompt semantics despite being trained solely on a next-token prediction objective?} (2) \textit{Through what mechanism does ICL facilitate performance gains without explicit parameter updates?} and (3) \textit{Why do intermediate reasoning steps in CoT prompting effectively unlock capabilities for complex, multi-step problems?}

    Our results demonstrate that, through the autoregressive process, LLMs are capable of exactly inferring the transition probabilities between tokens across distinct tasks using provided prompts. We show that ICL enhances performance by reducing prompt ambiguity and facilitating posterior concentration on the intended task. Furthermore, we find that CoT prompting activates the model's capacity for task decomposition, breaking complex problems into a sequence of simpler sub-tasks that the model has mastered during the pretraining phase. By comparing their individual error bounds, we provide novel theoretical insights into the statistical superiority of advanced prompt engineering techniques.
\end{abstract}


\section{Introduction}
The development of Large Language Models (LLMs) has revolutionized the paradigm of human-computer interaction \citep{radford2018improving, radford2019language, brown2020language, chowdhery2022palm, touvron2023llama, achiam2023gpt, team2024gemini, comanici2025gemini, guo2025deepseek} based on Transformer architecture~\citep{vaswani2017attention}. These tools allow users to communicate in natural language while keeping model parameters fixed, an advancement that rests entirely on the models' comprehension capabilities for users' intention. Yet, the underlying training task for LLMs is remarkably consistent: the autoregressive prediction of the next token. How such a straightforward paradigm generates complex comprehension is a central inquiry into the mystery of these emergent systems. A key obstacle in this investigation is the formidable technical challenge of analyzing the Transformer architecture itself. While the Transformer is the primary vehicle for this intelligence, its theoretical characterization remains an elusive goal, posing a significant analytical barrier to understanding how simple autoregressive objectives are scaled into sophisticated cognitive abilities.

In the absence of a complete theoretical roadmap, our understanding of these models has been primarily driven by empirical discoveries at the interface of input design. In practice, it has been observed that the strategic design of prompts can induce specific desired behaviors in LLMs without any modification to the model parameters. This paradigm is known as prompt engineering \citep{brown2020language, wei2022chain, liu2023pre, zhou2023leasttomostpromptingenablescomplex, sahoo2024systematic}. Within this field, the two most prominent heuristics are In-Context Learning (ICL) \citep{brown2020language, dong2024survey} and Chain-of-Thought (CoT) \citep{wei2022chain}.

In-Context Learning allows LLMs to infer tasks from a small number of demonstrations provided directly in the prompt context. By leveraging these input-output examples, the model generalizes to novel instances and achieves superior performance compared to standard zero-shot prompting, even though intuitively the distribution of prompts, which concatenate independent examples, is quite different from natural language for pretraining. To date, this intuitive paradigm has become a cornerstone of prompt engineering, serving as the basis for increasingly sophisticated methods tailored for complex problem-solving. However, despite these empirical gains, the theoretical mechanisms underpinning such improvements remain poorly understood.

Building upon the success of few-shot demonstrations, Chain-of-Thought (CoT) prompting extends the few-shot paradigm by requiring the model to generate a sequence of intermediate reasoning steps before arriving at a final conclusion. This strategy has proven particularly effective for complex tasks involving arithmetic, symbolic logic, and multi-step commonsense reasoning—domains where standard direct-output ICL often falls short \citep{wei2022chain, kojima2022large}. By explicitly modeling the ``process" rather than just the ``result", CoT provides the model with a logical blueprint that aids in navigating high-dimensional state spaces to find the correct solution.

However, the reason why these intermediate reasoning steps in few-shot demonstrations effectively unlock such new capabilities remains a profound theoretical mystery. It is not immediately obvious how seeing a few examples of ``reasoning" allows a model to resolve complex problems that it cannot handle via standard ICL. Specifically, what fundamental advantage does this structured prompting provide compared to the other two prompts, and what emergent abilities are activated to facilitate this transition?

In this study, we provide a foundational analysis to address these inquiries. Leveraging a rigorous theoretical framework for Transformers, we demonstrate that the autoregressive objective enables LLMs to accurately decode prompt semantics and precisely infer transition probabilities between tokens across distinct tasks. Our findings reveal that while standard ICL facilitates performance gains via posterior concentration on an intended task (by reducing prompt ambiguity), few-shot CoT prompting goes further: it activates the model’s capacity for task decomposition. This allows the model to resolve complex problems by breaking them into a sequence of simpler sub-tasks mastered during the pretraining phase. By deriving and comparing individual error bounds, we provide a rigorous theoretical account of the statistical superiority of few-shot CoT over both standard ICL and zero-shot prompting.

\subsection{Related Works}
There is significant interest in exploring the underlying mechanisms behind the improvements induced by prompt engineering; however, many fundamental questions remain underexplored.
\paragraph{ICL Theory from a Bayesian Perspective}
\citet{xie2022explanationincontextlearningimplicit} were among the first to analyze ICL through a Bayesian lens. By modeling the token generation process as a Hidden Markov Model (HMM), they demonstrated that LLMs can recover transition rules between adjacent hidden states as the number of demonstrations approaches infinity. Building on this foundation, \citet{wies2023learnability} and \citet{jiang2023latentspacetheoryemergent} established non-asymptotic bounds for the number of demonstrations required for effective ICL. However, these works rely on the strong assumption that LLMs near-perfectly approximate the underlying language distribution following pretraining—a premise that remains difficult to verify as the sample complexity of Transformers is still an active area of exploration. Furthermore, \citet{jiang2023latentspacetheoryemergent} utilized ambiguity to characterize the uncertainty of natural language; yet, their results imply that demonstrations necessarily reduce zero-shot ambiguity, a conclusion that often contradicts empirical observations. We argue that this discrepancy arises from their omission of prior imbalances within the task space. In response, we introduce a novel proof strategy that establishes a theoretical threshold for ICL, offering a more robust perspective. Additionally, while \citet{pmlr-v258-zhang25d} characterized pretraining error using the PAC-Bayes framework, their analysis primarily concerns posterior convergence for static tasks and overlooks inherent task-level uncertainty. Finally, \citet{liu2025context} adopted a regression-based setup to provide asymptotic risk bounds for ICL.

\paragraph{ICL Theory from Other Perspectives}
In addition to the Bayesian view, some studies prioritize the representational capacity of Transformers, providing existence proofs for their potential to perform ICL, rather than describing the mechanistic reality of how existing LLMs function. Specifically, \citet{garg2022can, bai2023transformers, panwar2024incontextlearningbayesianprism, hataya2024automatic, kim2024transformers, li2025transformers, shen2025understanding, ching2026efficient} demonstrated the expressive power of Transformers, revealing that Transformers combined with ICL-like inputs can universally recover signals produced by regression functions with specific structures (e.g., linear functions, two-layer neural networks).
Conversely, another strand of research investigates the training dynamics, such as \citet{von2023transformers, von2023uncovering, ahn2023transformers, dai2023can, bai2023transformers, fu2024transformers, zhang2024trained} While these studies often require modifications to the softmax activation function within the attention layer (e.g., using Identity or ReLU), or just consider the shallow Transformer, to derive theoretical guarantees, \citet{ren2024towards} analyzed the standard softmax activation to provide a theory from a dynamical perspective. A critical limitation of this line of work is that the training data is often explicitly structured to teach in-context learning (e.g., linear regression tasks), whereas real-world LLM pretraining data are not structured in this manner.

\paragraph{Theory for Chain-of-Thought (CoT)}
Unlike standard ICL, the theoretical understanding of CoT prompting remains relatively scarce. \citet{li2023dissecting} adopted a dynamic view, modeling CoT as an implicit gradient descent process that reduces the sample complexity of learning compositional functions. \citet{prystawski2023think} proved that CoT prompts can bring LLM outputs much closer to the true distribution. Furthermore, from an information-theoretic lens, \citet{ton2024understanding} claimed that reasoning paths wrongly executed by LLMs provide negligible information gain. But the most central question is: \textit{where does the emergent ability induced by CoT which is observed by us come from?} \citet{feng2023towards} attributed the success of CoT to extended computational depth (proving that bounded-depth Transformers cannot solve certain math problems without CoT), we argue that this view neglects the uncertainty involved in task selection. In contrast, we demonstrate that the primary role of CoT is to disambiguate the reasoning path, enabling the model to navigate non-stationary trajectories that are not explicitly defined in the pretraining distribution. \citet{hu2024unveiling} modeled Chain-of-Thought (CoT) as sequential regression functions, allowing them to incorporate pretraining analysis into their error bounds. However, their theory explicitly assumes that the pretraining dataset inherently contains data with a multi-step structure similar to CoT. This assumption may obscure a deeper understanding of how CoT reasoning emerges from standard pretraining.


\paragraph{Summary of Research Gap} In short, existing theories establish guarantees only for limited cases, without a direct comparison between different prompting strategies. Moreover, these theories often rely on simple assumptions that may not reflect how real-world LLMs behave. This lack of comparison, combined with a failure to mechanistically explain the emergent reasoning catalyzed by CoT, leaves the comparative landscape of prompting strategies theoretically underexplored.

\paragraph{Transformer Approximation Ability} Current research on memorization primarily focuses on classification tasks, which does not adequately capture the impact of Transformers in large language models (LLMs). To address this gap, we investigate the ability of Transformers to memorize complex probability distributions, thereby providing a powerful tool for further exploration of LLMs. Based on this result, we furthermore explore the sample complexity of current LLMs, offering a rigorous understanding from statistic perspective.

The memorization capacity of neural networks is closely related to their approximation capabilities. For classic feedforward neural networks (FNNs), studies on memorization capacity can be found in \cite{park2021provable,vardi2022on}. For research on the memorization capacity of Transformers, see \cite{kim2023provable,mahdavi2024memorization,madden2024upper,kajitsuka2025on}.

\subsection{Contributions}
The primary contributions of this study are summarized as follows:
\begin{itemize}
    \item We propose a unified framework to conduct a comprehensive analysis of prominent LLM prompting strategies, including zero-shot (Theorem~\ref{Theorem: LLMs' comprehension}), ICL (Theorem~\ref{theorem: ICL}), and CoT (Theorem~\ref{theorem: CoT}), which enable us to characterize the underlying mechanisms through which these prompts enhance model performance.
    \item Our framework offers a novel insight for the emergent abilities elicited by CoT (Theorem~\ref{theorem: CoT}). Specifically, it demonstrates how pretrained LLMs facilitate task composition, thereby explaining the emergence of the complex behaviors observed in practice.  
    \item To ensure a rigorous pretraining analysis (Theorem~\ref{Theorem: pretraining}), we establish a comprehensive theory for the Transformer architecture, covering both its generalization (Appendix~\ref{section: generlization}) and memorization capabilities (Appendix~\ref{Section: approximation}). Notably, our analysis maintains high architectural fidelity, requiring no excessive structural modifications compared to existing theoretical studies. 
\end{itemize}

\subsection{Organization}
In Section~\ref{section: preliminaries}, we introduce preliminaries and our modeling setup. Section~\ref{section: pretraining} provides a theoretical analysis of the autoregressive pretraining pipeline, where Theorem~\ref{Theorem: LLMs' comprehension} characterizes the model's ability to comprehend and recover latent tasks. We then analyze prompt engineering strategies in detail: Section~\ref{section: ICL} focuses on In-Context Learning, with Theorem~\ref{theorem: ICL} showing how such prompts exponentially mitigate task ambiguity. Finally, Section~\ref{section: CoT} discusses Chain-of-Thought prompting, demonstrating theoretically how it provides instructions for task decomposition, allowing LLMs to solve complex problems through the composition of atomic tasks.

\section{Preliminaries}\label{section: preliminaries}
We adopt $|S|$ to represent the cardinality of the set $S$. For a given integer $k \in \N$, we use $[k]$ to denote the integer set $\{1, 2, \cdots, k\}$. We denote the $i$-th element of a vector $\bm{v} \in \R^k$ as $v_i$ and adopt $\|\bm{v}\|_2 := (\sum_{i=1}^kv_i^2)^{1/2}$ to denote its $2$-norm. For a matrix $\bm{A} \in \R^{a \times b}$, we denote its $i$-th row as $\bm{A}_{i,:}$, while its $j$-th column as $\bm{A}_{:, j}$ and the $(i,j)$-th entry as $A_{i,j}$. Moreover, the notation $\bm{A}_{:, \prec j}$ represents the first $j-1$ columns of $\bm{A}$. For non-negative sequences $(a_i)$ and $(b_i)$, we write $a_i = \mathcal{O}(b_i)$ if there exists $A > 0$ and $I \in \N$ such that $a_i \leq Ab_i$ for all $i \geq I$. 

In addition, given two discrete probability distributions $\mathbb{P}$ and $\mathbb{Q}$ defined on a shared countable sample space $\Omega$, we use $D_{KL}(\mathbb{P} \, \| \, \mathbb{Q})$ to denote their Kullback--Leibler divergence:
\begin{align*}
    D_{KL}(\mathbb{P} \, \| \, \mathbb{Q}) = \sum_{x \in \Omega} \mathbb{P}(x) \log \frac{\mathbb{P}(x)}{\mathbb{Q}(x)},
\end{align*}
where $\mathbb{P}(x)$ and $\mathbb{Q}(x)$ are the probability mass functions (PMFs) of $\mathbb{P}$ and $\mathbb{Q}$, respectively. Meanwhile, the total variation distance between $\mathbb{P}$ and $\mathbb{Q}$ is defined as 
\begin{align*}
    \mathrm{TV}(\mathbb{P}, \mathbb{Q}) = \frac{1}{2} \sum_{x \in \Omega} |\mathbb{P}(x) - \mathbb{Q}(x)|,
\end{align*}
which is equivalent to $\sup_{A \subseteq \Omega} |\mathbb{P}(A) - \mathbb{Q}(A)|$.

In the context of natural language processing, we use \texttt{<SOS>}, \texttt{<EOS>}, \texttt{<PAD>} to respectively denote the representations of Start-Of-Sequence, End-Of-Sequence and Padding identifiers. 

Let $\mathcal{V} = \{\bm{t}_i\}_{i \in [|\mathcal{V}|]}$ denote the predefined vocabulary of finite cardinality, which comprises all possible tokens involved in a Large Language Model. Specifically, each $\bm{t} \in \mathcal{V}$ corresponds to an embedding vector in the space $\R^d$. The set of all valid probability simplex over $\mathcal{V}$ is denoted by $\Delta(\mathcal{V})$. Namely, this set consists of vectors $\bm{p}$ such that $p_i \geq 0$ for all $i \in [\mathcal{V}]$ and $\sum_{i \in [\mathcal{V}]}p_i = 1$. Accordingly, $\Delta^n(\mathcal{V})$ denotes the set of matrices $\bm{A} \in \R^{|\mathcal{V}| \times n}$ such that each column $\bm{A}_{:, j}$ for $j \in [n]$ represents a valid probability distribution over $\mathcal{V}$, meaning $\bm{A}_{:, j} \in \Delta(\mathcal{V})$.

\subsection{Transformer}\label{subsection: transformer}
We adopt $\sigma_S$ to represent the column-wise softmax function. Specifically, $\sigma_S(\bm{A})_{i,j}:= \exp(A_{i,j}) / \sum_{k=1}^p\exp(A_{k,j})$. The element-wise ReLU activation function are denoted by $\sigma_R$, i.e., $\sigma_R(\bm{v}):= [\max(0, v_1), \max(0, v_2). \cdots, \max(0, v_p)]^\top \in \R^p$.

\begin{definition}[Masked Self-Attention Layer]
    Given an input sequence $\bm{Z} \in \R^{m \times n}$, the output of a self-attention layer $\mathcal{F}^{(\mathrm{SA})}_l:\R^{d \times n} \to \R^{d \times n}$ at block $l \in [D]$ is calculated by
    \begin{align}
        \label{eq: self-attention layer}
        \mathcal{F}^{(\mathrm{SA})}_l(\bm{Z}) := \bm{Z} + \sum_{h=1}^H\bm{W}_{hl}^{(O)}\bm{W}_{hl}^{(V)}\bm{Z}\sigma_S\Big[\Big(\bm{W}_{hl}^{(K)}\bm{Z}\Big)^\top\Big(\bm{W}_{hl}^{(Q)}\bm{Z}\Big) + \bm{M}\Big] \in \R^{d \times n},
    \end{align}
    where $\bm{W}_{hl}^{(V)}, \bm{W}_{hl}^{(K)}, \bm{W}_{hl}^{(Q)} \in \R^{s \times d}$ and $\bm{W}_{hl}^{(O)} \in \R^{d \times s}$ are value, key, query and projection matrices at head $h \in [H]$ with head size $s$, respectively, while $\bm{M} \in \R^{n \times n}$ is a masking matrix, which is defined as
    \begin{equation*}
        \bm{M} = \begin{pmatrix}
            0 & 0 & \cdots & 0 \\
            -\infty & 0 & \cdots & 0 \\
            \vdots & \vdots & \ddots \\
            -\infty & -\infty & \cdots & 0
        \end{pmatrix} \in \R^{n \times n}.
    \end{equation*}
\end{definition}

\begin{definition}[Feed-Forward Layer]
    The output $\bm{H} \in \R^{d \times n}$ of the masked self-attention layer at block $l$ is then passed to the feed-forward layer, which performs the following token-wise operation:
    \begin{align*}
        \mathcal{F}^{(\mathrm{FF})}_l(\bm{H})_{:,k} := \bm{H}_{:,k} + \bm{W}^{(2)}_l \sigma_R\Big(\bm{W}_l^{(1)}\bm{H}_{:,k} + \bm{b}_l^{(1)}\Big) + \bm{b}_l^{(2)} \in \R^{d}
    \end{align*}
    where $k \in [n], \bm{W}_l^{(1)} \in \R^{r \times d}$ and $\bm{W}_l^{(2)} \in \R^{d \times r}$ are weight matrices with hidden dimension $r$, and $\bm{b}_l^{(1)} \in \R^r$ and $\bm{b}_l^{(2)} \in \R^d$ are bias terms.
\end{definition}

Similar to established theoretical works \citep{kim2023provable, KajitsukaS24, kajitsuka2025on}, we adopt the following definition for the Transformer class:
\begin{definition}[Transformers]\label{def: transformers}
    We define the Transformer block $\mathcal{F}_l : \R^{d \times n} \to \R^{d \times n}$ at block $l \in [D]$ is defined as a composition of these two layers, that is, $\mathcal{F}_l := \mathcal{F}_l^{(\mathrm{FF})} \circ \mathcal{F}_l^{(\mathrm{SA})}$, and the whole architecture of the Transformer $\mathcal{T}: \R^{d \times n} \to \Delta^n(\mathcal{V})$ is expressed by
    \begin{align}
        \label{eq: Transformer}
        \mathcal{T} := \sigma_S\circ\mathcal{E}_{\mathrm{out}} \circ \mathcal{F}_D \circ \cdots \circ \mathcal{F}_1 \circ \mathcal{E}_{\mathrm{in}}
    \end{align}
    where $\mathcal{E}_{\mathrm{in}}: \R^{d \times n} \to \R^{d \times n}$ is defined as $\mathcal{E}_{\mathrm{in}}(\bm{X}) = \bm{X} + \bm{P}$, where $\bm{P} \in \R^{d \times n}$ is positional encoder and $\mathcal{E}_{\mathrm{out}}:\R^{d \times n} \to \R^{|\mathcal{V}| \times n}$ are token-wise linear mappings. By Following \citet{kim2023provable, KajitsukaS24, kajitsuka2025on}, we define the width of the Transformer model as $W := \max(d, sH, r)$ and its depth as the number of block $D$. Finally, we define $\mathcal{T}(W, D)$ as the collection of Transformers of width $W$ and depth $D$.
\end{definition}

\section{Setup}\label{Section: Setup}
Let $D_N = \{\bm{d}^{(i)}\}_{i \in [N]}$ denote a pretraining dataset for a Large Language Model (LLM) consisting of $N$ documents. Each document $\bm{d}^{(i)} \in \mathbb{R}^{d \times n}$ is represented as a sequence of $n$ tokens: $\bm{d}^{(i)} = (\bm{t}_1^{(i)}, \bm{t}_2^{(i)}, \dots, \bm{t}_n^{(i)})$, where each token $\bm{t}^{(i)}_j \in \mathbb{R}^d$ is an embedding vector in a $d$-dimensional Euclidean space.

We denote the underlying generative distribution of these documents as $q(\bm{d})$. Drawing upon the theoretical frameworks of \citet{xie2022explanationincontextlearningimplicit} and \citet{jiang2023latentspacetheoryemergent}, we model the document generation as a two-step hierarchical latent variable process:

\paragraph{Latent Task Sampling:} First, a latent variable $\theta$ is sampled from an inaccessible prior distribution $q(\theta)$, which is a probability measure defined over the support set $\Theta$.

\paragraph{Document Generation:} The latent variable $\theta$ then determines the conditional distribution from which the document is sampled, $q(\bm{d} \mid \theta)$. This process unfolds sequentially:
\begin{enumerate}
    \item The initial token is sampled from $q(\bm{t}_1 \mid \texttt{<SOS>}, \theta)$.
    \item Subsequent tokens are sampled from $q(\bm{t}_{j+1} \mid \bm{h}_j, \theta)$, where $\bm{h}_j \in \mathbb{R}^{d \times n}$ denotes the historical context up to index $j$. To maintain a consistent matrix dimension, the context is represented as: 
    \begin{align*} 
        \bm{h}_j= (\bm{t}_1, \dots, \bm{t}_j, \texttt{<PAD>}, \dots, \texttt{<PAD>}). 
    \end{align*}
    \item The process terminates when an \texttt{<EOS>} token is sampled. If the resulting sequence length is less than the maximum capacity $n$, the remainder of the matrix is filled with \texttt{<PAD>} identifiers to ensure that $\bm{d} \in \mathbb{R}^{d \times n}$.    
\end{enumerate}
We assume throughout this work that all documents adhere to a maximum length constraint of $n$ tokens. Formally, this implies that the \texttt{<EOS>} token is guaranteed to appear within the first $n$ positions of any sequence. Under this assumption, the marginal document distribution $q(\bm{d})$ is obtained by marginalizing over the latent variable $\theta$:
\begin{align*}
    q(\bm{d}) := \sum_{\theta \in \Theta} \prod_{j=1}^{n-1}q(\bm{t}_{j+1} \mid \bm{h}_j, \,\theta)\,q(\theta).
\end{align*}

Intuitively, the latent variable $\theta$ represents the underlying task or intent governing the generation process. For instance, given the prefix ``Albert Einstein was", various continuations, such as ``German", ``a physicist", or ``wise", are all statistically plausible. Without additional context, determining the ``optimal" token is ill-posed. However, if the latent task is specified as ``identify profession", the conditional probability of sampling ``physicist" will dominate other candidates. In this regard, a fundamental metric for evaluating the comprehension of LLMs is \textit{latent task identifiability: the model's capacity to precisely infer the latent variable $\theta$ solely from the provided prompt $\bm{x}$ and historical context $\bm{h}$.} To formally analyze this capability of LLMs, we define the genuine length of a given sequence as follows:

\begin{definition}[Genuine Length]
    We define the genuine length function $\mathcal{\ell} : \R^{d \times n} \to [n]$ to quantify the number of non-\texttt{<PAD>} tokens in a sequences $\bm{s} \in \R^{d \times n}$.

    Formally, the genuine length is the cardinality of the index set of valid tokens:
    \begin{align*}
        \mathcal{\ell}(\bm{s}) := \Big|\Big\{j \in [n] \,\Big|\, \bm{s}_{:, j} \neq \texttt{<PAD>}\Big\}\Big|.
    \end{align*}
\end{definition}

For notational convenience, consider a series of sequences $\bm{s}^{(1)}, \bm{s}^{(2)}, \dots, \bm{s}^{(k)} \in \R^{d \times n}$ such that the sum of their genuine lengths does not exceed $n$, i.e., $\sum_{i=1}^k \ell(\bm{s}^{(i)}) \leq n$. Let $n_i := \ell(\bm{s}^{(i)})$ denote the genuine length of each sequence $\bm{s}^{(i)}$ for $i \in [k]$. We define the concatenation operator $\circ$ to represent the sequential merging of the non-\texttt{<PAD>} segments of these contexts, followed by terminal padding to maintain a fixed sequence length. Formally:
\begin{align*}
    \bm{s}^{(1)} \circ \bm{s}^{(2)} \circ \cdots \circ \bm{s}^{(k)} = (\bm{s}^{(1)}_{1:n_1}, \cdots, \bm{s}^{(k)}_{1:n_k}, \underbrace{\texttt{<PAD>},  \cdots ,\texttt{<PAD>}}_{n - \sum_{i=1}^kn_i \text{ terms}}) \in \R^{d \times n}.
\end{align*}
We emphasize that the resulting document produced by the $\circ$ operation is always a matrix of dimension $d \times n$, ensuring it remains within the valid input space of the LLM.

In this framework, we adopt the following metrics to evaluate the quality of a given prompt $\bm{x} \in \mathbb{R}^{d \times n}$
\begin{definition}[Dominated Inferred Task and Ambiguity]\label{def: dominated task and ambiguity (non-history)}
    We define the Dominated Inferred Task $\theta_{\bm{x}}$ as the most probable latent task conditioned on the given prompt $\bm{x} \in \mathbb{R}^{d \times n}$:
\begin{align*}
    \theta_{\bm{x}} := \argmax_{\theta \in \Theta} q(\theta \mid \bm{x}).
\end{align*}
The \textbf{task ambiguity} induced by the prompt $\bm{x}$ is defined as: 
\begin{align*}
    \mathcal{A}_{\Theta}(\bm{x}) := 1- q(\theta_{\bm{x}} \mid \bm{x}).
\end{align*}
\end{definition}
The quantity $\mathcal{A}_{\Theta}(\bm{x})$ quantifies the uncertainty inherent in inferring the latent task. Intuitively, a smaller $\mathcal{A}_{\Theta}(\bm{x})$ indicates that the prompt $\bm{x}$ is significantly concise, as the posterior probability mass becomes highly concentrated on a single, dominant task. Moreover, these definitions lead to the following properties:

\begin{proposition}\label{proposition: ε2 < ε1 (without h)}
    For any given prompt $\bm{x} \in \R^{d \times n}$, if $\Theta_1 \subseteq \Theta_2$, then 
    \begin{align*}
        \mathcal{A}_{\Theta_2}(\bm{x}) \leq \mathcal{A}_{\Theta_1}(\bm{x}).
    \end{align*}
\end{proposition}
This property implies that a latent variable $q(\theta)$ with a larger support set reduces the ambiguity induced by the same document $\bm{x}$, as a broader task space provides more candidates for the dominant task.   

\begin{proposition}\label{proposition: ε of the compositional prompt (without h)}
Let $\bm{x}_1 \in \R^{d \times n}$ be any prompt, then
\begin{align*}
    \E_{\bm{x}_2\sim q(\cdot\,\mid\, \bm{x}_1)}\big\{\mathcal{A}_\Theta(\bm{x}_1\circ \bm{x}_2)\big\}\ \le\ \mathcal{A}_\Theta(\bm{x}_1).
\end{align*}
Furthermore, if $\bm{x}_2$ is another prompt such that $\theta_{\bm{x}_1} = \theta_{\bm{x}_1 \, \circ \,\bm{x}_2}$ and $q(\bm{x}_2\mid \bm{x}_1,\,\theta_{\bm{x}_1})\ \ge\ q(\bm{x}_2\mid \bm{x}_1,\,\theta)$ for all $\theta\in\Theta$, then
  \begin{align*}
      \mathcal{A}_\Theta(\bm{x}_1\circ \bm{x}_2)\ \le\ \mathcal{A}_\Theta(\bm{x}_1).
  \end{align*} 
\end{proposition}
This proposition asserts that augmenting the context (via the concatenation of $\bm{x}_2$) generally reduces task ambiguity. This aligns with the information-theoretic principle that additional observations typically reduce uncertainty regarding latent variables.

\begin{proposition}\label{proposition: A ≤ 1 - E (without h)}
Let $\mathcal{E}(\bm{x}) := -\sum_{\theta \in \Theta} q(\theta\mid \bm{x})\,\log q(\theta\mid \bm{x})$ be the Shannon entropy of the posterior distribution $q(\theta \mid \bm{x})$. Then
\begin{align*}
    \mathcal{A}_{\Theta}(\bm{x}) \;\leq \; 1 - \exp\big(-\mathcal{E}(\bm{x})\big).
\end{align*}
\end{proposition}
This bound provides an information-theoretic envelope for task ambiguity. A low-entropy posterior $q(\theta \mid \bm{x})$ necessitates concentration on a dominant task, formally forcing $\mathcal{A}_{\Theta}(\bm{x})$ to vanish.

The proofs of these three propositions are deferred to Appendix~\ref{Appendix: properties of ambiguity}, which establishes a more general version of these results (Propositions~\ref{proposition: ε2 < ε1}, \ref{proposition: ε of the compositional prompt}, and \ref{proposition: A ≤ 1 - E}).


\section{Pretraining via Auto-Regression}\label{section: pretraining}
As previously stated, the comprehension capability of LLMs is intrinsically linked to the precise inference of the latent task $\theta$ from the provided prompt and historical context. However, because the latent task is fundamentally unobservable in practice, the community adopts the auto-regressive training paradigm. This approach minimizes the empirical risk over the training dataset $\mathcal{D}$, defined as:
\begin{align}\label{eq: hat_p}
    \hat{p}(\bm{t} \mid \bm{h}) \in \argmin_{p \in \mathcal{F}} -\frac{1}{Nn}\sum\limits_{i=1}^N\sum_{j = 1}^{n}\log p(\bm{t}^{(i)}_j\mid \bm{h}^{(i)}_j).
\end{align}
Here, $\bm{h}^{(i)}_j$ represents the historical context for the $j$-th token of the $i$-th document, defined as the padded sequence of all preceding tokens:
\begin{align*}
    \bm{h}^{(i)}_j = (\bm{t}^{(i)}_1, \cdots, \bm{t}^{(i)}_j, \texttt{<PAD>}, \cdots, \texttt{<PAD>}) \in \R^{d \times n}.
\end{align*}
This minimization is performed over the hypothesis space $\mathcal{F}$, which is conventionally defined as the class of Transformer architectures. In this work, we specify the hypothesis space as $\mathcal{F} = \mathcal{T}(W, D)$, whose definition has been provided in Section~\ref{subsection: transformer}. 

At the population level, the corresponding risk can be expressed as the summation of the expected Kullback–Leibler (KL) divergence and an additive constant:
\begin{align*}
    \mathcal{R}(p) &= -\E_{(\bm{t}, \bm{h}) \sim q(\bm{t}, \bm{h})}\Big\{\log q(\bm{t} \mid \bm{h})\Big\} = -\E_{(\bm{t}, \bm{h}) \sim q(\bm{t}, \bm{h})}\Big\{\log\frac{q(\bm{t} \mid \bm{h})}{p(\bm{t} \mid \bm{h})}\Big\} + \mathrm{const} \\
    & = -\E_{\bm{h} \sim q(\bm{h})}\Big[\E_{\bm{t} \sim q(\bm{t} \,\mid \,\bm{h})}\Big\{\log \frac{q(\bm{t} \mid \bm{h})}{p(\bm{t} \mid \bm{h})}\Big\}\Big] + \mathrm{const} \\
    &= -\E_{\bm{h} \sim q(\bm{h})}\Big\{D_{KL}\Big(q(\cdot \mid \bm{h}) \,\Big\| \,p(\cdot \mid \bm{h})\Big)\Big\} + \mathrm{const}.
\end{align*}
For the convenience of subsequent theoretical derivations, we assume the existence of a positive constant $b > 0$ such that $q(\bm{t}, \bm{h}) \geq b$ for all tokens $\bm{t} \in \mathcal{V}$ and historical contexts $\bm{h} \in \mathbb{R}^{d \times n}$. This assumption ensures that the joint probability of any historical sequence and its succeeding token is strictly bounded away from zero. A direct consequence of this condition is that the conditional probability also satisfies $q(\bm{t} \mid \bm{h}) \geq b$, as $q(\bm{t} \mid \bm{h}) = q(\bm{t}, \bm{h}) / q(\bm{h}) \geq q(\bm{t}, \bm{h}) \geq b$. Such regularity conditions, or more stringent variants, are standard in established LLM theory \citep{xie2022explanationincontextlearningimplicit, wies2023learnability, pmlr-v258-zhang25d} to prevent the posterior $q(\theta \mid \mathcal{P})$ from becoming meaningless. Under this assumption, the true distribution $q$ is the unique minimizer of the population risk $\mathcal{R}(p)$:
\begin{align*}
    q \in \argmin_{p \geq b}\mathcal{R}(p),
\end{align*}
which follows from the fact $D_{KL}\big(q(\cdot \mid \bm{h}) \,\big\|\, p(\cdot \mid \bm{h})\big) = 0$ if and only if $p(\cdot \mid \bm{h}) = q(\cdot \mid \bm{h})$ with respect to the probability measure $q(\cdot \mid \bm{h})$ almost surly.

Thus, the estimator $\hat{p}$ obtained via Eq.~\eqref{eq: hat_p} is intuitively expected to be quite close to the underlying human's language distribution. However, existing literature \citep{xie2022explanationincontextlearningimplicit, prystawski2023think, jiang2023latentspacetheoryemergent} frequently simplifies this relationship by assuming $\hat{p}(\bm{t} \mid \bm{h}) = q(\bm{t} \mid \bm{h})$ exactly and uniformly. This assumption is often too restrictive for practical scenarios. We offer a more grounded theoretical framework by deriving a high-probability error bound for this closeness, based on the structural properties of the vocabulary described below:

\begin{assumption}[Separable Token Representation]\label{assumption: separable token representation}
        For a finite vocabulary $\mathcal{V} \subseteq \R^d$ consisting all possible tokens in the sense of their representation. We assume it is $(\alpha, \beta)$-separated for some $\alpha, \beta > 0$. That is, the following two conditions are satisfied:
    \begin{enumerate}
        \item[(1)] for every $\bm{t} \in \mathcal{V}$, we have $\|\bm{t}\|_2 \leq \alpha$ holds.
        \item[(2)] for every $i\neq j \in [|\mathcal{V}|]$, we have $\|\bm{t}_i - \bm{t}_j\|_2 \geq \beta$ holds.
    \end{enumerate}
\end{assumption}
This assumption requires that the token representations fed into the Transformer are (1) confined to a bounded region of $\R^d$ and (2) sufficiently distinguishable from one another. Based on this, if we set the positional encoder $\bm{P}$ as
\begin{equation}\label{eq: def of positional encoder}
    \bm{P} = \begin{pmatrix}
    2\alpha & 4\alpha & \cdots & 2n\alpha \\
    \vdots & \vdots & \ddots & \vdots \\
    2\alpha & 4\alpha & \cdots & 2n\alpha 
    \end{pmatrix},
\end{equation}
If we further denote $M$ as the number of all possible history-token pairs, then we have:
\begin{theorem}\label{Theorem: pretraining}
    If Assumption~\ref{assumption: separable token representation} holds, then for any token $\bm{t} \in \mathcal{V}$ and $\bm{h} \in \R^{d \times n}$, by setting $W = \mathcal{O}(\mathcal{|V|}^{n+1}), D = 1$, we have
    \begin{align*}
        \sup_{(\bm{t}, \bm{h}) \sim q(\bm{t} \mid \bm{h})}\Big|\hat{p}(\bm{t} \mid \bm{h}) - q(\bm{t} \mid \bm{h})\Big| \leq \mathcal{O}\Big(\frac{|\mathcal{V}|^{n+2}}{\sqrt[4]{Nn}}\Big) + \mathcal{O}\Big(\sqrt{\frac{\ln(1/\delta)}{Nn}}\Big)
    \end{align*}
    holds with probability at least $1 - \delta$ for any $\delta \in (0,1)$.
\end{theorem}

The proof for this theorem is presented at Appendix~\ref{Appendix: pretraining}.

In practice, LLMs' responses are not restricted to a single token $t \in \mathcal{V}$. Instead, the output length exhibits a degree of flexibility. Formally, given a query prompt $\bm{x} \in \R^{d \times n}$, for any response $\bm{y} \in \R^{d \times r}$ satisfying $r + \mathcal{\ell}(\bm{x}) \leq n$, ensuring the concatenated length of the query $\bm{x}$ and response $\bm{y}$ does not exceed the context limit $n$, we have: 
\begin{align*}
    \hat{p}(\bm{y} \mid \bm{x}) := \prod_{l=1}^r \hat{p}(\bm{y}_{:,l} \mid \bm{x} \circ \bm{y}_{:, \prec l})    
\end{align*}
As a direct consequence of Theorem~\ref{Theorem: pretraining}, we further establish the following more generic corollary.
\begin{corollary}\label{corollary: pretraining}
    If Assumption~\ref{assumption: separable token representation} holds, then for any response $\bm{y} \in \R^{d \times r}$ and query $\bm{x} \in \R^{d \times n}$, by setting $W = \mathcal{O}(|\mathcal{V}|^{n+1}), D = 1$, we have
    \begin{align*}
        \Big|\hat{p}(\bm{y} \mid \bm{x}) - q(\bm{y} \mid \bm{x})\Big| < \mathcal{O}\Big(\frac{|\mathcal{V}|^{n+2}r}{\sqrt[4]{Nn}}\Big) + \mathcal{O}\Big(\sqrt{\frac{r^2\ln(1/\delta)}{Nn}}\Big)
    \end{align*}
    holds with probability at least $1 - \delta$ for any $\delta \in (0,1)$.
\end{corollary}
\section{LLMs' Comprehension}\label{section: LLMs' comprehension}

Building on Corollary~\ref{corollary: pretraining}, the following theorem shows that LLM responses align with the true distribution of the most probable latent task, provided the query prompt $\bm{x}$ is sufficiently unambiguous.   
\begin{theorem}\label{Theorem: LLMs' comprehension}
    If Assumption~\ref{assumption: separable token representation} holds, then for any responds $\bm{y} \in \R^{d \times r}$ and query $\bm{x} \in \R^{d \times n}$, by setting $W =\mathcal{O}(|\mathcal{V}|^{n+1}), D = 1$, we have
    \begin{align*}
        \Big| \hat{p}(\bm{y} \mid \bm{x}) - q(\bm{y} \mid \bm{x},\, \theta_{\bm{x}})\Big| \leq \mathcal{O}\Big(\frac{|\mathcal{V}|^{n+2}r}{\sqrt[4]{Nn}}\Big) + \mathcal{O}\Big(\sqrt{\frac{r^2\ln(1/\delta)}{Nn}}\Big) + \mathcal{A}_\Theta(\bm{x}).
    \end{align*}
    holds with probability at least $1 - \delta$ for any $\delta \in (0,1)$.
\end{theorem}

This result indicates that auto-regressive pretraining yields models capable of effective task discrimination, thereby enabling semantic understanding. The tightness of this error bound is primarily governed by the query ambiguity $\mathcal{A}_\Theta(\bm{x})$. A excessive ambiguity causes standard, zero-shot prompt to fail to identify the latent task a user intends the LLM to resolve. We further show that In-Context Learning (ICL) serves as a mechanism to mitigate this ambiguity, leading to superior performance.   

\begin{proof}
    First we notice that
    \begin{align*}
        \Big| q(\bm{y}\mid \bm{x}) - q(\bm{y} \mid \bm{x}, \, \theta_{\bm{x}})\Big| \leq \Big|\hat{p}(\bm{y} \mid \bm{x}) - q(\bm{y} \mid \bm{x})\Big| + \Big|q(\bm{y} \mid \bm{x}) - q(\bm{y} \mid \bm{x}, \, \theta_{\bm{x}})\Big|,
    \end{align*}
    where the first term represents the error induced during pretraining, it can be bounded by $\mathcal{O}\Big(\frac{|\mathcal{V}|^{n+2}r}{\sqrt[4]{Nn}}\Big) + \mathcal{O}\Big(\sqrt{\frac{r^2\ln(1/\delta)}{Nn}}\Big)$ with probability at least $1 - \delta$ according to the results yielded in Section~\ref{section: pretraining}.
    
    On the other hand, as for the second term, we have
    \begin{align*}
    \Big| q(\bm{y}\mid \bm{x}) - q(\bm{y} \mid \bm{x},\, \theta_{\bm{x}})\Big| &= \Big|\Big\{q(\bm{y} \mid \bm{x},\, \theta_{\bm{x}}) \cdot q(\theta_{\bm{x}}\mid \bm{x}) - q(\bm{y} \mid \bm{x},\, \theta_{\bm{x}})\Big\} + \sum_{\theta \neq \theta_{\bm{x}}}q(\bm{y},\, \theta\mid \bm{x}) \cdot q(\theta \mid \bm{x})\Big| \\
    &= \Big|q(\bm{y} \mid \bm{x}, \, \theta_{\bm{x}}) \cdot \Big\{q(\theta_{\bm{x}}\mid \bm{x}) - 1\Big\} + \sum_{\theta \neq \theta_{\bm{x}}}q(\bm{y} \mid \bm{x}, \theta) \cdot q(\theta \mid \bm{x})\Big| \\
    &= \Big|q(\bm{y} \mid \bm{x},\, \theta_{\bm{x}}) \cdot \Big\{q(\theta_{\bm{x}} \mid \bm{x}) - 1\Big\} + \sum_{\theta \neq \theta_{\bm{x}}}q(\bm{y}\mid \bm{x}, \, \theta) \cdot q(\theta \mid \bm{x})\Big| \\
    &= \Big| - q(\bm{y}\mid \bm{x},\, \theta_{\bm{x}}) \cdot \sum_{\theta \neq \theta_{\bm{x}}}q(\theta \mid \bm{x}) + \sum_{\theta \neq \theta_{\bm{x}}}q(\bm{y} \mid \bm{x}, \, \theta) \cdot q(\theta \mid \bm{x})\Big| \\
    &= \Big| \sum_{\theta \neq \theta_{\bm{x}}} q(\theta \mid \bm{x}) \cdot \Big\{q(\bm{y}\mid \bm{x}, \, \theta) - q(\bm{y}\mid \bm{x}, \, \theta_{\bm{x}})\Big\}\Big| \\
    &\leq \mathcal{A}_\Theta(\bm{x}),
\end{align*}
where the inequality stems from the fact that both $q(\bm{y} \mid \bm{x}, \,\theta)$ and $q(\bm{y} \mid\bm{x},\, \theta_{\bm{x}})$ are within $[0, 1]$, which turns out $|q(\bm{y} \mid \bm{x}, \, \theta) - q(\bm{y}\mid  \bm{x}, \, \theta_{\bm{x}})| \leq 1$. Combining above two bounds yields
\begin{align*}
    \Big|q(\bm{y} \mid \bm{x}) - q(\bm{y} \mid \bm{x}, \, \theta_{\bm{x}}) \Big| \leq \mathcal{O}\Big(\frac{|\mathcal{V}|^{n+2}r}{\sqrt[4]{Nn}}\Big) + \mathcal{O}\Big(\sqrt{\frac{r^2\ln(1/\delta)}{Nn}}\Big) + \mathcal{A}_\Theta(\bm{x}).
\end{align*}
which completes the proof.
\end{proof}
\section{In-Context Learning}\label{section: ICL}
In-Context Learning (ICL) \citep{brown2020language} allows LLMs to perform few-shot tasks without parameter updates by leveraging task-specific examples within the prompt. This capability effectively resolves the problem of prompt ambiguity via latent task identification.

For instance, the underspecified query \textit{Albert Einstein was..."} is compatible with multiple completions \textit{``physicist,"} \textit{``German,"} or \textit{``wise"}. The correct output depends on the latent task $\theta$. By providing a sequence of consistent demonstrations—e.g., \textit{``Nikola Tesla was inventor; Isaac Newton was mathematician; Marie Curie was chemist; Albert Einstein was..."}—the user enables the model to identify the intended task and marginalize out those $\theta$ assigning high probability to \textit{``German"} and others.

Following \citet{xie2022explanationincontextlearningimplicit}, we represent the ICL prompt $\mathcal{P}_{\mathrm{ICL}} \in \R^{d \times n}$ as the concatenation of $m$ demonstrations and a final query:
\begin{align*}
    \mathcal{P}_{\mathrm{ICL}} = \bm{x}^{(1)}\circ \bm{y}^{(1)}  \circ \bm{t}^{(1)}_{delim}  \circ \cdots \circ \bm{x}^{(m)} \circ \bm{y}^{(m)}  \circ \bm{t}^{(m)}_{delim} \circ \bm{x} \in \R^{d \times n},
\end{align*}
where $\bm{t}^{(i)}_{delim}$ are delimiters acting as structural indicators. Within the example provided by us, $\bm{t}^{(i)}_{delim}$ is ``\textit{;}".

We assume $\mathcal{P}_{\mathrm{ICL}}$ is generated autoregressively, conditioned on a composite latent random variable $\bm{\theta}_{\mathrm{ICL}} =\theta^{(1)} \circ \theta_{delim}^{(1)} \circ \cdots \circ \theta_{delim}^{(m-1)} \circ \theta^{(m)} \circ \theta_{delim}^{(m)} \circ \theta^{(0)}$, where each $\theta^{(i)}$ governs the transition rules for its corresponding segment. The conditional distribution factorizes as:
\begin{align*}
    q(\mathcal{P}_{\mathrm{ICL}} \mid \bm{\theta}_{\mathrm{ICL}}) &= q(\bm{x} \mid \mathcal{P}_{delim}^{\preceq m }, \, \theta^{(0)})\prod_{i=1}^m\Big\{q(\bm{x}^{(i)}\circ \bm{y}^{(i)} \mid \mathcal{P}_{delim}^{\prec i},\, \theta^{(i)})q(\bm{t}^{(i)}_{delim} \mid \mathcal{P}^{\preceq i},\,\theta_{delim}^{(i)})\Big\}
\end{align*}
where we define the following sequence prefixes:
\begin{itemize}
    \item $\mathcal{P}^{\prec i} := \bm{x}^{(1)}\circ \bm{y}^{(1)}  \circ \bm{t}^{(1)}_{delim}  \circ \cdots \circ \bm{x}^{(i)} \circ \bm{y}^{(i)}$ is the sequence up to the $i$-th demonstration ($\bm{x}^{(i)} \circ \bm{y}^{(i)}$).
    \item $\mathcal{P}^{\prec i}_{delim} := \bm{x}^{(1)}\circ \bm{y}^{(1)} \circ \bm{t}^{(1)}_{delim}  \circ \cdots \circ \bm{x}^{(i-1)} \circ \bm{y}^{(i-1)}\circ \bm{t}^{(i-1)}_{delim}$ represents the sequence up to the $i-1$-th delimiter ($\bm{t}^{(i-1)}_{delim}$). By convention, $\mathcal{P}^{\prec 1}_{delim} := \texttt{<SOS>}$.
    \item $\mathcal{P}^{\preceq i}_{delim} := \bm{x}^{(1)}\circ \bm{y}^{(1)} \circ \bm{t}^{(1)}_{delim}  \circ \cdots \circ \bm{x}^{(i)} \circ \bm{y}^{(i)}\circ \bm{t}^{(i)}_{delim}$ is the sequence up to the $i$-th delimiter ($\bm{t}^{(i)}_{delim}$).
\end{itemize}

Drawing upon~\citet{jiang2023latentspacetheoryemergent}, each component within the ICL prompt here are assumed to share an common latent task, reflecting the fundamental premise that demonstrations are provided specifically to disambiguate the latent task intended for the query $\bm{x}$. We further require the optimal task inferred from each demonstration to be consistent with the query's objective:
\begin{assumption}[Task Consistency]\label{assumption: task consistency}
    The latent variables $\theta^{(i)}$ for $i \in \{0, 1, \cdots, m\}$ are identical almost surely:
    \begin{align*}
        q(\theta^{(i)} = \theta^{(j)}) = 1
    \end{align*}
     Furthermore, these instances share a common optimal task $\theta_{\bm{x}}$ consistent with the optimal task intended for the query $\bm{x}$:
     \begin{align*}
         \theta_{\bm{x}} \in \argmax_{\theta \in \Theta}q(\theta \mid \bm{x}^{(i)}\circ\bm{y}^{(i)}), \quad i\in [m].
     \end{align*}
\end{assumption}
This consistency requirement simplifies our analysis by reducing the $m+1$ distinct $\theta^{(i)}$ to a single task $\theta$. Building on this simplification, we define a composite latent variable, The joint distribution can be factorized as follows:
\begin{align*}
    q(\mathcal{P}_{\mathrm{ICL}}, \bm{\theta}_{\mathrm{ICL}}) = q(\mathcal{P}_{\mathrm{ICL}} \mid \bm{\theta}_{\mathrm{ICL}})q(\theta)\prod_{i=1}^mq(\theta^{(i)}_{delim}).
\end{align*}

While prior works \citep{jiang2023latentspacetheoryemergent, li2025transformers} often rely on the assumption that demonstrations are strictly independent conditioned on $\theta$, we draw inspiration from the insights of \citet{xie2022explanationincontextlearningimplicit} and treat delimiters as mechanisms for history truncation. First, we assume their generative process is governed by a delimiter space $\mathcal{D} \subset \Theta$ satisfying:
\begin{assumption}[Tasks of Delimiter]\label{assumption: tasks of delimiter}
    There exists a delimiter space $\mathcal{D}$ such that for any delimiter $\bm{t}_{delim} \in \mathcal{V}$ such that $q(\bm{t}_{delim} \mid \bm{h}, \,\theta_{delim}) =1$ holds for any $\theta_{delim} \in \mathcal{D}$ and $q(\bm{t}_{delim} \mid \bm{h}, \,\theta) = 0$ holds for any history $\bm{h}$ and $\theta \in \mathcal{D}$.
\end{assumption}
This assumption essentially claim that there is a one-to-one correspondence between delimiter tokens and their latent tasks. Furthermore, we assume that delimiters satisfy a nearly Markov property, which allows the model to partition distinct instances by effectively ``resetting" the historical context from the cumulative prefix $\mathcal{P}_{delim}^{\prec i}$ back to the start-of-sequence token $\texttt{<SOS>}$:

\begin{assumption}[Nearly Markov]\label{assumption: nearly markov}
    For any latent task $\theta \in \Theta$, token $\bm{t} \in \R^d$, history $\bm{h} \in \R^{d \times n}$ and sequence $\bm{s} \in \R^{d \times n}$, with $\ell(\bm{h}) + \ell(\bm{s})+1 \leq n$, we have
    \begin{align*}
        \Big|\log q\big(\bm{t} \mid \bm{h}\circ \bm{t}_{delim} \circ \bm{s},\, \theta\big) - \log q\big(\bm{t} \mid \bm{s}, \, \theta\big)\Big| \leq \phi
    \end{align*}
    holds for some $\phi > 0$.
\end{assumption}
Versions of this assumption are frequently adopted in the theoretical analysis of LLMs \citep{xie2022explanationincontextlearningimplicit, wies2023learnability}, notwithstanding slight differences in their precise mathematical expressions.

Beyond the structural properties of the prompt, the model's ability to recover the intended task depends on the ``fairness'' of the task space. We therefore require a balanced representation of latent tasks in the pretraining distribution to ensure the task-specific signal is not overshadowed by prior bias:
\begin{assumption}[Imbalance Prior]\label{Assumption: Bounded Priori Ratio}
    There exists a positive constant $c > 0$ such that
    \begin{align*}
        \sup_{\theta_1, \theta_2 \in \Theta}\frac{q(\theta_1)}{q(\theta_2)} \leq c
    \end{align*}
\end{assumption}
Assumption \ref{Assumption: Bounded Priori Ratio} essentially requires that the latent tasks in $\Theta$ are represented with a degree of parity in the pretraining distribution. This is a common assumption in latent variable models for LLMs (e.g., \cite[Theorem 1]{xie2022explanationincontextlearningimplicit}, \cite[Assumption 3]{wies2023learnability}, \cite[Theorem 5.5]{hu2024unveiling}), as it guarantees that the ``signal" provided by the in-context examples is not overshadowed by an extreme initial bias in the prior. Mathematically, this ensures that the posterior distribution $q(\theta \mid \mathcal{P}_{\mathrm{ICL}})$ can effectively concentrate around the true task $\theta_{\bm{x}}$, allowing the product of error terms in Theorem \ref{theorem: ICL} to drive the overall error toward zero.

Furthermore, if we denote 
\begin{align*}
   \eps=\frac{1}{1-\mathcal{A}_\Theta(\bm{x})}\max_{i \in [m]}\Big\{\frac{\mathcal{A}_\Theta(\bm{x}^{(i)}\circ\bm{y}^{(i)})}{1 - \mathcal{A}_\Theta(\bm{x}^{(i)}\circ\bm{y}^{(i)})}\Big\},
\end{align*}
where $\epsilon$ becomes sufficiently small provided the constructed ICL instances are unambiguous and well-defined. Under these conditions, the model's prediction error decays exponentially with the number of demonstrations $m$, as formalized below:
\begin{theorem}\label{theorem: ICL}
    If Assumption~\ref{assumption: separable token representation},~\ref{assumption: task consistency},~\ref{assumption: tasks of delimiter},~\ref{assumption: nearly markov} and~\ref{Assumption: Bounded Priori Ratio} hold, then for any response $\bm{y} \in \R^{d \times r}$ and query $\bm{x} \in \R^{d \times n}$, by setting $W =\mathcal{O}(|\mathcal{V}|^{n+1}), D = 1$, we have
    \begin{align*}
        \Big| \hat{p}(\bm{y} \mid \mathcal{P}_{\mathrm{ICL}}) - q(\bm{y} \mid \bm{x},\, \theta_{\bm{x}})\Big| \leq  \mathcal{O}\Big(\frac{|\mathcal{V}|^{n+2}r}{\sqrt[4]{Nn}}\Big) + \mathcal{O}\Big(\sqrt{\frac{r^2\ln(1/\delta)}{Nn}}\Big)+r\phi + (e^{2n\phi}\cdot c\cdot\eps)^m \mathcal{A}_\Theta(\bm{x})
    \end{align*}
    holds with probability at least $1 - \delta$ for any $\delta \in (0,1)$.
\end{theorem}
The proof of Theorem~\ref{theorem: ICL} can be founded in Appendix~\ref{Appendix: Proof for In-Context Learning}.

\paragraph{Comparison with standard prompt}
The theoretical significance of Theorem~\ref{theorem: ICL} is most salient when contrasted with the zero-shot performance established in Section~\ref{section: LLMs' comprehension}. In the absence of demonstrations ($m=0$), the prediction error is governed by the task ambiguity $\mathcal{A}_\Theta(\bm{x})$, which quantifies the model's intrinsic uncertainty regarding the user's intent when presented with an underspecified query. As the ``\textit{Einstein was...}" example illustrates, this term represents a fundamental bottleneck for zero-shot ``comprehension" whenever the prompt $\bm{x}$ lacks sufficient context to distinguish between competing latent tasks.

In contrast, the ICL paradigm fundamentally reconfigures this landscape by introducing the multiplicative decay factor $(e^{2n\phi}\cdot c \cdot \eps)^m$. Each demonstration $(\bm{x}^{(i)}\circ \bm{y}^{(i)})$ functions as a Bayesian filter that narrows the effective task space, systematically concentrating the posterior distribution $q(\theta \mid \mathcal{P}_{\mathrm{ICL}})$ around the target task $\theta_{\bm{x}}$. Within this framework, Assumption \ref{Assumption: Bounded Priori Ratio} ensures the model remains "open-minded": by bounding the prior imbalance $c$, we guarantee that the evidence provided by demonstrations can overcome initial prior biases.

As $m$ increases, this cumulative error term vanishes exponentially, provided the examples satisfy the disambiguation threshold ($\eps < 1/c$). This allows the models' prediction $\hat{p}$ to converge toward the true task-specific distribution $q(\bm{y} \mid \bm{x},\, \theta_{\bm{x}})$, effectively resolving even severe semantic ambiguities.

\paragraph{Limitation} Despite its success, In-context learning is not a panacea. Its effectiveness diminishes when the task involves complex logical structures, such as multi-step mathematical problems \citep{wei2022chain}. For instance, a standard ICL prompt might yield an incorrect result for a basic arithmetic query:
\begin{center}
    \begin{minipage}{0.85\textwidth} 
        \begin{itemize}[leftmargin=2.2cm, labelwidth=1.8cm, labelsep=0.4cm, align=left]
            \item[Model Input] \textit{Q: Roger has $5$ tennis balls. He buys $2$ more cans of 
            tennis balls. Each can has $3$ tennis balls. How many tennis balls does he have now? \\ 
            A: The answer is $11$. \\
            Q: The cafeteria had $23$ apples. If they used $20$ to 
            make lunch and bought $6$ more, how many apples 
            do they have?}
            \item[Model Output] The answer is $27$.
        \end{itemize}
    \end{minipage}
\end{center}
However, the correct answer is $9$. 
\section{Chain-of-Thought}\label{section: CoT}
A robust solution to above failure involves replacing the direct answer in the demonstration with an explicit reasoning path:
\begin{center}
    \begin{minipage}{0.85\textwidth} 
        \textit{A: Roger started with $5$ balls. $2$ cans of $3$ tennis balls 
each is $6$ tennis balls. $5 + 6 = 11$. The answer is $11$.}
    \end{minipage}
\end{center}
Under this setting, the model is significantly more likely to produce the correct answer. Up to now, inducing such a ``chain of thought" (CoT), a sequence of intermediate reasoning steps, is known to significantly outperform direct prediction~\citep{wei2022chain}, though the theoretical basis for this improvement remains mysterious.

To activate these capabilities, the standard CoT prompting method \citep{wei2022chain} is always designed to include $m$ few-shot demonstrations: for each $i \in [m]$, it is denoted by $\bm{x}^{(i)} \circ \bm{Y}^{(i)}$. Further, each $\bm{Y}^{(i)}$ are consists of $L$ reasoning steps, i.e., $\bm{Y}^{(i)} = \bm{y}^{(i)}_1 \circ \bm{y}^{(i)}_2 \circ \dots \circ \bm{y}^{(i)}_L$, where every $\bm{y}^{(i)}_j \in \R^{d \times r^{(i)}_j}$. These instances serve as explicit reasoning blueprints, guiding the model through the intermediate stages required to resolve a target query $\bm{x}$. The resulting CoT prompt, $\mathcal{P}_{\mathrm{CoT}}$, is defined as:
\begin{align}\label{eq: PCoT}
    \mathcal{P}_{\mathrm{CoT}} := \bm{x}^{(1)} \circ \bm{Y}^{(1)} \circ \bm{t}_{delim}^{(1)} \circ\cdots \circ \bm{x}^{(m)} \circ \bm{Y}^{(m)} \circ \bm{t}_{delim}^{(m)} \circ \bm{x}.
\end{align}

To help reader quickly understand the notation within above CoT prompt, we provide a simple multi-step word problem with single demonstration ($m=1$) and three reasoning steps ($L=3$):
\begin{center}
    \begin{minipage}{0.85\textwidth} 
        \begin{itemize}[leftmargin=2.2cm, labelwidth=1.8cm, labelsep=0.4cm, align=left]
            \item[$\bm{x}$] \textit{``Alice has \$20. She buys 3 books for \$5 each. How much remains?''}
            \item[$\bm{x}^{(1)}$] \textit{``Bob has 4 boxes of chocolates with 5 pieces each. He gives 2 away. How many remain?''}
            \item[$\bm{y}^{(1)}_1$] \textit{``Total pieces are $4 \times 5 = 20$.''}
            \item[$\bm{y}^{(1)}_2$] \textit{``Subtracting the gift, $20 - 2 = 18$.''}
            \item[$\bm{y}^{(1)}_3$] \textit{``The answer is 18.''}
        \end{itemize}
    \end{minipage}
\end{center}

While such prompt engineering is widely recognized as a robust method for activating emergent reasoning in LLMs, its underlying theoretical mechanism remains an open question. Intuitively, different from the case that only involved an atomic task, which is discussed in previous sections, the complex reasoning tasks are always composed of multiple latent subtasks. Referring to the previously provided instance, the subtask associated with $\bm{y}_1^{(1)}$ is related to ``multiplication", while the task for $\bm{y}_2^{(1)}$ is about ``addition and subtraction", and the final step $\bm{y}^{(1)}_3$ entails ``result extraction and summarization". As the result, the transition probability at different stages should be distinct.

Specifically, assume there are $m$ independent latent variables $\theta^{(i)} \in \Theta$, which represents the corresponding tasks governing the provided $m$ queries token by token, i.e., $\bm{x}^{(i)} \sim q(\cdot \mid \mathcal{P}^{\prec i}_{delim},\, \theta^{(i)})$. Similarly, $\theta^{(0)}$ is the transition rule governing $\bm{x}$: $\bm{x} \sim q(\cdot \mid \mathcal{P}^{\preceq m}_{delim},\, \theta^{(0)})$. By contrast, all $\bm{Y}^{(i)}, i \in [m]$ are assumed to share a common compositional reasoning task $\bm{\theta} = \theta_1\circ \dots \circ \theta_L \in \Theta^L$. In detail, each $\bm{Y}^{(i)}$ are independently generated under $\bm{\theta}$ from several steps: $\theta_1 \in \Theta$ first indicates the transition rule from the query $\mathcal{P}^{\prec i}_{delim} \circ \bm{x}^{(i)}$ to the next reasoning step $\bm{y}_1^{(i)}$, i.e., $\bm{y}_1^{(i)}\sim q(\cdot \mid \mathcal{P}^{\prec i}_{delim} \circ \bm{x}^{(i)},\, \theta_1)$ token by token. Subsequently, we obtain $\bm{y}_j^{(i)} \sim q(\cdot \mid \mathcal{P}^{\prec i}_{delim} \circ \bm{x}^{(i)}\circ\bm{y}^{(i)}_{\prec j}, \, \theta_j)$ step by step for $j \in \{2, \cdots, L\}$. Here we denote $\bm{y}^{(i)}_{\prec j}:=\bm{y}^{(i)}_1 \circ \cdots \circ \bm{y}^{(i)}_{j-1}$. In this context, the concatenation $\vec{\bm{\theta}} = (\theta_0^{(0)}, \theta_0^{(1)}, \theta_{delim}^{(1)}, \cdots, \theta^{(m)}_0, \theta^{(m)}_{delim}, \bm{\theta}) \in \Theta^{L+2m+1}$ entirely determines the transition rule regarding the generation of $\mathcal{P}_{\mathrm{CoT}}$, and its joint distribution implicitly follow the decomposition below:

\begin{gather*}\label{eq: joint independence of CoT}
    \tilde{q}(\mathcal{P}_{\mathrm{CoT}} \circ \bm{Y}, \,\vec{\bm{\theta}}) := \tilde{q}(\bm{\theta})\prod_{i=0}^mq(\theta^{(i)} \mid \bm{\theta})\tilde{q}(\mathcal{P}_{\mathrm{CoT}} \mid \vec{\bm{\theta}})
\end{gather*}
where
\begin{align}\label{eq: CoT likelihood decomposition}
    q(\mathcal{P}_{\mathrm{CoT}} \mid \vec{\bm{\theta}}) &:= q(\bm{x} \mid \mathcal{P}^{\preceq m}_{delim},\, \theta^{(0)})\prod_{i=1}^m\Big\{q(\bm{x}^{(i)} \mid \mathcal{P}^{\prec i}_{delim}, \,\theta^{(i)})q(\bm{t}^{(i)}_{delim} \mid \mathcal{P}^{\preceq i}, \,\theta^{(i)}_{delim}) \notag\\
    &\quad \cdot \prod_{j=1}^Lq( \bm{y}^{(i)}_j \mid \mathcal{P}^{\prec i}_{delim} \circ \bm{x}^{(i)}\circ \bm{y}^{(i)}_{\prec j}, \,\theta_j)\Big\}.
\end{align}
For notational consistency, we let $\bm{x}^{(i)} \circ \bm{y}_{\prec{1}}^{(i)}$ denote $\bm{x}^{(i)}$.

However, even if each atom probability $q(\bm{t}\mid \bm{h},\, \theta)$ has been mastered after pretraining according to Theorem~\ref{Theorem: LLMs' comprehension}, their compositional structure are never explicitly indicated at the pretraining phase, resulting in the model often struggles to resolve such multi-stage problems during inference.

We name such drift from pretraining to reasoning as a \textit{compositional shift}. This can be viewed as a novel framework of transfer learning. For the transfer learning problem, it is inevitable to quantify the the extent of distribution drift. Yet, the inconsistency between these two domains impedes us adopt the existing transfer learning frameworks to describe the distribution shift. To address this, we propose an novel transfer learning framework, which pushes forward the original measure onto another domain, thereby facilitating a rigorous comparison on a shared support.

\paragraph{Distribution Shift} Since the answer regarding the query $\bm{x}$ is generated token-by-token, the sequence generated under an atomic task $\theta \in \Theta$ will be equivalent to the counterpart producing under $\theta\circ \theta\circ \dots\circ \theta \in \Theta^{L}$. Thus, the pretraining space $\Theta$ can be viewed as a stationary embedding within the broader compositional space, i.e., $\Theta \subseteq \Theta^L$. Furthermore, the stationary prior $q(\theta)$, originally defined over the atomic task space $\Theta^L$, naturally induces a distribution over the broader compositional space $\Theta^L$. Specifically, for any stationary task sequence $\bm{\theta} = \theta \circ \theta \circ \dots \circ \theta \in \Theta^L$ induced by $\theta \in \Theta$, we define the induced prior as $q(\bm{\theta}) := q(\theta)$. In contrast, let $\tilde{q}(\bm{\theta})$ denote the prior distribution over the shared transition rule $\bm{\theta} \in \Theta^L$.

The induced measure over $\Theta^L$ by $q(\theta)$ admits us to compare it with $\tilde{q}(\bm{\theta})$ since they have been defined over a common support. Specifically, we adopt following quantity to describe the difference between these two prior distribution:
\begin{definition}[Prior Mismatch]\label{def: prior mismatch}
Given a prompt $\mathcal{P} \in \mathbb{R}^{d \times n}$, we define the pointwise prior mismatch $\Delta_\mathcal{P}$ as: 
\begin{align*}
    \Delta_\mathcal{P} = \sum_{\bm{\theta} \in \mathscr{F}(\mathcal{P})}\Big|\tilde{q}(\bm{\theta}) - q(\bm{\theta})\Big|
\end{align*}
where $\mathscr{F}(\mathcal{P}) = \{\bm{\theta} \in \Theta^L \mid q(\mathcal{P}\mid \bm{\theta}) > 0\}$ represents the collection of $\bm{\theta}$ which makes the likelihood regarding $\mathcal{P}$ strictly larger than $0$.
\end{definition}
This metric captures the cumulative deviation of the ``reasoning world" $\tilde{q}(\bm{\theta})$ from the ``pretraining world" $q(\theta)$ for a specific prompt $\mathcal{P}$. While the pretraining distribution $q(\bm{\theta})$ is supported only on stationary subset of $\Theta^{L+1}$, the inference prior $\tilde{q}(\bm{\theta})$ assigns mass to non-stationary trajectories, allowing us to characterize CoT reasoning as the model's ability to access non-stationary trajectories $\bm{\theta} \in \Theta^L \setminus \Theta$, which indicates that trajectories composed of known atomic tasks in configurations never explicitly encountered during pretraining.


Given that the query and reasoning steps correspond to different sub-tasks in this context, we require a generalized ambiguity definition (Definition~\ref{def: dominated task and ambiguity (non-history)}) capable of characterizing how uncertainty varies across distinct historical sequences.  
\begin{definition}[Dominated Inferred Task and Ambiguity]\label{def: dominated task and ambiguity}
    For any prompt $\bm{x} \in \mathbb{R}^{d \times n}$ and historical context $\bm{h} \in \mathbb{R}^{d \times n}$ satisfying $\ell(\bm{x}) + \ell(\bm{h}) \leq n$, we define the Dominated Inferred Task $\theta^{\bm{h}}_{\bm{x}}$ as the most probable latent task conditioned on the concatenated context:
\begin{align*}
    \theta^{\bm{h}}_{\bm{x}} := \argmax_{\theta \in \Theta} q(\theta \mid \bm{h}\circ \bm{x}).
\end{align*}

The \textbf{task ambiguity} induced by the prompt $\bm{x}$ in the presence of history $\bm{h}$ is defined as: 
\begin{align*}
    \mathcal{A}_{\Theta}^{\bm{h}}(\bm{x}) := 1- q(\theta^{\bm{h}}_{\bm{x}} \mid \bm{h}\circ\bm{x}).
\end{align*}
For brevity, when the historical context is the start-of-sequence token $\texttt{<SOS>}$, we denote $\theta_{\bm{x}} := \theta_{\bm{x}}^{\texttt{<SOS>}}$ and $\mathcal{A}_{\Theta}(\bm{x}) := \mathcal{A}^{\texttt{<SOS>}}_\Theta(\bm{x})$.
\end{definition}

Similar to Definition~\ref{def: dominated task and ambiguity (non-history)}, this new definition also satisfies the following three properties:
\begin{proposition}\label{proposition: ε2 < ε1}
    For any given context $\bm{x}$ and $\bm{h}$ with $\ell(\bm{x}) + \ell(\bm{h}) \leq n$, if $\Theta_1 \subseteq \Theta_2$, then 
    \begin{align*}
        \mathcal{A}_{\Theta_2}^{\bm{h}}(\bm{x}) \leq \mathcal{A}_{\Theta_1}^{\bm{h}}(\bm{x}).
    \end{align*}
\end{proposition}

\begin{proposition}\label{proposition: ε of the compositional prompt}
Let $\bm{h}$ be any historical context and $\bm{x}_1$ be any prompt with $\ell(\bm{x}) + \ell(\bm{h}) \leq n$, then
\begin{align*}
    \E_{\bm{x}_2\sim q(\cdot\mid\, \bm{h}\, \circ \, \bm{x}_1)}\big\{\mathcal{A}_\Theta^{\bm{h}}(\bm{x}_1\circ \bm{x}_2)\big\}\ \le\ \mathcal{A}_\Theta^{\bm{h}}(\bm{x}_1).
\end{align*}
Furthermore, if $\bm{x}_2$ is another prompt such that $\theta_{\bm{x}_1} = \theta_{\bm{x}_1 \, \circ \,\bm{x}_2}$ and $q(\bm{x}_2\mid \bm{h} \circ \bm{x}_1,\theta^{\bm{h}}_{\bm{x}_1})\ \ge\ q(\bm{x}_2\mid \bm{h} \circ \bm{x}_1,\theta)$ for all $\theta\in\Theta$, then
  \begin{align*}
      \mathcal{A}_\Theta^{\bm{h}}(\bm{x}_1\circ \bm{x}_2)\ \le\ \mathcal{A}_\Theta^{\bm{h}}(\bm{x}_1).
  \end{align*} 
\end{proposition}

\begin{proposition}\label{proposition: A ≤ 1 - E}
Let $\mathcal{E}^{\bm{h}}(\bm{x}) := -\sum_{\theta \in \Theta} q(\theta\mid \bm{h} \circ \bm{x})\,\log q(\theta\mid \bm{h} \circ \bm{x})$ be the Shannon entropy of the posterior distribution $q(\theta \mid \bm{h} \circ \bm{x})$. Then
\begin{align*}
    \mathcal{A}_{\Theta}^{\bm{h}}(\bm{x}) \;\leq \; 1 - \exp\big(-\mathcal{E}^{\bm{h}}(\bm{x})\big).
\end{align*}
\end{proposition}

Building on Definition~\ref{def: dominated task and ambiguity}, we define the step-wise optimal tasks as follows:
\begin{gather*}
    \theta^\star_l = \argmax_{\theta \in \Theta}q\big(\theta \mid (\bm{x}^{(i)}\circ \bm{y}^{(i)}_{\prec l})\circ \bm{y}_l^{(i)} \big), \quad l \in [L].
\end{gather*}
This $\theta^\star_l$ characterizes the optimal task for reasoning step $\bm{y}^{(i)}_l$ starting from its local history $\bm{x}^{(i)}\circ \bm{y}^{(i)}_{\prec l}$. Notably, since each sequence $\bm{Y}^{(i)}$ is generated under a shared compositional task, the resulting optima $\theta^\star_l$ are invariant to the instance index $i \in [m]$.

Furthermore, their composition is denoted by $\bm{\theta}^\star = \theta^\star_1 \circ \cdots \circ \theta^\star_L \in \Theta^L$. Within our proposed framework, the critical question is: \textit{Is the predictive distribution $\hat{p}(\bm{Y} \mid \mathcal{P}_{\mathrm{CoT}})$ capable of inferring the optimal compositional task $\bm{\theta}^\star$ from the provided context $\mathcal{P}_{\mathrm{CoT}}$ exactly?} Toward this end, we introduce several assumptions requisite for the development of our proposed framework.

\begin{definition}[Difference set and Hamming distance]
    For any two compositional tasks $\bm{\theta}_1, \bm{\theta}_2 \in \Theta^{L}$, their Hamming distance $d_H(\bm{\theta}_1, \bm{\theta}_2)$ is defined as the number of positions at which their corresponding components differ:
    \begin{align*}
        d_H(\bm{\theta}_1, \bm{\theta}_2) := \big|\{i \in [L]: \bm{\theta}_{1,i} \neq \bm{\theta}_{2,i}\}\big|.
    \end{align*}
\end{definition}

Building on this definition, let $\mathscr{S}({\mathcal{P}}) = \big\{\bm{\theta} \in \Theta^L \mid q(\mathcal{P} \mid \bm{\theta}) > 0\big\}$ denote the set of compositional tasks consistent with a given prompt $\mathcal{P}$, we assume that the task space satisfies the following separation property:
\begin{assumption}[$K$-separation]\label{assumption: K-separation}
    There exists an integer $K \in [L]$ such that for any two distinct compositional tasks $\bm{\theta}_1, \bm{\theta}_2 \in \mathscr{S}(\mathcal{P}_{\mathrm{CoT}})$, their Hamming distance satisfies:
    \begin{align*}
        d_{H}(\bm{\theta}_1, \bm{\theta}_2) \geq K.
    \end{align*}
\end{assumption}
It essentially assumes that the ``Reasoning World" is structured such that you can't easily mistake one logical process for another. By showing the model $L$ steps, and knowing that any ``wrong" path is $K$ steps away, the model can confidently ``lock in" on the correct reasoning chain. In what follows, we will show the impact of $K$ on task identification accuracy.

\begin{assumption}\label{assumption: regularity}
    There exists a positive constant $c_1, c_2>0$ such that
    \begin{align*}
        \sup_{\bm{\theta}_1\neq\bm{\theta}_2 \in \mathscr{S}(\mathcal{P}_{\mathrm{CoT}})}\frac{\tilde{q}(\bm{\theta}_1)}{\tilde{q}(\bm{\theta}_2)} \leq c_1, \quad \sup_{\bar\theta_j \neq \tilde{\theta}_j \in \mathscr{S}(\bm{x}^{(i)} \circ \bm{y}^{(i)}_{\prec j})}\frac{q(\bar\theta_j\mid \bm{x}^{(i)} \circ \bm{y}_{\prec j}^{(i)})}{q(\tilde{\theta}_j\mid \bm{x}^{(i)} \circ \bm{y}_{\prec j})} \leq c_1, 
    \end{align*}
    and
    \begin{align*}
        \sup_{\theta^{(i)} \in \mathscr{S}(\bm{x}^{(i)})}q(\bm{x}^{(i)} \mid \theta^{(i)}) \geq c_2.
    \end{align*}
\end{assumption}
Assumption \ref{assumption: regularity} generalizes the prior parity requirement of ICL (Assumption~\ref{Assumption: Bounded Priori Ratio}) to both the global and local structures of multi-step reasoning:
\begin{itemize}[leftmargin=0.6cm]
\item \textbf{Global Path Parity:} The bound on $\tilde{q}(\bm{\theta}_1)/\tilde{q}(\bm{\theta}_2)$ ensures that the reasoning trajectories within the compositional space $\Theta^L$ are represented with a degree of parity, preventing any single logical path from dominating the prior.
\item \textbf{Local Step Parity:} Local Step Parity mirrors the ICL prior balance (Assumption~\ref{Assumption: Bounded Priori Ratio}), simply updating the conditioning history to the reasoning context accumulated before step $j$. This keeps the prior balance, guaranteeing that evidence provided can counteract such prior biases.
\item \textbf{Likelihood Floor:} The lower bound $c_1$ on the likelihood $q(\bm{x}^{(i)} \mid \theta^{(i)})$ ensures that the provided demonstrations are sufficiently representative of their tasks to provide a meaningful signal for inference.
\end{itemize}

Under the aforementioned assumptions, define the ambiguity coefficient as: 
\begin{align*}
    \eps:=\max_{i \in [m], j\in[L]}\Big\{\frac{\mathcal{A}^{\bm{x}^{(i)}\circ\bm{y}^{(i)}_{\prec j}}_\Theta(\bm{y}^{(i)}_j)}{1 - \mathcal{A}_\Theta^{\bm{x}^{(i)}\circ\bm{y}^{(i)}_{\prec j}}(\bm{y}^{(i)}_j)}\Big\},
\end{align*} 
which remains negligible when the prompt $\mathcal{P}_{\mathrm{CoT}}$ sufficiently unambiguous. Furthermore, Furthermore, let $\mathcal{M} := \max\big\{q(\mathcal{P}_{\mathrm{CoT}})^{-1}, \tilde{q}(\mathcal{P}_{\mathrm{CoT}})^{-1}\big\}$ denote the reciprocal of the marginal prompt probabilities; this value is a positive constant under our assumption that $q(\bm{t}, \bm{h}) > 0$ for all $\bm{t} \in \mathcal{V}$ and $\bm{h} \in \R^{d \times n}$. We also define the scaling constant $C := \frac{c_1e^{2n(L - K)\phi}c_2^{-(m+1)}}{1 - c_1\cdot \eps}$, which approaches unity as both the Nearly Markov parameter $\phi$ and the ambiguity $\eps$ vanish. Under these conditions, we establish the following result:
\begin{theorem}\label{theorem: CoT}
    If Assumption~\ref{assumption: separable token representation},~\ref{assumption: tasks of delimiter},~\ref{assumption: nearly markov},~\ref{assumption: K-separation} and ~\ref{assumption: regularity} hold, by setting appropriate $W =\mathcal{O}(|\mathcal{V}|^{n+1}), D = 1$, we have
    \begin{align*}
        \Big| \hat{p}(\bm{y} \mid \mathcal{P}_{\mathrm{CoT}}) - \tilde{q}(\bm{y} \mid \bm{x},\, \bm{\theta}^\star)\Big| &\leq  \mathcal{O}\Big(\frac{|\mathcal{V}|^{n+2}r}{\sqrt[4]{Nn}}\Big) + \mathcal{O}\Big(\sqrt{\frac{r^2\ln(1/\delta)}{Nn}}\Big) + r\phi + \mathcal{M} \Delta_{\mathcal{P}_{\mathrm{CoT}}} \\
        &\quad +C\cdot (e^{2n\phi} \cdot c_1\cdot\eps)^{mK}
    \end{align*}
    with probability at least $1 - \delta$ for any $\delta \in (0,1)$.
\end{theorem}
The proof is deferred to Appendix~\ref{Appendix: Proof for Chain-of-Thought}.

\paragraph{Comparison with ICL and standard prompt}
The significance of Theorem \ref{theorem: CoT} is best highlighted by the ``compositional bottleneck" inherent in zero-shot and standard ICL prompting. Zero-shot prompts fundamentally struggle to resolve multi-stage tasks in the space $\Theta^L \setminus \Theta$ in a single leap. While standard ICL reduces task ambiguity at an exponential rate of order $m$, it remains insufficient for complex reasoning as it lacks the structural granularity to bridge non-stationary trajectories. 

In contrast, CoT prompting facilitates the emergence of capabilities on compositional tasks by providing a significantly higher-order error reduction of order $mK$. Here, the exponent captures the synergy between the number of demonstrations ($m$) and distinguishable reasoning length ($K$). By decomposing novel global trajectories into mastered atomic sub-tasks, CoT provides a formal mechanism to resolve the compositional shift, allowing LLMs to accurately navigate complex reasoning paths that were never encountered as unified blocks during pretraining.

\subsection{More Discussions}
In this section, we extend our proposed Chain-of-Thought (CoT) theory to a more general setting. In Eq.~\eqref{eq: CoT likelihood decomposition}, we assumed that even if the task sequence $\vec{\bm{\theta}}$ does not appear during pretraining, the token distributions followed during the inference phase exactly align with their pretraining counterparts once the history and task are provided. This is based on the premise that once the context and intention are identified, the output follows patterns established in the pretraining data. However, we demonstrate that our framework can also accommodate cases where
    \begin{align}\label{eq: CoT likelihood decomposition (~q)}
        \tilde{q}(\mathcal{P}_{\mathrm{CoT}} \mid \vec{\bm{\theta}}) &:= \tilde{q}_0(\bm{x} \mid \mathcal{P}^{\preceq m}_{delim},\, \theta^{(0)})\prod_{i=1}^m\Big\{\tilde{q}_0(\bm{x}^{(i)} \mid \mathcal{P}^{\prec i}_{delim}, \,\theta^{(i)})\tilde{q}_0(\bm{t}^{(i)}_{delim} \mid \mathcal{P}^{\preceq i}, \,\theta^{(i)}_{delim}) \notag\\
    &\quad \cdot \prod_{j=1}^L\tilde{q}_j( \bm{y}^{(i)}_j \mid \mathcal{P}^{\prec i}_{delim} \circ \bm{x}^{(i)}\circ \bm{y}^{(i)}_{\prec j}, \,\theta_j)\Big\}.
\end{align}
with $\tilde{q}(\bm{t} \mid \bm{h},\, \theta) \neq q(\bm{t} \mid \bm{h}, \, \theta)$ provided following assumption:
\begin{assumption}[Evidence Shift]
\label{assumption: conditional distribution shift}
We say the measure $\tilde{q}$ exhibits an $\phi$-evidence shift if, for every $l \in \{0, 1,\cdots, L\}$, the following inequality holds for all histories $\bm{h} \in \R^{d \times n}$ and atomic tasks $\theta \in \Theta$:
\begin{align*}
    \mathrm{TV}\!\Big(q\big(\bm{t}\mid \bm{h},\theta\big),\,\tilde{q}_l\big(\bm{t}\mid \bm{h}, \theta\big)\Big)\ \leq \varphi
\end{align*}
where the total variation distance is defined over token representations $\bm{t} \in \R^d$.
\end{assumption}

This formulation is analogous to the ``model shift" or ``distribution shift" studied in traditional transfer learning frameworks \citep{reeve2021adaptive, cai2021transfer, maity2024linear, lang2025unsupervised}. It accounts for discrepancies in linguistic style or logical structure between the pretraining and inference phases. By invoking Assumption~\ref{assumption: conditional distribution shift}, we can reconcile this generalized likelihood $\tilde{q}(\bm{t} \mid \bm{h}, \, \theta)$ with our original theoretical framework. As a result, the bound in Theorem~\ref{theorem: CoT} is updated by substituting $\mathcal{M}\Delta_{\mathcal{P}_{\mathrm{CoT}}}$ with $2n\phi + 3m\mathcal{M}(\varphi + \Delta_{\mathcal{P}_{\mathrm{CoT}}})$. The proof are deferred to Appendix~\ref{Appendix: more discussions}.




\newpage

\appendix
\section{Appendices Structure}
In Appendix~\ref{Appendix: properties of ambiguity}, we establish six fundamental properties of ambiguity, ranging from Proposition~\ref{proposition: ε2 < ε1 (without h)} to Proposition~\ref{proposition: A ≤ 1 - E}. Appendix~\ref{Appendix: pretraining} provides the proofs for Theorem~\ref{Theorem: pretraining} and Corollary~\ref{corollary: pretraining}; within this appendix, generalization error and memorization capacity are discussed in Sections~\ref{section: generlization} and~\ref{Section: approximation}, respectively, to maintain the conciseness of the main argument. Finally, Appendices~\ref{Appendix: Proof for In-Context Learning} and~\ref{Appendix: Proof for Chain-of-Thought} present the proofs for Theorem~\ref{theorem: ICL} and Theorem~\ref{theorem: CoT}.
\section{Properties of Ambiguity}\label{Appendix: properties of ambiguity}
In this section, we provide the proofs for Propositions~\ref{proposition: ε2 < ε1}, \ref{proposition: ε of the compositional prompt}, and \ref{proposition: A ≤ 1 - E}. We note that these results are the generalized versions of Propositions~\ref{proposition: ε2 < ε1 (without h)}, \ref{proposition: ε of the compositional prompt (without h)}, and \ref{proposition: A ≤ 1 - E (without h)}, respectively. 
\begin{mypro}{\ref{proposition: ε2 < ε1}}
    For any given context $\bm{x}$ and $\bm{h}$ with $\mathcal{L}(\bm{x}) + \mathcal{L}(\bm{h}) \leq n$, if $\Theta_1 \subseteq \Theta_2$, then 
    \begin{align*}
        \mathcal{A}_{\Theta_2}^{\bm{h}}(\bm{x}) \leq \mathcal{A}_{\Theta_1}^{\bm{h}}(\bm{x}).
    \end{align*}
\end{mypro}
\begin{proof}
    The desire conclusion stems from
    \begin{align*}
        \mathcal{A}_{\Theta_2}^{\bm{h}}(\bm{x}) = 1 - \max_{\theta \in \Theta_2}q(\theta \mid \bm{h}, \bm{x}) \leq 1 - \max_{\theta \in \Theta_1}q(\theta \mid \bm{h}, \bm{x}) = \mathcal{A}^{\bm{h}}_{\Theta_1}(\bm{x}).
    \end{align*}
\end{proof}

\begin{mypro}{\ref{proposition: ε of the compositional prompt}}
Let $\bm{h}$ be any historical context and $\bm{x}_1$ be any prompt, then
\begin{align*}
    \E_{\bm{x}_2\sim q(\cdot\mid\, \bm{h}\, \circ \,\bm{x}_1)}\big\{\mathcal{A}_\Theta^{\bm{h}}(\bm{x}_1\circ \bm{x}_2)\big\}\ \le\ \mathcal{A}_\Theta^{\bm{h}}(\bm{x}_1).
\end{align*}
Furthermore, assume $\theta_{\bm{x}_1}^{\bm{h}} = \theta_{\bm{x}_1 \,\circ\, \bm{x}_2}^{\bm{h}}$ and denote it by $\theta^{\bm{h}}$. If $q(\bm{x}_2\mid \bm{h} \circ \bm{x}_1,\theta^{\bm{h}})\ \ge\ q(\bm{x}_2\mid \bm{h} \circ \bm{x}_1,\theta)$ for all $\theta\in\Theta$, then
  \begin{align*}
      \mathcal{A}_\Theta^{\bm{h}}(\bm{x}_1\circ \bm{x}_2)\ \le\ \mathcal{A}_\Theta^{\bm{h}}(\bm{x}_1).
  \end{align*} 
\end{mypro}

\begin{proof}
By the law of total expectation,
\begin{align*}
    \E_{\bm{x}_2 \sim q(\cdot \mid \,\bm{h} \,\circ \,\bm{x}_1)}\big\{q(\theta\mid \bm{h} \circ\bm{x}_1 \circ \bm{x}_2)\big\} = q(\theta\mid \bm{h} \circ \bm{x}_1). 
\end{align*}
This is based on the fact that operator is a deterministic function.

Therefore, by Jensen's inequality, we have
\begin{align*}
    \E_{\bm{x}_2 \sim q(\cdot \mid \, \bm{h} \, \circ \, \bm{x}_1)}\big\{\max_{\theta \in \Theta} q(\theta\mid \bm{h} \circ \bm{x}_1 \circ \bm{x}_2)\big\} &\geq \max_{\theta \in \Theta}\E_{\bm{x}_2 \sim q(\cdot \mid \, \bm{h} \, \circ\,  \bm{x}_1)}\big\{ q(\theta\mid \bm{h} \circ \bm{x}_1\circ\bm{x}_2)\big\} \\
    &= \max_{\theta \in \Theta}q(\theta\mid \bm{x}_1).
\end{align*}

Moreover, by the definition of $\mathcal{A}^{\bm{h}}_\Theta(\bm{x}_1 \circ \bm{x}_2)$, along with the Bayes' rule and the chain factorization,
\begin{align*}
    q(\theta^{\bm{h}}\mid \bm{h} \circ \bm{x}_1\circ{\bm{x}}_2) &=
\frac{q(\theta^{\bm{h}}\mid \bm{h} \circ \bm{x}_1)\,q(\bm{x}_2\mid \bm{h} \circ \bm{x}_1,\theta^{\bm{h}})}{\sum_{\theta\in\Theta} q(\theta\mid \bm{h} \circ \bm{x}_1)\,q(\bm{x}_2\mid \bm{h} \circ \bm{x}_1,\theta)} \\
&\geq \,  \frac{q(\theta^{\bm{h}} \mid \bm{h} \circ \bm{x}_1)}{\sum_{\theta\in\Theta} q(\theta\mid \bm{h} \circ \bm{x}_1)} \\
&= \, q(\theta^{\bm{h}} \mid \bm{x}_1),
\end{align*}
which turns out $\mathcal{A}^{\bm{h}}_\Theta(\bm{x}_1 \circ \bm{x}_2)\, \leq \,\mathcal{A}^{\bm{h}}_\Theta(\bm{x}_1)$ by the definition~\ref{def: dominated task and ambiguity}.
\end{proof}

\begin{mypro}{\ref{proposition: A ≤ 1 - E}}
Let $\mathcal{E}^{\bm{h}}(\bm{x}) := -\sum_{\theta \in \Theta} q(\theta\mid \bm{h} \circ \bm{x})\,\log q(\theta\mid \bm{h} \circ \bm{x})$ be the Shannon entropy of the posterior distribution $q(\theta \mid \bm{h} \circ \bm{x})$. Then
\begin{align*}
    \mathcal{A}_{\Theta}^{\bm{h}}(\bm{x}) \;\leq \; 1 - \exp\big(-\mathcal{E}^{\bm{h}}(\bm{x})\big).
\end{align*}
\end{mypro}

\begin{proof}
Since $q(\theta\mid \bm{h} \circ \bm{x})\le 1 - \mathcal{A}_\Theta^{\bm{h}}(\bm{x})$ for all $\theta$ due to the definition of $\mathcal{A}^{\bm{h}}_\Theta(\bm{x})$, we know
\begin{align*}
    \mathcal{E}^{\bm{h}}(\bm{x}) &= -\sum_{\theta \in \Theta} q(\theta\mid \bm{h} \circ \bm{x})\,\log q(\theta\mid \bm{h} \circ\bm{x}) \\&\geq -\sum_{\theta \in \Theta} q(\theta\mid \bm{h} \circ \bm{x})\,\log \big(1 - \mathcal{A}_\Theta^{\bm{h}}(\bm{x})\big) \\
    &= -\log \big(1 - \mathcal{A}_\Theta^{\bm{h}}(\bm{x})\big).
\end{align*}
Hence $1 - \mathcal{A}^{\bm{h}}_\Theta(\bm{x}) \geq \exp\big(-\mathcal{E}^{\bm{h}}(\bm{x})\big)$, which is equivalent to 
\begin{align*}
    \mathcal{A}^{\bm{h}}_\Theta(\bm{x}) \;\le\; 1 - \exp\big(-\mathcal{E}^{\bm{h}}(\bm{x})\big).
\end{align*}
\end{proof}
\section{Pretraining via Auto-Regression}\label{Appendix: pretraining}
In this section, we provide a rigorous proof for Theorem~\ref{Theorem: pretraining}, which necessitates a comprehensive analysis of the pretraining error. To this end, we first introduce several essential technical lemmas and then present the proof of our main result. For notational convenience, given two non-negative sequences $(a_i)$ and $b_i$, we write $A \lesssim B$ to denote that $a_i \leq Ab_i$ holds for some constant $A > 0$. We note that the proof of Theorem~\ref{Theorem: pretraining} requires an analysis of generalization error and memorization capacity; to ensure the conciseness of the main line of proof, we have deferred these proofs to Appendix~\ref{section: generlization} and Appendix~\ref{Section: approximation}, respectively.
\subsection{Technical Lemmas}
\begin{lemma}[\cite{Maurer2016contraction}, Corollary 4]\label{lemma: contraction}
    Let $\mathcal{X}$ be any set, $(x_1, \ldots, x_{n'}) \in \mathcal{X}^n$, let $F$ be a class of functions $f : \mathcal{X} \to \ell_2$ and let $h_i: \ell_2 \to \R$ have Lipschitz norm $L$. Then
    \begin{align*}
        \E\sup_{f \in F}\sum_i \sigma_ih_i\big(f(x_i)\big) \leq \sqrt{2}L\E\sup_{f \in F}\sum_{i,k}\sigma_{ik}f_k(x_i)
    \end{align*}
    where $\sigma_{ik}$ is an independent doubly indexed Rademacher sequence and $f_k(x_i)$ is the $k$-th component of $f(x_i)$. 
\end{lemma}

\begin{lemma}[Khintchine inequality]\label{lemma: Khintchine inequality}
    Let $\{\sigma_i\}_{i \in [n']}$ be a sequence with Rademacher distribution. Let $0 < u < \infty$ and let $x_1, \ldots, x_{n'} \in \R$. Then
    \begin{align*}
        \Big(\E\Big| \sum_{i=1}^{n'}\sigma_ix_i\Big|^u\Big)^{1/u} \leq C_u\Big(\sum_{i=1}^{n'}|x_i|^2\Big)^{1/2}
    \end{align*}
    for some constants $C_u > 0$. In particular, $C_u = 1$ when $0 < u \leq 2$.
\end{lemma}

\begin{lemma}[Pinsker's inequality]\label{lemma: Pinsker's inequality}
    If $\P$ and $\mathbb{Q}$ are two probability distributions on a measurable space $(\Omega, \Sigma)$, then
    \begin{align*}
        \mathrm{TV}(\P, \mathbb{Q}) \leq \sqrt{\frac{1}{2}D_{KL}(\P \,\|\, \mathbb{Q})}.
    \end{align*}
\end{lemma}

\subsection{Main Results}
Now let us prove Theorem~\ref{Theorem: pretraining}, which is stated as follows:
\begin{mythm}{\ref{Theorem: pretraining}}
    If Assumption~\ref{assumption: separable token representation} holds, then for any token $\bm{t} \in \mathcal{V}$ and $\bm{h} \in \R^{d \times n}$, by setting $W = \mathcal{O}(|\mathcal{V}|^{n+1}), D = 1$, we have
    \begin{align*}
        \sup_{(\bm{t}, \bm{h}) \sim q(\bm{t} \mid \bm{h})}\Big|\hat{p}(\bm{t} \mid \bm{h}) - q(\bm{t} \mid \bm{h})\Big| \leq \mathcal{O}\Big(\frac{|\mathcal{V}|^2M}{\sqrt[4]{Nn}}\Big) + \mathcal{O}\Big(\sqrt{\frac{\ln(1/\delta)}{Nn}}\Big)
    \end{align*}
    holds with probability at least $1 - \delta$ for any $\delta \in (0,1)$.
\end{mythm}
\begin{proof}
Recall the population risk is defined by
\begin{align*}
    \mathcal{R}(p) = \E_{(\bm{t}, \bm{h}) \sim q(\bm{t}, \bm{h})}\Big\{\log \frac{q(\bm{t} \mid \bm{h})}{p(\bm{t} \mid \bm{h})}\Big\},
\end{align*}
The corresponding risk at the sample level can be written as:
\begin{align*}
    \widehat{\mathcal{R}}(p) = \sum_{i=1}^N\sum_{j=1}^n\log\frac{q(\bm{t}_j^{(i)} \mid \bm{h}^{(i)}_j)}{p(\bm{t}_j^{(i)} \mid \bm{h}^{(i)}_j)}
\end{align*}
In this context, let $p$ be any element in $\mathcal{F}$, then
\begin{align*}
    \mathcal{R}(\hat{p}) &\leq \mathcal{R}(\hat{p}) - \widehat{\mathcal{R}}(\hat{p}) + \widehat{\mathcal{R}}(\hat{p}) - \mathcal{R}(p)  + \mathcal{R}(p)  \\
    &\leq \mathcal{R}(\hat{p}) - \widehat{\mathcal{R}}(\hat{p}) + \widehat{\mathcal{R}}(p) - \mathcal{R}(p)  + \mathcal{R}(p) \\
    &\leq 2\sup_{p \in \mathcal{F}}\big|\mathcal{R}(p) - \widehat{\mathcal{R}}(p)\big| + \mathcal{R}(p)
\end{align*}
Due to the arbitrariness of $p \in \mathcal{F}$, Taking infimum on both side turns out:
\begin{align}\label{eq: R(hatp) ≤ 2sup + inf}
    \mathcal{R}(\hat{p}) \leq 2\sup_{p \in \mathcal{F}}\big|\mathcal{R}(p) - \widehat{\mathcal{R}}(p)\big| +  \inf_{p \in \mathcal{F}}\mathcal{R}(p).
\end{align}
We furthermore define the generalization error and approximation error as
\begin{align*}
    \mathcal{E}_{\mathrm{sta}} = \E_{D_N}\Big\{\sup_{p \in \mathcal{F}}\Big|\widehat{\mathcal{R}}(p) -  \mathcal{R}(p)\Big|\Big\}, \quad \mathcal{E}_{\mathrm{app}} = \inf_{p \in \mathcal{F}}\mathcal{R}(p).
\end{align*}
Taking expectation on both sides of~\eqref{eq: R(hatp) ≤ 2sup + inf} obtains
\begin{align*}
       \E_{D_N}\big\{\mathcal{R}(\hat{p})\big\} = 2\mathcal{E}_{\mathrm{sta}} + \mathcal{E}_{\mathrm{app}}.
\end{align*}

Wherein, the approximation error $\mathcal{E}_{\mathrm{app}} = \inf_{p \in \mathcal{F}}\mathcal{R}(p)$ vanishes provided the condition that the width of Transformer $W \geq \mathcal{O}\big(|\mathcal{V}|^{n+1}\big)$ and its depth $D \geq 1$. We provide a formal proof in Theorem~\ref{theorem: memorization} in Section~\ref{Section: approximation}. Regarding the generalization error, we utilize the standard symmetrization technique by constructing a ghost dataset $D'_N = \big\{(\bm{t}^{'(i)}_j, \bm{h}^{'(i)}_j)\big\}_{i \in [N], j \in [n]}$ that is independent and identically distributed (i.i.d.) with respect to $D_N$.
\begin{align*}
    \mathcal{E}_{\mathrm{sta}} &= \E_{D_N}\Big\{\sup_{p \in \mathcal{F}}\Big|\widehat{\mathcal{R}}(p) - \mathcal{R}(p)\Big|\Big\}\\
    &= \E_{D_N}\Big[\sup_{p \in \mathcal{F}}\Big|\widehat{\mathcal{R}}(p) - \E_{D_N^\prime}\big\{\widehat{\mathcal{R}}(p)\big\}\Big|\Big]\\
    &= \E_{D_N, D_N', \bm{\xi}}\Big[\sup_{p \in \mathcal{F}}\Big|\frac{1}{Nn}\sum_{i=1}^N\sum_{j=1}^n\xi_{ij}\Big\{\log\frac{q(\bm{t}^{(i)}_j \mid \bm{h}^{(i)}_j)}{p(\bm{t}^{(i)}_j \mid \bm{h}^{(i)}_j)} - \log\frac{q(\bm{t}^{'(i)}_j \mid \bm{h}^{'(i)}_j)}{p(\bm{t}^{'(i)}_j \mid \bm{h}^{'(i)}_j)}\Big\}\Big|\Big]
\end{align*}
where $\xi_{ij}, i \in [N], j \in [n]$ are i.i.d Rademacher random variables (i.e. $\P(\sigma_{ij} = 1) = 1/2$ $\P(\sigma_{ij} = -1) = 1/2$). The last equality is derived from its symmetrization. Furthermore, by triangular inequality we have:
\begin{align*}
    &\E_{D_N, D_N', \bm{\xi}}\Big[\sup_{p \in \mathcal{F}}\Big|\frac{1}{Nn}\sum_{i=1}^N\sum_{j=1}^n\xi_{ij}\Big\{\log\frac{q(\bm{t}^{(i)}_j \mid \bm{h}^{(i)}_j)}{p(\bm{t}^{(i)}_j \mid \bm{h}^{(i)}_j)} - \log\frac{q(\bm{t}^{'(i)}_j \mid \bm{h}^{'(i)}_j)}{p(\bm{t}^{'(i)}_j \mid \bm{h}^{'(i)}_j)}\Big\}\Big|\Big] \\
    &\leq 2\E_{D_N, \bm{\xi}}\Big[\sup_{p \in \mathcal{F}}\Big|\frac{1}{Nn}\sum_{i=1}^N\sum_{j=1}^n\xi_{ij}\log\Big\{\frac{q(\bm{t}^{(i)}_j \mid \bm{h}^{(i)}_j)}{p(\bm{t}^{(i)}_j \mid \bm{h}^{(i)}_j)}\Big\}\Big|\Big] \\
    &\leq 2\E_{D_N, \bm{\xi}}\Big[\sup_{p \in \mathcal{F}}\Big|\frac{1}{Nn}\sum_{i=1}^N\sum_{j=1}^n\xi_{ij}\Big\{\log q(\bm{t}^{(i)}_j \mid \bm{h}^{(i)}_j) - \log p(\bm{t}^{(i)}_j \mid \bm{h}^{(i)}_j)\Big\}\Big|\Big] \\
    &\leq  2\E_{D_N, \bm{\xi}}\Big\{\Big|\frac{1}{Nn}\sum_{i=1}^N\sum_{j=1}^n\xi_{ij}\log q(\bm{t}^{(i)}_j \mid \bm{h}^{(i)}_j)\Big|\Big\}+ 2\E_{D_N, \bm{\xi}}\Big\{\sup_{p \in \mathcal{F}}\Big|\frac{1}{Nn}\sum_{i=1}^N\sum_{j=1}^n\xi_{ij}\log p(\bm{t}^{(i)}_j \mid \bm{h}^{(i)}_j)\Big|\Big\}.
\end{align*}
The first yielded can be bounded by $\frac{\log(1/b)}{\sqrt{Nn}}$ by applying Khintchine inequality (Lemma~\ref{lemma: Khintchine inequality}) along with the assumption that $q(\bm{t} \mid \bm{h}) \geq b$:
\begin{align*}
    \E_{D_N, \bm{\xi}}\Big\{\Big|\frac{1}{Nn}\sum_{i=1}^N\sum_{j=1}^n\xi_{ij}\log q(\bm{t}^{(i)}_j \mid \bm{h}^{(i)}_j)\Big|\Big\} \leq \frac{1}{Nn}\sqrt{\sum_{i=1}^N\sum_{j=1}^n \big\{\log q(\bm{t}^{(i)}_j \mid \bm{h}^{(i)}_j)\big\}^2} \leq \frac{\log(1/b)}{\sqrt{Nn}}.
\end{align*}
As for the second term, by contraction principle of Rademacher complexity (\citep{Maurer2016contraction}), we know
\begin{align*}
    \E_{D_N, \bm{\xi}}\Big\{\sup_{p \in \mathcal{F}}\Big|\frac{1}{Nn}\sum_{i=1}^N\sum_{j=1}^n\xi_{ij}\log p(\bm{t}^{(i)}_j \mid \bm{h}^{(i)}_j)\Big|\Big\}\leq \sqrt{2}b^{-1}\E_{D_N, \bm{\xi}}\Big\{\sup_{p \in \mathcal{F}}\Big|\frac{1}{Nn}\sum_{i=1}^N\sum_{j=1}^n\xi_{ij}p(\bm{t}^{(i)}_j \mid \bm{h}^{(i)}_j)\Big|\Big\}
\end{align*}

Furthermore, by noting that the probability $p(\bm{t} \mid \bm{h})$ corresponds to a specific element of the vector-valued Transformer output $\bm{T}(\bm{h}) \in \R^{|\mathcal{V}|}$, we can treat the probability mapping as a composition of a 
$1$-Lipschitz coordinate projection and the underlying Transformer architecture. Specifically, since $p(\bm{t} \mid \bm{h}) = (\bm{T}(\bm{h}))_{\mathrm{idx}(\bm{t})}$, where $\mathrm{idx}(\bm{t})$ represents the index of token $\bm{t}$ in vocabulary, we can invoke the vector-valued contraction principle (Lemma~\ref{lemma: contraction}). By conditioning on the dataset $D_N$, each coordinate projection is a $1$-Lipschitz mapping. Summing over all possible coordinate function classes and exploiting the architectural symmetry (i.e., $\mathcal{T}_1 = \cdots = \mathcal{T}_{|\mathcal{V}|}$), we obtain:
\begin{align*}
    \E_{D_N, \bm{\xi}}\Big\{\sup_{p \in \mathcal{F}}\Big|\frac{1}{Nn}\sum_{i=1}^N\sum_{j=1}^n\xi_{ij}p(\bm{t}^{(i)}_j \mid \bm{h}^{(i)}_j)\Big|\Big\} \leq \sqrt{2}|\mathcal{V}|\E_{D_N, \bm{\xi}}\Big\{\sup_{T \in \mathcal{T}_1(W,D)}\frac{1}{Nn}\sum_{i=1}^N\sum_{j=1}^n\xi_{ij}T(\bm{h}^{(i)}_j)\Big\},
\end{align*}
where $\mathcal{T}_i(W,D):=\{{T}_i:\boldsymbol{T}\in\mathcal{T}(W,D)\}$ for each $i\in[|\mathcal{V}|]$.

Combing above with the result yielded in Theorem~\ref{theorem: generlization} in Section~\ref{section: generlization}, we eventually obtain:
\begin{align}
    \mathcal{E}_{\mathrm{sta}} &\lesssim |\mathcal{V}|\sqrt{|\mathcal{V}|d^2(d+M)(n^2+dn+|\mathcal{V}|d+|\mathcal{V}|M)}\frac{\ln Nn}{\sqrt{Nn}} = \mathcal{O}\Big(\frac{|\mathcal{V}|^2M}{\sqrt{Nn}}\Big).
\end{align}

On the other hand, we have
\begin{align*}
    \sup_{(\bm{t},\bm{h}) \sim q(\bm{t}, \bm{h})}\Big|\hat{p}(\bm{t} \mid \bm{h}) - q(\bm{t} \mid \bm{h})\Big| &\leq b^{-1}\Big\|\hat{p}(\bm{t}\mid \bm{h}) - q(\bm{t} \mid \bm{h})\Big\|_{L_1} \\
    &= b^{-1}\E_{\bm{h} \sim q(\bm{h})}\Big\{\Big\|\hat{p}(\bm{t} \mid \bm{h}) - q(\bm{t} \mid \bm{h})\Big\|_{L_1}\Big\},
\end{align*}
where the first inequality stems from
\begin{align*}
    \Big\|\hat{p}(\bm{t}\mid \bm{h}) - q(\bm{t} \mid \bm{h})\Big\|_{L_1} &= \sum_{(\bm{t}, \bm{h}) \sim q(\bm{t}, \bm{h})}q(\bm{t}, \bm{h})\Big\vert\hat{p}(\bm{t} \mid \bm{h}) - q(\bm{t} \mid \bm{h})\Big\vert \\
    &\geq b\sup_{(\bm{t}, \bm{h}) \sim q(\bm{t}, \bm{h})}\Big|\hat{p}(\bm{t}\mid\bm{h}) - q(\bm{t} \mid \bm{h})\Big|.
\end{align*}
Further applying Jensen's inequality and Pinsker's inequality (\ref{lemma: Pinsker's inequality}) yields
\begin{align*}
    \Big\{\E_{\bm{h} \sim q(\bm{h})}\Big\|p(\bm{t} \mid \bm{h}) - q(\bm{t} \mid \bm{h})\Big\|_{L_1}\Big\}^2 &\leq \E_{\bm{h} \sim q(\bm{h})}\Big\{\Big\|p(\bm{t} \mid \bm{h}) - q(\bm{t}\mid \bm{h})\Big\|_{L_1}^2\Big\} \\
    &=2\E_{\bm{h}\sim q(\bm{h})}\Big\{\mathrm{TV}\Big(p(\bm{t} \mid \bm{h}) , q(\bm{t}\mid \bm{h})\Big)^2\Big\} \\
    &\leq \E_{\bm{h} \sim q(\bm{h})}\Big\{D_{\mathrm{KL}}\Big(q(\bm{x}\mid \bm{h})\, \Big\| \,p(\bm{x}\mid \bm{h})\Big)\Big\}
\end{align*}
where the first inequality are derived from Jensen's inequality while the second one stems from Pinsker's inequality. Combining above results together obtains
\begin{align*}
    \Big[\E_{D_N}\Big\{\sup_{(\bm{t},\, \bm{h}) \sim q(\bm{t},\, \bm{h})}\Big|\hat{p}(\bm{t}\mid \bm{h}) - q(\bm{t} \mid \bm{h})\Big|\Big\}\Big]^2 &\lesssim \E_{D_N}\Big[\Big\{\sup_{(\bm{t}, \bm{h}) \sim q(\bm{t}, \bm{h})}\Big|\hat{p}(\bm{t}\mid \bm{h}) - q(\bm{t} \mid \bm{h})\Big|\Big\}^2\Big]
    \\ 
    &\lesssim \Big(2\mathcal{E}_{\mathrm{sta}} + \mathcal{E}_{\mathrm{app}}\Big) \lesssim \mathcal{O}\Big(\frac{|\mathcal{V}|^2M}{\sqrt{Nn}}\Big).
\end{align*}

Building on it, By invoking McDiarmid’s concentration inequality, we conclude that for any $\delta \in (0,1)$, the following inequality holds with probability at least $1 - \delta$:
\begin{align}\label{eq: sup < 1/√Nn}
    \sup_{(\bm{t}, \bm{h}) \sim q(\bm{t}, \bm{h})}\Big|\hat{p}(\bm{t} \mid \bm{h}) - q(\bm{t} \mid \bm{h})\Big| \leq \mathcal{O}\Big(\frac{|\mathcal{V}|^2M}{\sqrt[4]{Nn}}\Big) + \mathcal{O}\Big(\sqrt{\frac{\ln(1/\delta)}{Nn}}\Big).
\end{align}
\end{proof}
plugging the identity $M \leq |\mathcal{V}|^{n}$ into above completes the proof.

Finally, given a prompt $\bm{x} \in \R^{d \times n}$, as far as the probability generating a sentence $\bm{y}$ of length $r$, i.e., $\bm{y} \in \R^{d \times r}$, we have following upper bound based on~\eqref{eq: sup < 1/√Nn},
\begin{mycor}{\ref{corollary: pretraining}}
    If Assumption~\ref{assumption: separable token representation} holds, then for any response $\bm{y} \in \R^{d \times r}$ and query $\bm{x} \in \R^{d \times n}$, by setting $W = \mathcal{O}(|\mathcal{V}|^{n+1}), D = 1$, we have
    \begin{align*}
        \Big|\hat{p}(\bm{y} \mid \bm{x}) - q(\bm{y} \mid \bm{x})\Big| < \mathcal{O}\Big(\frac{|\mathcal{V}|^{n+2}r}{\sqrt[4]{Nn}}\Big) + \mathcal{O}\Big(\sqrt{\frac{r^2\ln(1/\delta)}{Nn}}\Big)
    \end{align*}
    holds with probability at least $1 - \delta$ for any $\delta \in (0,1)$.
\end{mycor}
\begin{proof}
\begin{align*}
    \sup_{(\bm{y}, \bm{x})}\Big|\hat{p}(\bm{y} \mid \bm{x}) - q(\bm{y} \mid \bm{x})\Big| \lesssim \mathcal{O}\Big(\frac{|\mathcal{V}|^{n+2}r}{\sqrt[4]{Nn}}\Big) + \mathcal{O}\Big(\sqrt{\frac{r^2\ln(1/\delta)}{Nn}}\Big)
\end{align*}
with probability at least $1 - \delta$, which is the conclusion presented at Corollary~\ref{corollary: pretraining}. Specifically,
\begin{align*}
    \Big| \hat{p}(\bm{y} \mid \bm{x}) - q(\bm{y} \mid \bm{x}) \Big| &= \Big| \prod_{l=1}^r \hat{p}(\bm{y}_{:,l} \mid \bm{x} \circ \bm{y}_{:, \prec l}) - \prod_{l=1}^r q(\bm{y}_{:,l} \mid \bm{x} \circ \bm{y}_{:,\prec l})\Big| \\
    &\leq \Big| \prod_{l=1}^r \hat{p}(\bm{y}_{:,l} \mid \bm{x} \circ \bm{y}_{:,\prec l}) - q(\bm{y}_{:,r} \mid \bm{x} \circ \bm{y}_{:,\prec r})\prod_{l=1}^{r-1}\hat{p}(\bm{y}_{:,l} \mid \bm{x} \circ \bm{y}_{:,\prec l}) \\
    &\quad+ q(\bm{y}_{:,r} \mid \bm{x} \circ \bm{y}_{:,\prec r})\prod_{l=1}^{r-1}\hat{p}(\bm{y}_{:,l} \mid \bm{x} \circ \bm{y}_{:,\prec l}) - \prod_{l=1}^r q(\bm{y}_{:,l} \mid \bm{x} \circ \bm{y}_{:,\prec l})\Big| \\
    &\leq \Big|\hat{p}(\bm{y}_{:,r} \mid \bm{x} \circ \bm{y}_{:, \prec r}) - q(\bm{y}_{:,r} \mid \bm{x} \circ \bm{y}_{:, \prec r})\Big|\prod_{l=1}^{r-1}\hat{p}(\bm{y}_{:,l} \mid \bm{x} \circ \bm{y}_{:, \prec l}) \\
    &\quad  + q(\bm{y}_{:,r} \mid \bm{x} \circ \bm{y}_{:,\prec r})\Big|\prod_{l=1}^{r-1}\hat{p}(\bm{y}_{:,l} \mid \bm{x} \circ \bm{y}_{:,\prec l}) - \prod_{l=1}^{r-1} q(\bm{y}_{:,l} \mid \bm{x} \circ \bm{y}_{:,\prec l})\Big| \\
    &\lesssim \mathcal{O}\Big(\frac{|\mathcal{V}|^{n+2}}{\sqrt[4]{Nn}}\Big) + \mathcal{O}\Big(\sqrt{\frac{\ln(1/\delta)}{Nn}}\Big)+ \Big|\prod_{l=1}^{r-1}\hat{p}(\bm{y}_{:,l} \mid \bm{x} \circ \bm{y}_{:,\prec l}) - \prod_{l=1}^{r-1} q(\bm{y}_{:,l} \mid \bm{x} \circ \bm{y}_{:,\prec l})\Big| \\
    &\lesssim \cdots \lesssim \mathcal{O}\Big(\frac{|\mathcal{V}|^{n+2}r}{\sqrt[4]{Nn}}\Big) + \mathcal{O}\Big(\sqrt{\frac{r^2\ln(1/\delta)}{Nn}}\Big),
\end{align*}
which completes the proof.
\end{proof}

\section{Proof for In-Context Learning}\label{Appendix: Proof for In-Context Learning}
\begin{mythm}{\ref{theorem: ICL}}
    If Assumption~\ref{assumption: separable token representation},~\ref{assumption: task consistency},~\ref{assumption: tasks of delimiter},~\ref{assumption: nearly markov} and~\ref{Assumption: Bounded Priori Ratio} hold, then for any response $\bm{y} \in \R^{d \times r}$ and query $\bm{x} \in \R^{d \times n}$, by setting $W =\mathcal{O}(|\mathcal{V}|^{n+1}), D = 1$, we have
    \begin{align*}
        \Big| \hat{p}(\bm{y} \mid \mathcal{P}_{\mathrm{ICL}}) - q(\bm{y} \mid \bm{x},\, \theta_{\bm{x}})\Big| \leq  \mathcal{O}\Big(\frac{|\mathcal{V}|^{n+2}r}{\sqrt[4]{Nn}}\Big) + \mathcal{O}\Big(\sqrt{\frac{r^2\ln(1/\delta)}{Nn}}\Big)+r\phi + (e^{2n\phi}\cdot c\cdot\eps)^m \mathcal{A}_\Theta(\bm{x})
    \end{align*}
    holds with probability at least $1 - \delta$ for any $\delta \in (0,1)$.
\end{mythm}
\begin{proof}
    Similar to the proving method presented in Section~\ref{section: LLMs' comprehension}, we first decompose the objective to be three parts: 
     \begin{align*}
        \Big|\hat{p}(\bm{y}\mid \mathcal{P}_{\mathrm{ICL}}) - q(\bm{y} \mid \bm{x},\,\theta_{\bm{x}})\Big| &\leq \Big|\hat{p}(\bm{y} \mid \mathcal{P}_{\mathrm{ICL}}) - q(\bm{y} \mid \mathcal{P}_{\mathrm{ICL}})\Big| + \Big|q(\bm{y}\mid \mathcal{P}_{\mathrm{ICL}}) - q(\bm{y} \mid \mathcal{P}_{\mathrm{ICL}},\, \theta_{\bm{x}})\Big| \\
        &+ \Big|q(\bm{y} \mid \mathcal{P}_{\mathrm{ICL}}, \,\theta_{\bm{x}}) - q(\bm{y} \mid \bm{x}, \,\theta_{\bm{x}})\Big|,
    \end{align*}
    where the final term is bounded by $\mathcal{O}(r\phi)$, which directly follows from Assumption~\ref{assumption: nearly markov}. Specifically,
    \begin{align*}
        \Big|q(\bm{y} \mid \mathcal{P}_{\mathrm{ICL}}, \,\theta_{\bm{x}}) - q(\bm{y} \mid \bm{x}, \,\theta_{\bm{x}})\Big| &= \Big| q(\bm{y} \mid \mathcal{P}_{\mathrm{ICL}}, \, \theta_{\bm{x}}) - q(\bm{y} \mid \bm{x},\, \theta_{\bm{x}}) \Big| \\
        &= q(\bm{y} \mid \bm{x},\, \theta_{\bm{x}}) \cdot \left| \frac{q(\bm{y} \mid \mathcal{P}_{\mathrm{ICL}},\, \theta_{\bm{x}})}{q(\bm{y} \mid \bm{x},\, \theta_{\bm{x}})} - 1 \right| \\
        &\leq e^{r\phi} - 1 = \mathcal{O}(r\phi),
    \end{align*}
    where the last inequality stems from the Assumption~\ref{assumption: nearly markov}:
    \begin{align*}
        \Big|\log q\big(\bm{t} \mid \bm{h}\circ \bm{t}_{delim} \circ \bm{s},\, \theta\big) - \log q\big(\bm{t} \mid \bm{s}, \, \theta\big)\Big| \leq \phi \Rightarrow e^{-\phi}\leq \frac{q(\bm{t} \mid \bm{h}\circ \bm{t}_{delim} \circ \bm{s},\, \theta)}{q(\bm{t} \mid \bm{s}, \, \theta)} \leq e^{\phi},
    \end{align*}
    combining with the fact $\frac{q(\bm{y} \,|\, \mathcal{P}_{\mathrm{ICL}},\, \theta_{\bm{x}})}{q(\bm{y} \,|\, \bm{x},\, \theta_{\bm{x}})} = \frac{\prod_{l=1}^rq(\bm{y}_{:,l} \,|\, \mathcal{P}_{\mathrm{ICL}}\,\circ\, \bm{y}_{:, \prec l},\, \theta_{\bm{x}})}{\prod_{l=1}^rq(\bm{y}_{:,l} \,|\, \mathcal{P}_{\mathrm{ICL}}\,\circ \,\bm{y}_{:, \prec l},\, \theta_{\bm{x}})}$ yields:
    \begin{align*}
        e^{-r\phi}\leq \frac{q(\bm{y} \mid \mathcal{P}_{\mathrm{ICL}},\, \theta_{\bm{x}})}{q(\bm{y} \mid \bm{x},\, \theta_{\bm{x}})} \leq e^{r\phi}
    \end{align*}
    On the other hand, according to the pretraining results derived in Section~\ref{section: pretraining}, the first term is bounded by $\mathcal{O}\Big(\frac{|\mathcal{V}|^{n+2}r}{\sqrt[4]{Nn}}\Big) + \mathcal{O}\Big(\sqrt{\frac{r^2\ln(1/\delta)}{Nn}}\Big)$ with probability at least $1 - \delta$. To address the second term, we observe that
    \begin{align*}
        \Big|q(\bm{y} \mid \mathcal{P}_{\mathrm{ICL}}) - q(\bm{y} \mid \mathcal{P}_{\mathrm{ICL}}, \, \theta_{\bm{x}})\Big| &= \Big|\sum_{\theta \in \Theta}q(\bm{y}\mid \mathcal{P}_{\mathrm{ICL}},\,\theta)\,q(\theta \mid \mathcal{P}_{\mathrm{ICL}}) - \sum_{\theta \in \Theta}q(\theta \mid \mathcal{P}_{\mathrm{ICL}})\,q(\bm{y} \mid \mathcal{P}_{\mathrm{ICL}}, \, \theta_{\bm{x}})\Big| \\
        &= \Big|\sum_{\theta \in \Theta}q(\theta\mid \mathcal{P}_{\mathrm{ICL}})\Big\{ q(\bm{y}\mid \mathcal{P}_{\mathrm{ICL}}, \,\theta) - q(\bm{y}\mid\mathcal{P}_{\mathrm{ICL}}, \, \theta_{\bm{x}})\Big\} \Big| \\
        &\leq \sum_{\theta\in \Theta}q(\theta \mid \mathcal{P}_{\mathrm{ICL}})\Big|q(\bm{y} \mid \mathcal{P}_{\mathrm{ICL}}, \,\theta) - q(\bm{y} \mid \mathcal{P}_{\mathrm{ICL}}, \,\theta_{\bm{x}})\Big| \\
        &= \sum_{\theta \neq \theta_{\bm{x}}}q(\theta\mid \mathcal{P}_{\mathrm{ICL}})\Big|q(\bm{y}\mid \mathcal{P}_{\mathrm{ICL}},\, \theta) - q(\bm{y} \mid \mathcal{P}_{\mathrm{ICL}}, \, \theta_{\bm{x}})\Big| \\
        &\leq \sum_{\theta \neq \theta_{\bm{x}}}\,q(\theta \mid \mathcal{P}_{\mathrm{ICL}}) = \sum_{\theta \neq \theta_{\bm{x}}}\,q(\mathcal{P}_{\mathrm{ICL}},\, \theta) / q(\mathcal{P}_{\mathrm{ICL}}) \\
        &\leq \sum_{\theta \neq \theta_{\bm{x}}}q(\mathcal{P}_{\mathrm{ICL}},\, \theta) / q(\mathcal{P}_{\mathrm{ICL}}, \,\theta_{\bm{x}}).
    \end{align*}
    Inserting the delimiter tasks turns out:
    \begin{align*}
        &\sum_{\theta \neq \theta_{\bm{x}}}q(\mathcal{P}_{\mathrm{ICL}},\, \theta) / q(\mathcal{P}_{\mathrm{ICL}}, \,\theta_{\bm{x}}) \\
        &= \sum_{\theta \neq \theta_{\bm{x}}}\frac{\sum_{\theta^{(1)}_{delim} \in \mathcal{D}}\cdots \sum_{\theta^{(m)}_{delim} \in \mathcal{D}}q(\mathcal{P}_{\mathrm{ICL}},\, \theta,\, \theta^{(1)}_{delim} , \cdots,\, \theta^{(m)}_{delim})}{\sum_{\theta^{(1)}_{delim} \in \mathcal{D}}\cdots \sum_{\theta^{(m)}_{delim} \in \mathcal{D}}q(\mathcal{P}_{\mathrm{ICL}},\, \theta_{\bm{x}},\, \theta^{(1)}_{delim}, \cdots,\, \theta^{(m)}_{delim})}\\
        &= \sum_{\theta \neq \theta_{\bm{x}}}\frac{\sum_{\theta^{(1)}_{delim} \in \mathcal{D}}\cdots \sum_{\theta^{(m)}_{delim} \in \mathcal{D}}q(\mathcal{P}_{\mathrm{ICL}} \mid \theta,\, \theta^{(1)}_{delim}, \cdots,\, \theta^{(m)}_{delim})q(\theta)\big\{\prod_{i=1}^mq(\theta^{(i)}_{delim})\big\}}{\sum_{\theta^{(1)}_{delim} \in \mathcal{D}}\cdots \sum_{\theta^{(m)}_{delim} \in \mathcal{D}}q(\mathcal{P}_{\mathrm{ICL}}\mid \theta_{\bm{x}},\, \theta^{(1)}_{delim}, \cdots,\, \theta^{(m)}_{delim})q(\theta_{\bm{x}})\big\{\prod_{i=1}^mq(\theta^{(i)}_{delim})\big\}}.
    \end{align*}
    We further utilize Assumption~\ref{assumption: nearly markov} to decompose the yielded likelihood: 
    \begin{align}\label{eq: q(P|θ) ≤ e^{nmΦ}}
        &q(\mathcal{P}_{\mathrm{ICL}} \mid \theta,\, \theta^{(1)}_{delim}, \cdots,\, \theta^{(m)}_{delim}) = q(\bm{x} \mid \mathcal{P}_{delim}^{\preceq m}, \, \theta)\prod_{i=1}^m\Big\{q(\bm{x}^{(i)}\circ \bm{y}^{(i)} \mid \mathcal{P}_{delim}^{\prec i},\, \theta)q(\bm{t}^{(i)}_{delim} \mid \mathcal{P}^{\preceq i},\,\theta_{delim}^{(i)})\Big\} \notag\\
        &= q(\bm{x} \mid \mathcal{P}_{delim}^{\preceq m}, \, \theta)\prod_{i=1}^mq(\bm{x}^{(i)}\circ \bm{y}^{(i)} \mid \mathcal{P}_{delim}^{\prec i},\, \theta)\1\{\theta^{(i)}_{delim} = \theta_{\bm{t}^{(i)}_{delim}}\} \notag\\
        &\leq e^{n\phi}\prod_{l = 1}^nq(\bm{x}_{:, l} \mid \bm{x}_{:, \prec l}, \, \theta)\prod_{i=1}^{m}q(\mathcal{P}^{\succeq i}_{\mathrm{ICL}} \mid \mathcal{P}^{\prec i}_{\mathrm{ICL}}, \, \theta)\1\{\theta^{(i)}_{delim} = \theta_{\bm{t}^{(i)}_{delim}}\} \notag\\
        &= e^{n\phi}q(\bm{x} \mid \theta) \prod_{i=1}^{m}q(\mathcal{P}^{\succeq i}_{\mathrm{ICL}} \mid \mathcal{P}^{\prec i}_{\mathrm{ICL}}, \, \theta)\1\{\theta^{(i)}_{delim} = \theta_{\bm{t}^{(i)}_{delim}}\} \notag\\
        &\leq e^{nm\phi}q(\bm{x} \mid \theta)\prod_{i=1}^mq(\bm{x}^{(i)} \circ \bm{y}^{(i)} \mid \theta)\1\{\theta^{(i)}_{delim} = \theta_{\bm{t}^{(i)}_{delim}}\} 
    \end{align}
    where we define $\mathcal{P}^{\succeq i} = (\bm{x}^{(i)}\circ \bm{y}^{(i)}\big) \circ \bm{t}_{delim}^{(1)} \circ\cdot \cdots \circ (\bm{x}^{(m)} \circ \bm{y}^{(m)}) \circ \bm{t}_{delim}^{(m)} \circ \bm{x} \in \R^{d \times n}$ and recall $\mathcal{P}^{\prec i}_{delim} = (\bm{x}^{(1)}\circ \bm{y}^{(1)}\big) \circ \bm{t}_{delim} \circ\cdot \cdots \circ (\bm{x}^{(i-1)} \circ \bm{y}^{(i-1)}) \circ \bm{t}_{delim} \in \R^{d \times n}$ with $\mathcal{P}^{\prec 1}_{delim} = \texttt{<SOS>}$. The second line is derived from Assumption~\ref{assumption: tasks of delimiter}, which implies that $q(\bm{t}^{(i)}_{delim} \mid \mathcal{P}^{\preceq i}, \theta^{(i)}_{delim})$ doesn't vanish if and only if $\theta^{(i)}_{delim} = \theta_{\bm{t}^{(i)}_{delim}}$. 

    Plugging the yielded results into~\eqref{eq: q(P|θ) ≤ e^{nmΦ}} turns out:
    \begin{align*}
        \sum_{\theta \neq \theta_{\bm{x}}}\frac{q\big(\mathcal{P}_{\mathrm{ICL}}, \, \theta\big)}{q\big(\mathcal{P}_{\mathrm{ICL}}, \,\theta_{\bm{x}}\big)} 
        \leq e^{2nm\phi} \frac{\sum_{\theta \neq \theta_{\bm{x}}}q(\theta)\,q(\bm{x} \mid \theta) \prod_{i=1}^m q(\bm{x}^{(i)} \circ \bm{y}^{(i)} \mid \theta)}{q(\theta_{\bm{x}})\,q(\bm{x} \mid \theta_{\bm{x}})\prod_{i=1}^m q(\bm{x}^{(i)} \circ \bm{y}^{(i)} \mid  \theta_{\bm{x}})}
    \end{align*}
    Multiplying the numerator by $q(\theta)^m$ and the denominator by $q(\theta_{\bm{x}})^m$, we further have
    \begin{align*}
        \sum_{\theta \neq \theta_{\bm{x}}}\frac{q\big(\mathcal{P}_{\mathrm{ICL}}, \, \theta\big)}{q\big(\mathcal{P}_{\mathrm{ICL}}, \,\theta_{\bm{x}}\big)} &\leq e^{2nm\phi}c^m\cdot\frac{\sum_{\theta \neq \theta_{\bm{x}}}q(\theta)^{m+1}\, q(\bm{x} \mid \theta) \prod_{i=1}^m q(\bm{x}^{(i)} \circ \bm{y}^{(i)} \mid \theta)}{q(\theta_{\bm{x}})^{m+1}\,q(\bm{x} \mid  \theta_{\bm{x}})\prod_{i=1}^m q(\bm{x}^{(i)} \circ \bm{y}^{(i)} \mid \theta_{\bm{x}})} \\
        &=  e^{2nm\phi}c^m\cdot \frac{\sum_{\theta \neq \theta_{\bm{x}}} q(\bm{x},\, \theta) \prod_{i=1}^m q(\bm{x}^{(i)} \circ \bm{y}^{(i)}, \theta)}{q(\bm{x},\, \theta_{\bm{x}})\prod_{i=1}^m q(\bm{x}^{(i)} \circ \bm{y}^{(i)},  \theta_{\bm{x}})}
    \end{align*}
    Finally, by the inequality that $\sum_i a_ib_i \leq (\sum_ia_i)(\sum_ib_i)$ for $a_i, b_i > 0$, we can turn out
    \begin{align*}
         \sum_{\theta \neq \theta_{\bm{x}}}\frac{q\big(\mathcal{P}_{\mathrm{ICL}}, \, \theta\big)}{q\big(\mathcal{P}_{\mathrm{ICL}}, \,\theta_{\bm{x}}\big)}&\leq e^{2nm\phi}c^m \cdot  \frac{\sum_{\theta \neq \theta_{\bm{x}}} q(\bm{x},\, \theta) \prod_{i=1}^m q(\bm{x}^{(i)} \circ \bm{y}^{(i)}, \theta)}{q(\bm{x},\, \theta_{\bm{x}})\prod_{i=1}^m q(\bm{x}^{(i)} \circ \bm{y}^{(i)},  \theta_{\bm{x}})} \\
         &\le e^{2nm\phi}c^m\cdot\prod_{i=1}^m\frac{\sum_{\theta \neq \theta_{\bm{x}}} q(\bm{x}^{(i)} \circ \bm{y}^{(i)},\, \theta)}{q(\bm{x}^{(i)} \circ \bm{y}^{(i)}, \, \theta_{\bm{x}})}\frac{\sum_{\theta \neq \theta_{\bm{x}}} q(\bm{x},\, \theta)}{q(\bm{x},\, \theta_{\bm{x}})} \\ 
         &= e^{2nm\phi}c^m\cdot\frac{\mathcal{A}_\Theta(\bm{x})}{1 - \mathcal{A}_\Theta(\bm{x})} \prod_{i=1}^m \frac{\mathcal{A}_\Theta(\bm{x}^{(i)} \circ \bm{y}^{(i)})}{1 - \mathcal{A}_\Theta(\bm{x}^{(i)} \circ\bm{y}^{(i)})}. 
    \end{align*}
Where the last inequality stems from the fact that
\begin{align*}
    \frac{\sum_{\theta \neq \theta_{\bm{x}}} q(\bm{x},\, \theta)}{q(\bm{x},\, \theta_{\bm{x}})} &= \frac{\sum_{\theta \neq \theta_{\bm{x}}}q(\theta \mid \bm{x})}{q(\theta_{\bm{x}}\mid \bm{x})} = \frac{\mathcal{A}_\Theta(\bm{x})}{1-\mathcal{A}_\Theta(\bm{x})}
\end{align*}
Combining the yielded results obtains
    \begin{align*}
        \Big| \hat{p}(\bm{y} \mid \mathcal{P}_{\mathrm{ICL}}) - q(\bm{y} \mid \bm{x},\, \theta_{\bm{x}})\Big| &\leq  \mathcal{O}\Big(\frac{|\mathcal{V}|^{n+2}r}{\sqrt[4]{Nn}}\Big) + \mathcal{O}\Big(\sqrt{\frac{r^2\ln(1/\delta)}{Nn}}\Big) + r\phi\\
        &\quad +  e^{2nm\phi}c^m\cdot\frac{\mathcal{A}_\Theta(\bm{x})}{1- \mathcal{A}_\Theta(\bm{x})}\prod_{i=1}^m\frac{\mathcal{A}_\Theta(\bm{x}^{(i)}\circ\bm{y}^{(i)})}{1 - \mathcal{A}_\Theta(\bm{x}^{(i)}\circ\bm{y}^{(i)})} \\
        &\leq  \mathcal{O}\Big(\frac{|\mathcal{V}|^{n+2}r}{\sqrt[4]{Nn}}\Big) + \mathcal{O}\Big(\sqrt{\frac{r^2\ln(1/\delta)}{Nn}}\Big) + r\phi + (e^{2n\phi}\cdot c\cdot\eps)^m \mathcal{A}_\Theta(\bm{x})
    \end{align*}
holds with probability at least $1 - \delta$, which completes the proof.
\end{proof}

\section{Proof for Chain-of-Thought}\label{Appendix: Proof for Chain-of-Thought}
We first decompose $|\hat{p}(\bm{Y} \mid \mathcal{P}_{\mathrm{CoT}}) - \tilde{q}(\bm{Y} \mid \bm{x}, \bm{\theta}^\star)|$ to be following four terms:
\begin{align}\label{eq: error decomposition for CoT}
        \Big|\hat{p}(\bm{Y}\mid \mathcal{P}_{\mathrm{CoT}}) - \tilde{q}(\bm{Y} \mid \bm{x},\, \bm{\theta}^\star)\Big| &\leq \Big|\hat{p}(\bm{Y} \mid \mathcal{P}_{\mathrm{CoT}}) - q(\bm{Y} \mid \mathcal{P}_{\mathrm{CoT}})\Big|  + \Big|q(\bm{Y} \mid \mathcal{P}_{\mathrm{CoT}}) - \tilde{q}(\bm{Y}\mid \mathcal{P}_{\mathrm{CoT}})\Big| \notag\\
        &\quad + \Big|\tilde{q}(\bm{Y} \mid \mathcal{P}_{\mathrm{CoT}}) - q(\bm{Y} \mid \mathcal{P}_{\mathrm{CoT}}, \, \bm{\theta}^\star)\Big| \notag \\
        &\quad + \Big|q(\bm{Y} \mid \mathcal{P}_{\mathrm{CoT}},\, \bm{\theta}^\star) - q(\bm{Y} \mid \bm{x},\, \bm{\theta}^\star)\Big|.
    \end{align}
We now proceed with the resulting decomposition term by term. First, as previously discussed, the first term can be bounded by $\mathcal{O}\Big(\frac{|\mathcal{V}|^{n+2}r}{\sqrt[4]{Nn}}\Big) + \mathcal{O}\Big(\sqrt{\frac{r^2\ln(1/\delta)}{Nn}}\Big)$ with probability at least $1 - \delta$ according to the results in Section~\ref{section: pretraining}.
\begin{align*}
    \Big|\hat{p}(\bm{y}\mid \mathcal{P}_{\mathrm{CoT}}) - q(\bm{y} \mid \mathcal{P}_{\mathrm{CoT}})\Big| \leq \mathcal{O}\Big(\frac{|\mathcal{V}|^{n+2}r}{\sqrt[4]{Nn}}\Big) + \mathcal{O}\Big(\sqrt{\frac{r^2\ln(1/\delta)}{Nn}}\Big)
\end{align*}
holds with probability at least $1 - \delta$.

As for the last term, by Assumption~\ref{assumption: nearly markov}, we have:
\begin{align*}
    \Big|q(\bm{y} \mid \mathcal{P}_{\mathrm{CoT}}, \,\theta_{\bm{x}}) - q(\bm{y} \mid \bm{x}, \,\bm{\theta}^\star)\Big| &= \Big| q(\bm{y} \mid \mathcal{P}_{\mathrm{CoT}}, \, \bm{\theta}^\star) - q(\bm{y} \mid \bm{x},\, \bm{\theta}^\star) \Big| \\
    &= q(\bm{y} \mid \bm{x},\, \bm{\theta}^\star) \cdot \left| \frac{q(\bm{y} \mid \mathcal{P}_{\mathrm{CoT}},\, \bm{\theta}^\star)}{q(\bm{y} \mid \bm{x},\, \bm{\theta}^\star)} - 1 \right| \\
    &\leq e^{r\phi} - 1 = \mathcal{O}(r\phi).
\end{align*}

For the second term in Eq.~\eqref{eq: error decomposition for CoT}, if we denote $\mathcal{M}$ satisfying $\min\{\tilde{q}(\mathcal{P}_{\mathrm{CoT}}), q(\mathcal{P}_{\mathrm{CoT}})\} \geq 1 / \mathcal{M}$, then we have following Lemma.
\begin{lemma}\label{lemma: |q - ~q| < MΔ} 
$\big|q(\bm{Y} \mid \mathcal{P}_{\mathrm{CoT}}) - \tilde{q}(\bm{Y}\mid \mathcal{P}_{\mathrm{CoT}})\big| \leq \mathcal{M}\Delta_{\mathcal{P}_{\mathrm{CoT}}}$.
\end{lemma}
\begin{proof}
    Firstly, we notice that
    \begin{align*}
        q(\bm{Y} \mid \mathcal{P}_{\mathrm{CoT}}) &= \sum_{\bm{\theta} \in \Theta^L}q(\bm{\theta} \mid \mathcal{P}_{\mathrm{CoT}}) q(\bm{Y} \mid \mathcal{P}_{\mathrm{CoT}},\, \bm{\theta}), \\
        \tilde{q}(\bm{Y} \mid \mathcal{P}_{\mathrm{CoT}}) &= \sum_{\bm{\theta} \in \Theta^L}\tilde{q}(\bm{\theta} \mid \mathcal{P}_{\mathrm{CoT}})q(\bm{Y} \mid \mathcal{P}_{\mathrm{CoT}},\, \bm{\theta}).
    \end{align*}
    Thus we have
     \begin{align*}
        \Big|q(\bm{Y} \mid \mathcal{P}_{\mathrm{CoT}}) - \tilde{q}(\bm{Y} \mid \mathcal{P}_{\mathrm{CoT}})\Big| &= \Big| \sum_{\bm{\theta} \in \Theta^L} q(\bm{\theta} \mid \mathcal{P}_{\mathrm{CoT}}) q(\bm{Y} \mid \mathcal{P}_{\mathrm{CoT}},\, \bm{\theta}) - \tilde{q}(\bm{\theta} \mid \mathcal{P}_{\mathrm{CoT}})q(\bm{Y} \mid \mathcal{P}_{\mathrm{CoT}}, \, \bm{\theta})\Big|  \\
        &= \sum_{\bm{\theta} \in \Theta^L}\Big|q(\bm{\theta} \mid \mathcal{P}_{\mathrm{CoT}}) - \tilde{q}(\bm{\theta} \mid \mathcal{P}_{\mathrm{CoT}})  \Big|q(\bm{Y} \mid \mathcal{P}_{\mathrm{CoT}},\, \bm{\theta}).
    \end{align*}
    According to Bayesian rule, we can transform the posterior probability into the likelihood:
    \begin{align*}
        q(\bm{\theta} \mid \mathcal{P}_{\mathrm{CoT}}) = \frac{q(\mathcal{P}_{\mathrm{CoT}}\mid \bm{\theta})q(\bm{\theta})}{q(\mathcal{P}_{\mathrm{CoT}})}, \quad q(\bm{\theta} \mid \mathcal{P}_{\mathrm{CoT}})=\frac{q(\mathcal{P}_{\mathrm{CoT}} \mid \bm{\theta})\tilde{q}(\bm{\theta})}{\tilde{q}(\mathcal{P}_{\mathrm{CoT}})}.
    \end{align*}
    Further plugging into above yields:
    \begin{align*}
        \Big|q(\bm{\theta} \mid \mathcal{P}_{\mathrm{CoT}}) - \tilde{q}(\bm{\theta} \mid \mathcal{P}_{\mathrm{CoT}})  \Big| &\leq \mathcal{M}\cdot\Big|q(\mathcal{P}_{\mathrm{CoT}}\mid \bm{\theta})q(\bm{\theta}) - q(\mathcal{P}_{\mathrm{CoT}} \mid \bm{\theta})\tilde{q}(\bm{\theta})\Big| \\
        &=  \mathcal{M}\cdot\sum_{\bm{\theta} \in \Theta^L}\Big|\tilde{q}(\bm{\theta}) - q(\bm{\theta})\Big|q(\mathcal{P}_{\mathrm{CoT}} \mid \bm{\theta}) \\
        &\leq \mathcal{M}\cdot\sum_{\bm{\theta} \in \mathscr{F}(\mathcal{P}_{\mathrm{CoT}})}\Big|\tilde{q}(\bm{\theta}) - q(\bm{\theta})\Big| = \mathcal{M}\Delta_{\mathcal{P}_\mathrm{CoT}},
    \end{align*}
    which completes the proof, where the first inequality stems from the definition of $\mathcal{M}$.
\end{proof}

Finally let's focus on the tired term in eq~\eqref{eq: error decomposition for CoT}, we also have following Lemma regarding it:
\begin{lemma}
    $\Big|\tilde{q}(\bm{Y} \mid \mathcal{P}_{\mathrm{CoT}}) - q(\bm{Y} \mid \mathcal{P}_{\mathrm{CoT}}, \, \bm{\theta}^\star)\Big| \leq C\cdot (e^{2n\phi} \cdot c_1\cdot\eps)^{mK}$
\end{lemma}
\begin{proof}
    Mirroring the proof of Theorem~\ref{theorem: ICL}, we have:
    \begin{align*}
        &\Big|\tilde{q}(\bm{Y}\mid \mathcal{P}_{\mathrm{CoT}}) - q(\bm{Y} \mid \mathcal{P}_{\mathrm{CoT}},\, \bm{\theta}^\star)\Big| \\
        &= \Big|\sum_{\bm{\theta} \in \Theta^L}q(\bm{Y} \mid \mathcal{P}_{\mathrm{CoT}}, \,\bm{\theta})\tilde{q}(\bm{\theta} \mid \mathcal{P}_{\mathrm{CoT}}) - \sum_{\bm{\theta} \in \Theta^L}\tilde{q}(\bm{\theta} \mid \mathcal{P}_{\mathrm{CoT}})q(\bm{Y} \mid \mathcal{P}_{\mathrm{CoT}},\, \bm{\theta}^\star)\Big| \\
        &= \Big|\sum_{\bm{\theta} \in \Theta^L}\tilde{q}(\bm{\theta} \mid \mathcal{P}_{\mathrm{CoT}})\Big\{ q(\bm{Y} \mid \mathcal{P}_{\mathrm{CoT}}, \, \bm{\theta}) - q(\bm{Y} \mid \mathcal{P}_{\mathrm{CoT}}, \, \bm{\theta}^\star)\Big\} \Big| \\
        &\leq \sum_{\bm{\theta}\in \Theta^L}\tilde{q}(\bm{\theta} \mid \mathcal{P}_{\mathrm{CoT}})\Big|q(\bm{Y} \mid \mathcal{P}_{\mathrm{CoT}}, \, \bm{\theta}) - q(\bm{Y} \mid \mathcal{P}_{\mathrm{CoT}}, \, \bm{\theta}^\star)\Big| \\
        &= \sum_{\bm{\theta} \neq \bm{\theta}^\star}\tilde{q}(\bm{\theta} \mid \mathcal{P}_{\mathrm{CoT}})\Big|q(\bm{Y} \mid \mathcal{P}_{\mathrm{CoT}},\, \bm{\theta}) - q(\bm{Y} \mid \mathcal{P}_{\mathrm{CoT}},\, \bm{\theta}^\star)\Big| \\
        &\leq \sum_{\bm{\theta} \neq  \bm{\theta}^\star}\tilde{q}(\bm{\theta} \mid \mathcal{P}_{\mathrm{CoT}}) = \sum_{\bm{\theta} \neq \bm{\theta}^\star}\tilde{q}(\mathcal{P}_{\mathrm{CoT}}, \, \bm{\theta}) / \tilde{q}(\mathcal{P}_{\mathrm{CoT}}) \\
        &\leq \underbrace{\sum_{\bm{\theta} \neq \bm{\theta}^\star}\tilde{q}(\mathcal{P}_{\mathrm{CoT}}, \, \bm{\theta}) / \tilde{q}(\mathcal{P}_{\mathrm{CoT}}, \, \bm{\theta}^\star)}_{\mathcal{S}}
    \end{align*}
    Furthermore, by inserting the missing components of $\vec{\bm{\theta}}$, we know: 
    \begin{align*}
        \mathcal{S} &= \sum_{\bm{\theta} \neq \bm{\theta}^\star}\frac{\tilde{q}(\mathcal{P}_{\mathrm{CoT}}, \, \bm{\theta})}{\tilde{q}(\mathcal{P}_{\mathrm{CoT}}, \, \bm{\theta}^\star)} = \sum_{\bm{\theta} \in \mathscr{S}(\mathcal{P}_{\mathrm{CoT}}) \setminus \{\bm{\theta}^\star\}} \frac{\mathscr{A}(\bm{\theta})}{\mathscr{A}(\bm{\theta}^\star)},
    \end{align*}
    where the numerator $\mathscr{A}(\bm{\theta})$ and denominator $\mathscr{A}(\bm{\theta}^\star)$ are defined by marginalizing over the demonstration and delimiter tasks as follows:
    \begin{align*}
        \mathscr{A}(\bm{\theta}) &= \sum_{\theta^{(1)} \in \Theta} \sum_{\theta^{(1)}_{delim} \in \mathcal{D}} \cdots \sum_{\theta^{(m)} \in \Theta} \sum_{\theta^{(i)}_{delim} \in \mathcal{D}} \sum_{\theta^{(0)} \in \Theta}  q(\mathcal{P}_{\mathrm{CoT}} \mid \vec{\bm{\theta}}) \, \tilde{q}(\vec{\bm{\theta}}), \\
        \mathscr{A}(\bm{\theta}^\star) &= \sum_{\theta^{(1)} \in \Theta} \sum_{\theta^{(1)}_{delim} \in \mathcal{D}} \cdots \sum_{\theta^{(m)} \in \Theta} \sum_{\theta^{(i)}_{delim} \in \mathcal{D}} \sum_{\theta^{(0)} \in \Theta} q(\mathcal{P}_{\mathrm{CoT}} \mid \vec{\bm{\theta}}^\star) \, \tilde{q}(\vec{\bm{\theta}}^\star),
    \end{align*}
    where $\vec{\bm{\theta}} = \theta^{(1)} \circ \bm{\theta} \circ \theta_{delim}^{(1)} \circ \cdots \circ \theta^{(m)} \circ \bm{\theta} \circ \theta_{delim}^{(m)} \circ \theta^{(0)}$ and $\vec{\bm{\theta}}^\star = \theta^{(1)} \circ \bm{\theta}^\star\circ \theta^{(1)} \circ \cdots \circ\theta^{(m)} \circ  \bm{\theta}^\star \circ \theta^{(m)}_{delim} \circ \theta^{(0)}$. This step admits us to conduct decomposition on the yielded likelihood using Eq~\eqref{eq: CoT likelihood decomposition}. Taking $q(\mathcal{P}_{\mathrm{CoT}} \mid \vec{\bm{\theta}})$ as example:
    \begin{align*}
        &q(\mathcal{P}_{\mathrm{CoT}} \mid \vec{\bm{\theta}}) \\
        &= q(\bm{x} \mid \mathcal{P}_{delim}^{\preceq m},\, \theta^{(0)})\prod_{i=1}^mq(\bm{x}^{(i)} \mid \mathcal{P}_{delim}^{\prec i},\,\theta^{(i)})q(\bm{t}^{(i)}_{delim} \mid \mathcal{P}^{\preceq i}, \, \theta^{(i)}_{delim})\prod_{j=1}^Lq(\bm{y}^{(i)}_j \mid \mathcal{P}_{delim}^{\prec i} \circ \bm{x}^{(i)}\circ\bm{y}_{\prec j}^{(i)},\, \theta_j)\\
        &\leq e^{nmL\phi}q(\vec{\bm{\theta}})\cdot q(\bm{x} \mid  \theta^{(0)})\prod\limits_{i=1}^mq(\bm{x}^{(i)} \mid \theta^{(i)})\1\{\theta_{delim}^{(i)} = \theta_{\bm{t}_{delim}^{(i)}}\}\prod\limits_{j=1}^Lq(\bm{y}^{(i)}_j \mid \bm{x}^{(i)} \circ \bm{y}_{\prec j}^{(i)},\, \theta_j)
    \end{align*}
    where the final inequality stems from Assumption~\ref{assumption: nearly markov}. Similarly, we also have:
    \begin{align*}
        q(\mathcal{P}_{\mathrm{CoT}} \mid \vec{\bm{\theta}}^\star) \geq e^{nmL\phi}q(\vec{\bm{\theta}})\cdot q(\bm{x} \mid  \theta^{(0)})\prod\limits_{i=1}^mq(\bm{x}^{(i)} \mid \theta^{(i)})\1\{\theta_{delim}^{(i)} = \theta_{\bm{t}_{delim}^{(i)}}\}\prod\limits_{j=1}^Lq(\bm{y}^{(i)}_j \mid \bm{x}^{(i)} \circ \bm{y}_{\prec j}^{(i)},\, \theta_j^\star)
    \end{align*}
    On the other hand, we assume the prior $\tilde{q}(\vec{\bm{\theta}})$ has following decomposition:
    \begin{align*}
        \tilde{q}(\vec{\bm{\theta}}) = \tilde{q}(\bm{\theta})q(\theta^{(0)} \mid \bm{\theta})\prod_{i=1}^m\tilde{q}(\theta^{(i)}\mid \bm{\theta})\tilde{q}(\theta^{(i)}_{delim})
    \end{align*}    
    Therefore
    \begin{align*}
    &\frac{\mathscr{A}(\bm{\theta})}{\mathscr{A}(\bm{\theta}^\star)}=\frac{\sum_{\theta^{(1)} \in \Theta} \sum_{\theta^{(1)}_{delim} \in \mathcal{D}} \cdots \sum_{\theta^{(m)} \in \Theta} \sum_{\theta^{(i)}_{delim} \in \mathcal{D}} \sum_{\theta^{(0)} \in \Theta}  q(\mathcal{P}_{\mathrm{CoT}} \mid \vec{\bm{\theta}}) \, \tilde{q}(\vec{\bm{\theta}})}{\sum_{\theta^{(1)} \in \Theta} \sum_{\theta^{(1)}_{delim} \in \mathcal{D}} \cdots \sum_{\theta^{(m)} \in \Theta} \sum_{\theta^{(i)}_{delim} \in \mathcal{D}} \sum_{\theta^{(0)} \in \Theta} q(\mathcal{P}_{\mathrm{CoT}} \mid \vec{\bm{\theta}}^\star) \, \tilde{q}(\vec{\bm{\theta}}^\star)} \\
    &\leq c_1e^{2nmL\phi}\cdot \frac{\sum\limits_{\theta^{(0)} \in \mathscr{S}(\bm{x})}\cdots\sum\limits_{\theta^{(m)} \in \mathscr{S}(\bm{x}^{(m)})}\prod\limits_{k=0}^m\tilde{q}(\theta^{(k)} \mid \bm{\theta})q(\bm{x} \mid  \theta^{(0)})\prod\limits_{i=1}^mq(\bm{x}^{(i)} \mid \theta^{(i)})\prod\limits_{j=1}^Lq(\bm{y}^{(i)}_j \mid \bm{x}^{(i)} \circ\bm{y}_{\prec j}^{(i)},\, \theta_j)}{\sum\limits_{\theta^{(0)} \in \mathscr{S}(\bm{x})}\cdots\sum\limits_{\theta^{(m)} \in \mathscr{S}(\bm{x}^{(i)})}\prod\limits_{k=0}^m\tilde{q}(\theta^{(k)} \mid \bm{\theta}^\star)q(\bm{x} \mid \theta^{(0)})\prod\limits_{i=1}^mq(\bm{x}^{(i)} \mid \theta^{(i)})\prod\limits_{j=1}^Lq(\bm{y}^{(i)}_j \mid \bm{x}^{(i)} \circ\bm{y}_{\prec j}^{(i)},\, \theta_j^\star)} \\
    &\leq c_1e^{2nmL\phi}c_2^{-(m+1)} \cdot \frac{\sum\limits_{\theta^{(0)} \in \mathscr{S}(\bm{x})}\cdots\sum\limits_{\theta^{(i)} \in \mathscr{S}(\bm{x}^{(i)})}\cdots\sum\limits_{\theta^{(m)} \in \mathscr{S}(\bm{x}^{(m)})}\prod\limits_{k=0}^m\tilde{q}(\theta^{(k)} \mid \bm{\theta})\prod\limits_{i=1}^m\prod\limits_{j=1}^Lq(\bm{y}^{(i)}_j \mid \bm{x}^{(i)} \circ\bm{y}_{\prec j}^{(i)},\, \theta_j)}{\sum\limits_{\theta^{(0)} \in \mathscr{S}(\bm{x})}\cdots\sum\limits_{\theta^{(i)} \in \mathscr{S}(\bm{x}^{(i)})}\cdots\sum\limits_{\theta^{(m)} \in \mathscr{S}(\bm{x}^{(m)})}\prod\limits_{k=0}^m\tilde{q}(\theta^{(k)} \mid \bm{\theta}^\star)\prod\limits_{i=1}^m\prod\limits_{j=1}^Lq(\bm{y}^{(i)}_j \mid \bm{x}^{(i)} \circ\bm{y}_{\prec j}^{(i)},\, \theta_j^\star)} \\
    &= c_1e^{2nmL\phi} c_2^{(m+1)} \cdot\frac{\prod\limits_{i=1}^m\prod\limits_{j=1}^Lq(\bm{y}^{(i)}_j \mid \bm{x}^{(i)} \circ\bm{y}_{\prec j}^{(i)},\, \theta_j)}{\prod\limits_{i=1}^m\prod\limits_{j=1}^Lq(\bm{y}^{(i)}_j \mid \bm{x}^{(i)} \circ\bm{y}_{\prec j}^{(i)},\, \theta_j^\star)}.
    \end{align*}
    Here we note the summations over $\theta^{(i)}_{delim}$ are all eliminated due to the existence of $\1\big\{\theta^{(i)}_{delim} = \theta_{\bm{t}^{(i)}_{delim}}\big\}$ and the fact that $q(\bm{t}^{(i)} \mid \bm{h}, \,\theta) \in \{0,1\}$. Furthermore, if we define $C = \frac{c_1e^{2n(L - K)\phi}c_2^{-(m+1)}}{1 - c_1\cdot \eps}$, then we have:
    \begin{align*}
        \mathcal{S} &\leq  c_1e^{2nmL\phi} c_2^{-(m+1)}\cdot\sum_{\bm{\theta} \in \mathscr{S}(\mathcal{P}_{\mathrm{CoT}}) \setminus \{\bm{\theta}^\star\}}\frac{\prod\limits_{i=1}^m\prod\limits_{j=1}^Lq(\bm{y}^{(i)}_j \mid \bm{x}^{(i)} \circ\bm{y}_{\prec j}^{(i)},\, \theta_j)}{\prod\limits_{i=1}^m\prod\limits_{j=1}^Lq(\bm{y}^{(i)}_j \mid \bm{x}^{(i)} \circ\bm{y}_{\prec j}^{(i)},\, \theta_j^\star)}\\
        &=  c_1e^{2nmL\phi} c_2^{-(m+1)}\cdot\sum_{r = K}^L\sum_{\bm{\theta} \in \mathscr{S}(\mathcal{P}_{\mathrm{CoT}}) \atop d_H(\bm{\theta, \bm{\theta}}^\star)=r }\frac{\prod\limits_{i=1}^m\prod\limits_{j=1}^Lq(\bm{y}^{(i)}_j \mid \bm{x}^{(i)} \circ\bm{y}_{\prec j}^{(i)},\, \theta_j)}{\prod\limits_{i=1}^m\prod\limits_{j=1}^Lq(\bm{y}^{(i)}_j \mid \bm{x}^{(i)} \circ\bm{y}_{\prec j}^{(i)},\, \theta_j^\star)}\\
        &\leq c_1e^{2nmL\phi} c_2^{-(m+1)}\cdot\sum_{r = K}^Lc_1^{mr}\sum_{\bm{\theta} \in \mathscr{S}(\mathcal{P}_{\mathrm{CoT}}) \atop d_H(\bm{\theta, \bm{\theta}}^\star) = r}\frac{\prod\limits_{i=1}^m\prod\limits_{j=1}^Lq(\theta_j \mid \bm{x}^{(i)} \circ\bm{y}_{\prec j}^{(i)},\, \bm{y}^{(i)}_j)}{\prod\limits_{i=1}^m\prod\limits_{j=1}^Lq(\theta_j^\star \mid \bm{x}^{(i)} \circ\bm{y}_{\prec j}^{(i)},\, \bm{y}^{(i)}_j)} \\
        &\leq c_1e^{2nmL\phi}c_2^{-(m+1)}\cdot \sum_{r = K}^L(c_1\cdot \eps)^{mr} \\
        &= c_1e^{2nmL\phi}c_2^{-(m+1)}\cdot\frac{(c_1\cdot\eps)^{mK}}{1 - c_1\cdot\eps} \\
        &= C\cdot (e^{2n\phi} \cdot c_1\cdot\eps)^{mK}
    \end{align*}
    where the first equation is derived from Assumption~\ref{assumption: K-separation} and the second inequality stems from Assumption~\ref{assumption: regularity}.
\end{proof}
Combining above decomposition and Lemmas turns out:
\begin{mythm}{\ref{theorem: CoT}}
    If Assumption~\ref{assumption: separable token representation},~\ref{assumption: tasks of delimiter},~\ref{assumption: nearly markov},~\ref{assumption: K-separation} and ~\ref{assumption: regularity} hold, by setting appropriate $W =\mathcal{O}(|\mathcal{V}|^{n+1}), D = 1$, we have
    \begin{align*}
        \Big| \hat{p}(\bm{y} \mid \mathcal{P}_{\mathrm{CoT}}) - \tilde{q}(\bm{y} \mid \bm{x},\, \bm{\theta}^\star)\Big| &\leq  \mathcal{O}\Big(\frac{|\mathcal{V}|^{n+2}r}{\sqrt[4]{Nn}}\Big) + \mathcal{O}\Big(\sqrt{\frac{r^2\ln(1/\delta)}{Nn}}\Big) + r\phi + \mathcal{M} \Delta_{\mathcal{P}_{\mathrm{CoT}}} \\
        &\quad +C\cdot (e^{2n\phi} \cdot c_1\cdot\eps)^{mK}
    \end{align*}
    with probability at least $1 - \delta$ for any $\delta \in (0,1)$.
\end{mythm}
\section{Memorization}\label{Section: approximation}
In this section, we dive into addressing the approximation error $\mathcal{E}_{\mathrm{app}}: = \inf_{p \in \mathcal{F}}\mathcal{R}(p)$. To show that this term vanishes, we need to construct a Transformer in $\mathcal{F}$ that exactly reproduces the true human distribution $q(\bm{t} \mid \bm{h})$ over all possible histories $\bm{h} \in \R^{d \times n}$ and tokens $\bm{t} \in \mathcal{V}$.

Given that the vocabulary and sequence length are finite, the set of possible histories is likewise finite. We aim to construct a Transformer $\bm{T}$ that implements the mapping $\bm{T}(\bm{h}^{(i)}) =  \bm{q}^{(i)}$ for $i \in [M]$, where $M$  is the total number of unique histories. Each history is denoted by $\bm{h}^{(i)} \in \R^{d \times n}$ in this Section, and its respective true distribution is defined as $\bm{q}^{(i)} = \big(q(t_1 \mid \bm{h}^{(i)}), q(t_2 \mid \bm{h}^{(i)}), \cdots, q(t_{|\mathcal{V}|} \mid \bm{h}^{(i)})\big) \in \Delta(\mathcal{V})$. Here, $\Delta(\mathcal{V})$ represents the probability simplex over the vocabulary $\mathcal{V}$.

To this end, we extend the memorization results for Transformers established by \citet{KajitsukaS24}—specifically Corollaries 1 and 2—to the masked Transformer setting. We begin by introducing several technical lemmas required for our proof.

\subsection{Technical Lemma}
The key to obtaining the desired conclusion lies in showing the separability of the Boltzmann operator:
\begin{align*}
\mathbf{boltz} : \mathbb{R}^{w} \to \mathbb{R},\ 
\bm{a} \mapsto \bm{a}^{\top} {\sigma}_{S}[\bm{a}],
\end{align*}
where $w \in [d]$.

The following lemma states that the Boltzmann operator is indeed separable, which is a generalization of \citet[Lemma 1]{KajitsukaS24}. While \citet[Lemma 1]{KajitsukaS24} requires that $\bm{a}$ and $\bm{b}$ are vectors of the same dimension, here we allow them to have different dimensions. This relaxation is the key to ensure that masked attention can realize contextual mapping.

\begin{lemma}\label{boltzmann separateness}
Let $\bm{a}\in\mathbb{R}^{u}, \bm{b}\in \mathbb{R}^v$ with $u,v \in [d]$ satisfy the following three conditions:
\begin{align}
|a_i|,|b_j|&\leq \rho,\quad i\in[u],j\in[v]; \label{boltzmann separateness-1} \\
|a_{i_1}-a_{i_2}|,|b_{j_1}-b_{j_2}|&>\varrho,\quad i_1\neq i_2\in[u],j_1\neq j_2\in[v]; \label{boltzmann separateness-2} \\
a_i=b_j \text{ or }|a_{i}-b_{j}|&>\varrho,\quad i\in[u],j\in[v]. \label{boltzmann separateness-3}
\end{align}
Suppose $\bm{b}$ is not a permutation of $\bm{a}$ and  $\varrho > 2 \log d + 3$. Then, the outputs of the Boltzmann operator satisfy 
\begin{align*}
\left| \mathbf{boltz}(\bm{a}) \right|,\left| \mathbf{boltz}(\bm{b}) \right| &\leq \rho, \\
\left| \mathbf{boltz}(\bm{a}) - \mathbf{boltz}(\bm{b}) \right| &> \varrho' := (\log d)^2 e^{-2\rho}.
\end{align*}
\end{lemma}
With Lemma~\ref{boltzmann separateness} in hand, we are able to construct a masked attention layer to realize contextual mapping. The proof of Lemma~\ref{boltzmann separateness} relies on the following technical lemma.

\begin{lemma}[\cite{KajitsukaS24}, Lemma 6]\label{basic lower bound}
Let $u \in \{2, 3, \cdots, d\}$. Let $\bm{a} , \bm{b}\in \mathbb{R}^u$ be two vectors such that $a_i = b_i$ for $i \in [u-1]$ and
\begin{align*}
    \max_{i \in [u-1]} a_i - \varrho > a_u > b_u
\end{align*}
with $t > \log u + 3$. Then  
\begin{align*}
    \mathbf{boltz}(\bm{b}) - \mathbf{boltz}(\bm{a}) > (a_u - b_u)(t + a_u - b_u - \log u - 1) \cdot \frac{e^{b_u}}{\sum_{i=1}^u e^{b_i}}.
\end{align*}
\end{lemma}

\begin{proof}[Proof of Lemma \ref{boltzmann separateness}]
This proof is based on the proof of \cite[Lemma 1]{KajitsukaS24}. The $\rho$-boundedness of Boltzmann  operator can be implied directly by the fact that it is a weighted average of the input vector. So we focus on showing the $\varrho'$-separateness of Boltzmann operator. To this end, since the Boltzmann
operator is permutation invariant, we assume without loss of generality 
\begin{align*}
a_1>a_2>\cdots>a_u,\quad b_1>b_2>\cdots>b_v.
\end{align*}
In the following, we consider three cases.

\textbf{Case 1:} $a_1\neq b_1$. 

Assume without loss of generality $a_1>b_1$. Since the Boltzmann operator can be regarded as weighted averaging, $ \mathbf{boltz}(\bm{b}) \leq b_1 $ always holds. Thus, it is enough to evaluate how much greater $ \mathbf{boltz}(\bm{a}) $ is than $ b_1 $. Let $ \overline{\bm{a}},\overline{\overline{\bm{a}}} $ be
\begin{align*}
\overline{\bm{a}} &= (a_1, a_1 - \varrho, \ldots, a_1 - \varrho)\in \mathbb{R}^d,
\end{align*}
Then, $\mathbf{boltz}(\bm{a}) \geq \mathbf{boltz}(\overline{\bm{a}})$ follows from the fact that Boltzmann operator can be regarded as weighted averaging. On the other hand,
\begin{align*}
    \mathbf{boltz}(\overline{\bm{a}}) &= \frac{a_1 e^{a_1} + (d-1)(a_1 - \varrho)e^{a_1 - \varrho}}{e^{a_1} + (d-1)e^{a_1 - \varrho}} \\
    &= \frac{a_1 + (d-1)(a_1 - \varrho)e^{-\varrho}}{1 + (d-1)e^{-\varrho}} \\
    &= a_1 - \frac{(d-1)\varrho e^{-\varrho}}{1 + (d-1)e^{-\varrho}}.
\end{align*}
Therefore, 
\begin{align*}
\mathbf{boltz}(\bm{a}) - \mathbf{boltz}(\bm{b}) &\geq \mathbf{boltz}(\overline{\bm{a}}) - b_1 \\
&> a_1 - \frac{(d-1)\varrho e^{-\varrho}}{1 + (d-1)e^{-\varrho}} - (a_1 - \varrho) \\
&= \varrho - \frac{(d-1)\varrho e^{-\varrho}}{1 + (d-1)e^{-\varrho}} \\
&= \frac{\varrho}{1 + (d-1)e^{-\varrho}} \\
&\geq \log d,
\end{align*}
where the last inequality follows from the assumption $ \varrho > 2\log d $. Note that the right-hand side is greater than $ (\log d)^2 e^{-2\rho} $ since following inequality holds:
\begin{align*}
2\rho\geq a_1-b_1>\varrho>\log(\log d).
\end{align*}

\textbf{Case 2:} There exists $1\leq k\leq\min\{u,v\} - 1$ such that for $i \in [k]$, $a_i = b_i$.

Assume without loss of generality $a_{k+1} > b_{k+1}$. Let $\underline{\bm{a}},\overline{\bm{b}},\overline{\overline{\bm{b}}}$ be
\begin{align*}
\underline{\bm{a}} &= (a_1, a_2, \ldots, a_k, a_{k+1}) \in \mathbb{R}^{k+1},\\
\overline{\bm{b}} &= (a_1, a_2, \ldots, a_k, b_{k+1}, b_{k+1}, \ldots, b_{k+1}) \in \mathbb{R}^d
\end{align*}
Using the fact that $\mathbf{boltz}(\cdot)$ is actually a weighted summation of input yields
\begin{align*}
    \mathbf{boltz}(\bm{a}) \leq \mathbf{boltz}(\underline{\bm{a}})
    \text{ and } \mathbf{boltz}(\bm{b}) \geq \mathbf{boltz}(\overline{\bm{b}}),
\end{align*}
implying we only need to justify $\mathbf{boltz}(\overline{\bm{b}}) - \mathbf{boltz}(\underline{\bm{a}}) \geq 0$. Let $\zeta_k$ and $\varsigma_k$ be
\begin{align*}
    \zeta_k = \sum_{l=1}^k a_l e^{a_l} \quad \text{and} \quad \varsigma_k = \sum_{l=1}^k e^{a_l}.
\end{align*}
Then, $\mathbf{boltz}(\overline{\bm{b}})$ can be decomposed into
\begin{align*}
\mathbf{boltz}(\overline{\bm{b}}) &= \frac{\zeta_k + (d-k)b_{k+1}e^{b_{k+1}}}{\varsigma_k + (d-k)e^{b_{k+1}}} \\
&= \frac{\zeta_k + b_{k+1}e^{b_{k+1}+\log(d-k)}}{\varsigma_k + e^{b_{k+1}+\log(d-k)}} \\
&= \frac{\zeta_k + (b_{k+1} + \log(d-k))e^{b_{k+1}+\log(d-k)}}{\varsigma_k + e^{b_{k+1}+\log(d-k)}} - \frac{\log(d-k) \cdot e^{b_{k+1}+\log(d-k)}}{\varsigma_k + e^{b_{k+1}+\log(d-k)}}.
\end{align*}
Therefore, the difference $\mathbf{boltz}(\overline{\bm{b}}) - \mathbf{boltz}(\underline{\bm{a}})$ can be written as
\begin{align}
\mathbf{boltz}(\overline{\bm{b}}) - \mathbf{boltz}(\underline{\bm{a}}) &=
\mathbf{boltz}(a_1, \ldots, a_k, b_{k+1} + \log(d-k)) - \mathbf{boltz}(\underline{\bm{a}}) \nonumber\\
&\quad - \frac{\log(d-k) \cdot e^{b_{k+1}+\log(d-k)}}{\varsigma_k + e^{b_{k+1}+\log(d-k)}}.\label{boltzmann separateness1}
\end{align}
According to Lemma \ref{basic lower bound}, the first two terms on the right-hand side are evaluated as
\begin{align*}
& \mathbf{boltz}(a_1, \ldots, a_k, b_{k+1} + \log(d - k)) - \mathbf{boltz}(\underline{\bm{a}}) \\
& > (a_{k+1} - b_{k+1} - \log(d - k))(\varrho + a_{k+1} - b_{k+1} - \log(d - k) - \log(k + 1) - 1) \cdot \frac{e^{b_{k+1} + \log(d - k)}}{\varsigma_k + e^{b_{k+1} + \log(d - k)}} \\
& > (\varrho - \log d)(2\varrho - 2\log d - 1) \cdot \frac{e^{b_{k+1} + \log(d - k)}}{\varsigma_k + e^{b_{k+1} + \log(d - k)}}.
\end{align*}
Since $\varrho > 2\log d + 3$ by assumption, the above inequality is further lower bounded by

\begin{align*}
& \mathbf{boltz}(a_1, \ldots, a_k, b_{k+1} + \log(d - k)) - \mathbf{boltz}(\underline{\bm{a}}) \\
& > (\varrho - \log d)(2\varrho - 2\log d - 1) \cdot \frac{e^{b_{k+1} + \log(d - k)}}{\varsigma_k + e^{b_{k+1} + \log(d - k)}} \\
& > (\log d + 3)(2\log \varrho + 5) \cdot \frac{e^{b_{k+1} + \log(d - k)}}{\varsigma_k + e^{b_{k+1} + \log(d - k)}}.
\end{align*}
Plugging the above inequality into \eqref{boltzmann separateness1} we see
\begin{align}
& \mathbf{boltz}(\overline{\bm{b}}) - \mathbf{boltz}(\underline{\bm{a}}) \nonumber\\
& = \mathbf{boltz}(a_1, \ldots, a_k, b_{k+1} + \log(d - k)) - \mathbf{boltz}(\underline{\bm{a}}) - \frac{\log(d - k) \cdot e^{b_{k+1} + \log(d - k)}}{\varsigma_k + e^{b_{k+1} + \log(d - k)}} \nonumber\\
& > \frac{e^{b_{k+1} + \log(d - k)}}{\varsigma_k + e^{b_{k+1} + \log(d - k)}} \big[\{(\log d + 3)(2\log d + 5) - \log(d - k)\}\big] \nonumber\\
& > \frac{e^{b_{k+1} + \log(d - k)}}{\varsigma_k + e^{b_{k+1} + \log(d - k)}} \cdot 2(\log d)^2.\label{boltzmann separateness2}
\end{align}
Lastly, notice that the following inequality follows from $\varrho$-separateness of $\bm{a}$ and $\bm{b}$ and the assumption that $\bm{a}$ has no duplicate token:
\begin{align*}
\varsigma_k + e^{b_{k+1} + \log(d - k)} &< \sum_{l=1}^{k+1} e^{a_l} \quad (\text{due to } a_{k+1} > b_{k+1} + \log(d - k)) \\
&< e^{a_1} \sum_{l=1}^{k+1} e^{-(l-1)\varrho} \\
&< 2e^{a_1} \quad (\text{due to } \varrho > \log 2).
\end{align*}
By using this inequality, \eqref{boltzmann separateness2} is lower bounded by

\begin{align*}
\mathbf{boltz}(\overline{\bm{b}}) - \mathbf{boltz}(\underline{\bm{a}}) 
&> \frac{e^{b_{k+1} + \log(d - k)}}{\varsigma_k + e^{b_{k+1} + \log(d - k)}} \cdot 2(\log d)^2 \\
&> \frac{e^{b_{k+1} + \log(d - k)}}{2e^{a_1}} \cdot 2(\log d)^2 \\
&> (\log d)^2 e^{-(a_1 - b_{k+1})}\\
&\geq(\log d)^2 e^{-2\rho},
\end{align*}
which implies $\mathbf{boltz}(\bm{b}) - \mathbf{boltz}(\bm{a}) > (\log d)^2 e^{-2\rho}$.

\textbf{Case 3:} For $i\in[\min\{u,v\}]$, $a_i=b_i$.

Assume without loss of generality $u>v$. Let 
\begin{align*}
\overline{\bm{b}} &= (a_1, a_2, \ldots, a_v, a_{v+1}-2\varrho) \in \mathbb{R}^{v+1}.
\end{align*}
By the same reason that $\mathbf{boltz}(\cdot)$ is essentially a weighted summation, we have 
\begin{align*}
\mathbf{boltz}(\bm{b})>\mathbf{boltz}(\overline{\bm{b}})>\mathbf{boltz}(\bm{a}).
\end{align*}
Hence $\mathbf{boltz}(\bm{b})-\mathbf{boltz}({\bm{a}})$ is lower bounded by $\mathbf{boltz}(\overline{\bm{b}})-\mathbf{boltz}(\bm{a})$. The derivation of a lower bound of $\mathbf{boltz}(\overline{\bm{b}})-\mathbf{boltz}(\bm{a})$ belongs to case 2, which completes the proof.
\end{proof}

In addition, we also need the following lemma.
\begin{lemma}[\cite{KajitsukaS24}, Lemma 3]\label{matrix construction}
Given a $(r_{\min}, r_{\max}, \eta)$-separated finite vocabulary $\mathcal{V} \subset \mathbb{R}^d$ with $r_{\min} > 0$. Then, for any $\kappa > 0$, there exists a unit vector $\bm{v} \in \mathbb{R}^d$ such that for any vectors $\bm{u}, \bm{u}' \in \mathbb{R}^s$ with
\begin{align*}
    |\bm{u}^\top \bm{u}'| = (|\mathcal{V}| + 1)^4 \frac{\pi d}{8} \frac{\kappa}{\eta r_{\min}},
\end{align*}
we have
\begin{gather*}
\left| \left( \bm{W}^{(K)} \bm{v}_a \right)^\top \left( \bm{W}^{(Q)} \bm{v}_c \right) - \left( \bm{W}^{(K)} \bm{v}_b \right)^\top \left( \bm{W}^{(Q)} \bm{v}_c \right) \right| > \kappa,\\
\frac{1}{\left( |\mathcal{V}| + 1 \right)^2} \sqrt{\frac{8}{\pi d}} \|\bm{v}_c\| \leq |\bm{v}^\top \bm{v}_c| \leq \|\bm{v}_c\|
\end{gather*}
for any $\bm{v}_a, \bm{v}_b, \bm{v}_c \in \mathcal{V}$ with $\bm{v}_a \neq \bm{v}_b$, where $\bm{W}^{(K)} = \bm{u}\bm{v}^\top \in \mathbb{R}^{s \times d}$ and $\bm{W}^{(Q)} = \bm{u}'\bm{v}^\top \in \mathbb{R}^{s \times d}$.
\end{lemma}

\begin{definition}[Feed-Forward Layer]
    The output $\bm{H} \in \R^{d \times n}$ of the masked self-attention layer at block $l$ is then passed to the feed-forward layer, which performs the following token-wise operation:
    \begin{align*}
        \mathcal{F}^{(\mathrm{FF})}_l(\bm{H})_{:,k} := \bm{H}_{:,k} + \bm{W}^{(2)}_l \sigma_R\Big(\bm{W}_l^{(1)}\bm{H}_{:,k} + \bm{b}_l^{(1)}\Big) + \bm{b}_l^{(2)} \in \R^{d}
    \end{align*}
    where $k \in [n], \bm{W}_l^{(1)} \in \R^{r \times d}$ and $\bm{W}_l^{(2)} \in \R^{d \times r}$ are weight matrices with hidden dimension $r$, and $\bm{b}_l^{(1)} \in \R^r$ and $\bm{b}_l^{(2)} \in \R^d$ are bias terms.
\end{definition}

\begin{lemma}[Residual elimination via ReLU decomposition]\label{lemma: residual elimination}
Let $\bm{H}_{:,k}\in\mathbb{R}^{d}$, 
$\bm{W}^{(1)}_l\in\mathbb{R}^{r\times d}$, 
$\bm{b}^{(1)}_l\in\mathbb{R}^{r}$, 
$\bm{W}^{(2)}_l\in\mathbb{R}^{d\times r}$, 
$\bm{b}^{(2)}_l\in\mathbb{R}^{d}$, 
and $\sigma_R(t)=\max\{t,0\}$ applied elementwise. Define the residual feed-forward sublayer

\begin{align*}
    \mathcal{F}^{(\mathrm{FF})}_l(\bm{H})_{:,k} = \bm{H}_{:,k} + \bm{W}^{(2)}_l\,\sigma_R\big(\bm{W}^{(1)}_l\bm{H}_{:,k}+\bm{b}^{(1)}_l\big) + \bm{b}^{(2)}_l \in\mathbb{R}^{d}.
\end{align*}

Then there exist widened parameters 
$\widetilde{\bm{W}}^{(1)}_l\in\mathbb{R}^{(r+2d)\times d}$, 
$\widetilde{\bm{b}}^{(1)}_l\in\mathbb{R}^{r+2d}$, 
$\widetilde{\bm{W}}^{(2)}_l\in\mathbb{R}^{d\times(r+2d)}$, 
$\widetilde{\bm{b}}^{(2)}_l\in\mathbb{R}^{d}$ such that

\begin{align*}
    \mathcal{F}^{(\mathrm{FF})}_l(\bm{H})_{:,k}=\widetilde{\bm{W}}^{(2)}_l\,\sigma_R\big(\widetilde{\bm{W}}^{(1)}_l\bm{H}_{:,k}+\widetilde{\bm{b}}^{(1)}_l\big)+\widetilde{\bm{b}}^{(2)}_l,
\end{align*}

i.e., the residual block equals a single non-residual ReLU MLP of width $r+2d$.
\end{lemma}

\begin{proof}
Define

\begin{align*}
    \widetilde{\bm{W}}^{(1)}_l:=
\begin{pmatrix}
\bm{W}^{(1)}_l \\
\bm{I}_d \\
-\bm{I}_d
\end{pmatrix}
\in\mathbb{R}^{(r+2d)\times d},
\quad
\widetilde{\bm{b}}^{(1)}_l
:=
\begin{pmatrix}
\bm{b}^{(1)}_l\\
\bm{0} \\
\bm{0}
\end{pmatrix}
\in\mathbb{R}^{r+2d},
\end{align*}

\begin{align*}
    \widetilde{\bm{W}}^{(2)}_l
:=
\begin{pmatrix}
\bm{W}^{(2)}_l & \bm{I}_d & -\bm{I}_d
\end{pmatrix}
\in\mathbb{R}^{d\times(r+2d)},
\quad
\widetilde{\bm{b}}^{(2)}_l:=\bm{b}^{(2)}_l.
\end{align*}

Then

\begin{align*}
\sigma_R\big(\widetilde{\bm{W}}^{(1)}_l\bm{H}_{:,k}+\widetilde{\bm{b}}^{(1)}_l\big)=
\begin{pmatrix}
\sigma_R\big(\bm{W}^{(1)}_l\bm{H}_{:,k}+\bm{b}^{(1)}_l\big)\\
\sigma_R(\bm{H}_{:,k})\\
\sigma_R(-\bm{H}_{:,k})
\end{pmatrix}.
\end{align*}

Multiplying by $\widetilde{\bm{W}}^{(2)}_l$ gives

\begin{align*}
    \bm{W}^{(2)}_l\,\sigma_R\big(\bm{W}^{(1)}_l\bm{H}_{:,k}+\bm{b}^{(1)}_l\big)
+\sigma_R(\bm{H}_{:,k})-\sigma_R(-\bm{H}_{:,k})
=
\bm{W}^{(2)}_l\,\sigma_R\big(\bm{W}^{(1)}_l\bm{H}_{:,k}+\bm{b}^{(1)}_l\big)
+\bm{H}_{:,k},
\end{align*}

using the identity $x=\sigma_R(x)-\sigma_R(-x)$ elementwise. Adding $\widetilde{\bm{b}}^{(2)}_l=\bm{b}^{(2)}_l$ matches $\mathcal{F}^{(\mathrm{FF})}_l(\bm{H})_{:,k}$. Thus the equality holds, completing the proof.
\end{proof}
This lemma allows us to ignore the residual term in the feed-forward layer when establishing the existence of the Transformer, thereby simplifying our subsequent derivation. 

\subsection{Core Result}
We divide the proof into two stages: first, we characterize the positional encoder's ability to encode spatial information into the tokens; second, we analyze how the masked self-attention layer injects global contextual information into single-token representations.  
\subsubsection{The Role of Positional Encoder}
To clarify the function of the positional encoder, we first define tokenwise separateness for input sequences. This property frequently appears in prior studies \citep{kim2023provable, KajitsukaS24, kajitsuka2025on} as a baseline assumption.
In this work, rather than assuming this property directly for the input sequence, we demonstrate that the combination of separable original token representations (Assumption~\ref{assumption: separable token representation}) and the positional encoder is sufficient to establish this condition.

\begin{definition}[Tokenwise Separateness]\label{def: tokenwise separateness}
Let $\bm{X}^{(1)}, \dots, \bm{X}^{(M)} \in \mathbb{R}^{d \times n}$ be input sequences. Then $\bm{X}^{(1)}, \dots, \bm{X}^{(M)}$ are called tokenwise $(r_{\min}, r_{\max}, \eta)$-separated if the following three conditions hold:
\begin{enumerate}
    \item For any $i \in [M]$ and $k \in [n]$,
    \begin{align*}
        \bigl\| \bm{X}^{(i)}_{:,k} \bigr\|_2 > r_{\min}.
    \end{align*}
    
    \item For any $i \in [M]$ and $k \in [n]$,
    \begin{align*}
        \bigl\| \bm{X}^{(i)}_{:,k} \bigr\|_2 < r_{\max}.
    \end{align*}

    \item For any $i,j \in [M]$ and $k,l \in [n]$ such that $\bm{X}^{(i)}_{:,k} \neq \bm{X}^{(j)}_{:,l}$,
    \begin{align*}
        \bigl\| \bm{X}^{(i)}_{:,k} - \bm{X}^{(j)}_{:,l} \bigr\|_2 > \eta.
    \end{align*}
\end{enumerate}
\end{definition}

In fact, if the token representations satisfy the requirements of Assumption~\ref{assumption: separable token representation}, then the resulting positional embeddings $\bm{h}^{(i)} + \bm{P}$ for $i \in [M]$ naturally satisfy Tokenwise Separateness (Definition~\ref{def: tokenwise separateness}). To demonstrate this, observe that for any tokens $\bm{t}_{i_1}, \bm{t}_{i_2} \in \mathcal{V}$ and indices $k_1, k_2\in [n]$ where $k_1 \neq k_2$, we have: 
\begin{gather*}
    \big\|\bm{t}_{i_1} + \bm{P}_{:,k_1}\big\|_2  \geq \big\|\bm{P}_{:,k_1}\big\|_2 - \big\|\bm{t}_{i_1}\big\|_2 \geq 2\alpha - \alpha =\alpha:= r_\mathrm{min},\\
    \big\|\bm{t}_{i_1} + \bm{P}_{:,k_1}\big\|_2 \leq \big\|\bm{t}_{i_1}\|_2 + \big\| \bm{P}_{:,k_1}\big\|_2 \leq \alpha + 2n\alpha = (2n+1)\alpha:= r_\mathrm{max}, \\
    \big\|\bm{t}_{i_1} + \bm{P}_{:, k_1} - \bm{t}_{k_2} - \bm{P}_{:, k_2}\big\|_2 \geq \big\|\bm{P}_{:, k_1}  - \bm{P}_{:, k_2}\big\|_2 -  \big\| \bm{t}_{i_1} - \bm{t}_{i_2}\big\|_2 \geq 2\alpha - \beta := \eta. \tag{If $2\alpha > \beta.$}
\end{gather*}
Thus, the effect of the positional encoder $\bm{P}$ in our theory manifests in three ways:
\begin{enumerate}
    \item It shifts the representations of tokens outside of a ball centered at $\bm{0} \in \R^{d}$ with radius $\alpha$.
    \item It ensures the embeddings remain bounded within a larger ball, even though their magnitudes increase.
    \item It guarantees separation between identical tokens that appear at different positions within a sequence.
\end{enumerate}
Overall, the positional encoder transforms a separable token representation into one that satisfies Tokenwise Separateness. Beyond the three properties discussed above, we emphasize that the encoder ensures distinct token-position pairs do not result in identical embeddings. Specifically, even for identical tokens (i.e., setting $i_1 = i_2$ in the third property), the representations remain separable due to their distinct positions, implying every $\bm{X}^{(i)} := \bm{h}^{(i)} + \bm{P} \in \R^{d \times n}, i \in [M]$ has no duplicate column, i.e., $\bm{X}_{:,k}^{(i)} \neq \bm{X}_{:,l}^{(i)}$ for any $i \in [M]$ and $k, l \in [n]$. 

\subsubsection{Masked Attention Implements Contextual Mapping}
Following the methodology of \citet{KajitsukaS24}, we first demonstrate that a single self-attention layer can implement contextual mapping building upon the technical lemmas introduced above, and then employ a feedforward layer to associate the resulting contextual IDs with their corresponding labels.

The definition of contextual mapping is given below.
\begin{definition}[Contextual Mapping]
Let $\bm{X}^{(1)}, \dots, \bm{X}^{(M)} \in \mathbb{R}^{d \times n}$ be input sequences. 
A map $f: \mathbb{R}^{d \times n} \to \mathbb{R}^{d \times n}$ is called an $(r, \gamma)$-contextual mapping if the following two conditions hold:

\begin{enumerate}
    \item For any $i \in [N]$ and $k \in [n]$, there holds
    \begin{align*}
        \bigl\| f(\bm{X}^{(i)})_{:,k} \bigr\|_2 < r.
    \end{align*}
    
    \item For any $i,j \in [N]$ and $k,l \in [n]$ such that $\bm{X}^{(i)}_{:,k} \neq \bm{X}_{:,l}^{(j)}$, there holds
    \begin{align*}
        \bigl\| f(\bm{X}^{(i)})_{:,k} - f(\bm{X}^{(j)})_{:,l} \bigr\|_2 > \gamma.
    \end{align*}
\end{enumerate}
In particular, $f(\bm{X}^{(i)})$ for $i \in [N]$ is called a \emph{context id} of $\bm{X}^{(i)}$.
\end{definition}

We will demonstrate that it is possible to construct a masked attention layer to realize a contextual mapping given the input is tokenwise separateness by following lemma.
\begin{lemma}\label{lemma: masked attention implements contextual mapping}
Let $ \bm{X}^{(1)}, \ldots, \bm{X}^{(M)} \in \mathbb{R}^{d \times n}$ be input sequences with no duplicate word token in each sequence, that is,
\begin{align*}
    \bm{X}_{:,k}^{(i)} \neq \bm{X}_{:,l}^{(i)}
\end{align*}
for any $ i \in [M] $ and $ k, l \in [n] $ with $k\neq l$. Also assume that $ \bm{X}^{(1)}, \ldots, \bm{X}^{(M)} $ are tokenwise $(r_{\min}, r_{\max}, \eta)$-separated. Then, there exist weight matrices $ \bm{W}^{(O)} \in \mathbb{R}^{d \times s} $ and $ \bm{W}^{(V)}, \bm{W}^{(K)}, \bm{W}^{(Q)} \in \mathbb{R}^{s \times d} $ such that the ranks of $ \bm{W}^{(V)}, \bm{W}^{(K)} $ and $ \bm{W}^{(Q)} $ are all 1, and 1-layer single head attention with softmax, i.e., $ \mathcal{F}_{SA} $ with $ h = 1 $ is an $(r, \gamma)$-contextual mapping for the input sequences $ \bm{X}^{(1)}, \ldots, \bm{X}^{(M)} \in \mathbb{R}^{d \times n} $ with $ r $ and $\gamma$ defined by

\begin{align*}
r &= r_{\max} + \frac{\eta}{4}, \\
\gamma &= \frac{2(\log n)^2 \eta^2 r_{\min}}{r_{\max}^{2}(|\mathcal{V}| + 1)^4(2\log n + 3)\pi d}
         \exp\left(-(|\mathcal{V}| + 1)^4 \frac{(2\log n + 3)\pi d r_{\max}^{2}}{4\eta r_{\min}}\right).
\end{align*}
\end{lemma}
The lemma reveals that the masked attention layer produces distinct outputs for identical input columns ($\bm{X}_{:,k}^{(i)} = \bm{X}_{:,k}^{(j)}$ for some $k$s) as long as the sequences differ globally $\big(\cup_{k=1}^n\bm{X}_{:,k}^{(i)} \neq \cup_{k=1}^n\bm{X}_{:,k}^{(j)}\big)$. Consequently, the attention mechanism effectively contextualizes each token by integrating information from the entire sequence into every output column.

Now let's focus on proving Lemma~\ref{lemma: masked attention implements contextual mapping}.
\begin{proof}[Proof of Lemma \ref{lemma: masked attention implements contextual mapping}]
Recall that a softmax-based self-attention function $\mathcal{F}_{SA} : \mathbb{R}^{d \times n} \to \mathbb{R}^{d \times n}$ with $h = 1$ is defined as
\begin{align*}
    \mathcal{F}_{SA}(\bm{Z}) = \bm{Z} + \bm{W}^{(O)} \left( \bm{W}^{(V)} \bm{Z} \right) \sigma_S \left[\left( \bm{W}^{(K)} \bm{Z} \right)^\top \left( \bm{W}^{(Q)} \bm{Z} \right) +\bm{M}\right],
\end{align*}
where $\bm{W}^{(O)} \in \mathbb{R}^{d \times 1}$ and $\bm{W}^{(V)}, \bm{W}^{(K)}, \bm{W}^{(Q)} \in \mathbb{R}^{1 \times d}$ are weight matrices, $\bm{M}$ is the masked matrix with entries being either $0$ or $-\infty$.

We construct a softmax-based self-attention function $ \mathcal{F}_{SA} $ with the property that
\begin{align}\label{eq: residual < η/4}
    \left\| \bm{W}^{(O)} \left( \bm{W}^{(V)} \bm{X}^{(i)} \right) \sigma_S \left[ \left( \bm{W}^{(K)} \bm{X}^{(i)} \right)^\top \left( \bm{W}^{(Q)} \bm{X}_{:,k}^{(i)} \right)  +\bm{M}_{:,k}\right]\right\|_2 < \frac{\eta}{4}
\end{align}
holds for any input sequence $\bm{X}^{(i)} $ with $ i \in [M]$ and index $k \in [n]$. When this property is fulfilled, it is easy to show that
\begin{align}
\left\| \mathcal{F}_{SA} \left( \bm{X}^{(i)} \right)_{:,k} \right\|_2 
&\le \left\| \bm{X}_{:,k}^{(i)} \right\| + \left\| \bm{W}^{(O)} \left( \bm{W}^{(V)} \bm{X}^{(i)} \right) \sigma_S \left[ \left( \bm{W}^{(K)} \bm{X}^{(i)} \right)^\top \left( \bm{W}^{(Q)} \bm{X}_{:,k}^{(i)} \right) +\bm{M}_{:,k}\right] \right\|_2 \nonumber\\
&< r_{\max} + \frac{\eta}{4}.\label{masked1}
\end{align}
holds for any $ i \in [M] $ and $ k \in [n] $, and also
\begin{align}
&\left\| \mathcal{F}_{SA} \left( \bm{X}^{(i)} \right)_{:,k} - \mathcal{F}_{SA} \left( \bm{X}^{(j)} \right)_{:,l} \right\|_2 \nonumber\\
&\ge \left\| \bm{X}_{:,k}^{(i)} - \bm{X}_{:,l}^{(j)} \right\|_2 - \left\| \bm{W}^{(O)} \left( \bm{W}^{(V)} \bm{X}^{(i)} \right) \sigma_S \left[ \left( \bm{W}^{(K)} \bm{X}^{(i)} \right)^\top \left( \bm{W}^{(Q)} \bm{X}_{:,k}^{(i)} \right)+\bm{M}_{:,k} \right] \right\|_2 \nonumber\\
&\quad - \left\| \bm{W}^{(O)} \left( \bm{W}^{(V)} \bm{X}^{(j)} \right) \sigma_S \left[ \left( \bm{W}^{(K)} \bm{X}^{(j)} \right)^\top \left( \bm{W}^{(Q)} \bm{X}_{:,l}^{(j)} \right) +\bm{M}_{:,l}\right] \right\|_2 \nonumber\\
&> \eta - \frac{\eta}{4} - \frac{\eta}{4} = \frac{\eta}{2}\label{masked2}
\end{align}
for any $i, j \in [M] $ and $k, l \in [n]$ such that $\bm{X}_{:,k}^{(i)} \neq \bm{X}_{:,l}^{(j)}$. So all that remains to prove is to construct a self-attention function $\mathcal{F}_{SA}$ that has the property like~\eqref{eq: residual < η/4}. 

Let $\kappa = 2 \log n + 3$ and fix any ${u}, {u}' \in \mathbb{R}$ with
\begin{align}\label{masked3}
\left| {u}{u}' \right| = (|\mathcal{V}| + 1)^4 \frac{\pi d}{8} \frac{\kappa}{\eta r_{\min}}.
\end{align}
Then, according to Lemma \ref{matrix construction} with $\kappa = 2 \log n + 3$, we see that there exists a unit vector $ \bm{v} \in \mathbb{R}^d $ such that
\begin{gather}
\left| \left( \bm{W}^{(K)} \bm{v}_a \right)^\top \left( \bm{W}^{(Q)} \bm{v}_c \right) - \left( \bm{W}^{(K)} \bm{v}_b \right)^\top \left( \bm{W}^{(Q)} \bm{v}_c \right) \right| > \kappa, \label{masked4}\\
\frac{1}{(|\mathcal{V}| + 1)^2} \sqrt{\frac{8}{\pi d}} \| \bm{v}_c \|_2 \le |\bm{v}^\top \bm{v}_c| \le \| \bm{v}_c \|_2\label{masked5}
\end{gather}
for any $ \bm{v}_a, \bm{v}_b, \bm{v}_c \in \mathcal{V} $ with $ \bm{v}_a \neq \bm{v}_b $, where $ \bm{W}^{(K)} = {u}\bm{v}^\top \in \mathbb{R}^{1 \times d} $ and $ \bm{W}^{(Q)} = {u}'\bm{v}^\top \in \mathbb{R}^{1 \times d} $. Furthermore, we configure $ \bm{W}^{(O)} \in \mathbb{R}^{d \times 1} $ and $ \bm{W}^{(V)} \in \mathbb{R}^{1 \times d} $ to be $ \bm{W}^{(V)} = {u}''\bm{v}^\top $ for any nonzero real ${u}'' \in \mathbb{R}$ such that
\begin{align}\label{masked6}
\bigl\| \bm{W}^{(O)} {u}'' \bigr\|_2 = \frac{\eta}{4r_{\max}}
\end{align}
holds. This can be accomplished, e.g., $ \bm{W}^{(O)} = {u}''\bm{u}'''^\top $ for any vector $ \bm{u}''' \in \mathbb{R}^d $ which satisfies $ \| \bm{u}''' \|_2 = \eta / (4r_{\max} | {u}'' |^2) $. In this case, the value of the self-attention without a skip-connection is upper-bounded by
\begin{align*}
&\left\| \bm{W}^{(O)}  \left( \bm{W}^{(V)} \bm{X}^{(i)} \right) \sigma_S \left[ \left( \bm{W}^{(K)} \bm{X}^{(i)} \right)^\top \left( \bm{W}^{(Q)} \bm{X}_{:,k}^{(i)} \right)+\bm{M}_{:,k} \right] \right\|_2 \\
&= \left\| \sum_{l = 1}^n s_{l}^k \bm{W}^{(O)} \left( \bm{W}^{(V)} \bm{X}^{(i)} \right)_{:,l} \right\|_2 \quad \text{with } s_{l}^k = \sigma_S \left[ \left( \bm{W}^{(K)} \bm{X}^{(i)} \right)^\top \left( \bm{W}^{(Q)} \bm{X}_{:,k}^{(i)} \right)+\bm{M}_{:,k} \right]_{l} \\
&\le \sum_{l = 1}^n s_{l}^k \left\| \bm{W}^{(O)} \left( \bm{W}^{(V)} \bm{X}^{(i)} \right)_{:,l} \right\|_2 \\
&\le \max_{l \in [n]} \left\| \bm{W}^{(O)} \left( \bm{W}^{(V)} \bm{X}^{(i)} \right)_{:,l} \right\|_2 \quad (\text{from } \sum_{l = 1}^n s_{l}^k = 1) \\
&= \max_{l \in [n]} \left\| \bm{W}^{(O)} {u}'' \bm{v}^\top \bm{X}_{:,l}^{(i)} \right\|_2 \\
&= \max_{l \in [n]} \left|\bm{v}^\top \bm{X}_{:,l}^{(i)}\right|\cdot\left\| \bm{W}^{(O)} {u}'' \right\|_2 \\
&\le \frac{\eta}{4 r_{\max}} \cdot \max_{l \in [n]} \left\| \bm{X}_{:,l}^{(i)} \right\|_2 \quad (\text{from \eqref{masked5} and \eqref{masked6}}) \\
&\le \frac{\eta}{4},
\end{align*}
which means that \eqref{masked1} and \eqref{masked2} are satisfied with the weight matrices defined above.

Now, we try to construct that the weight matrices $\bm{W}^{(O)}, \bm{W}^{(V)}, \bm{W}^{(K)}, \bm{W}^{(Q)}$ configured above can also distinguish the remaining case that $\bm{X}_{:,k}^{(i)} = \bm{X}_{:,k}^{(j)}$ but $\cup_{k=1}^n\bm{X}_{:,k}^{(i)} \neq \cup_{k=1}^n\bm{X}_{:,k}^{(j)}$. To this end, we define $\bm{a}^{(i)}, \bm{a}^{(j)}$ by
\begin{align*}
\bm{a}^{(i)} &= \overline{\bm{a}}^{(i)}+\bm{M}_{:,k} \in \mathbb{R}^n, \\
\bm{a}^{(j)} &= \overline{\bm{a}}^{(j)}+\bm{M}_{:,k} \in \mathbb{R}^n.
\end{align*}
with
\begin{align*}
\overline{\bm{a}}^{(i)} &= \left(\bm{W}^{(K)} \bm{X}^{(i)}\right)^\top \left(\bm{W}^{(Q)} \bm{X}_{:,k}^{(i)} \right) \in \mathbb{R}^n, \\
\overline{\bm{a}}^{(j)} &= \left(\bm{W}^{(K)} \bm{X}^{(j)}\right)^\top \left(\bm{W}^{(Q)} \bm{X}_{:,k}^{(j)}\right) \in \mathbb{R}^n.
\end{align*}
By the property of the Boltzmann operator, we have
\begin{align*}
\mathbf{boltz}\left( \bm{a}^{(i)} \right)&=\mathbf{boltz}\left(\overline{\bm{a}}^{(i)}+\bm{M}_{:,k}\right)=\mathbf{boltz}\left(\overline{\overline{\bm{a}}}^{(i)}\right),\\
\mathbf{boltz}\left( \bm{a}^{(j)} \right)&=\mathbf{boltz}\left(\overline{\bm{a}}^{(j)}+\bm{M}_{:,k}\right)=\mathbf{boltz}\left(\overline{\overline{\bm{a}}}^{(j)}\right),
\end{align*}
where $\overline{\overline{\bm{a}}}^{(i)}$ is the subvector of $\overline{\bm{a}}^{(i)}$ formed by collecting the entries of $\overline{\bm{a}}^{(i)}$ in the positions where the corresponding entries of $\bm{M}_{:,k}$ are zero and $\overline{\overline{\bm{a}}}^{(j)}$ is formed in the same manner.

In order to apply Lemma \ref{boltzmann separateness} to $\overline{\overline{\bm{a}}}^{(i)}$ and $\overline{\overline{\bm{a}}}^{(j)}$, we check the conditions in Lemma~\ref{boltzmann separateness}. Firstly, they satisfy \eqref{boltzmann separateness-1} with 
\begin{align*}
    \rho = (|\mathcal{V}| + 1)^4 \frac{\pi d}{8} \frac{\kappa r_{\max}^2}{\eta r_{\min}},
\end{align*}
since for any $ l \in [n] $, we have
\begin{align*}
\left| \overline{\bm{a}}_{l}^{(i)} \right| &= \left| \left( \bm{W}^{(K)} \bm{X}_{:,l}^{(i)} \right)^\top \left( \bm{W}^{(Q)} \bm{X}_{:,k}^{(i)} \right) \right| \\
&= \left| \left( \bm{v}^\top \bm{X}_{:,l}^{(i)} \right) \right| \cdot \left| {u}{u}' \right| \cdot \left| \left( \bm{v}^\top \bm{X}_{:,k}^{(i)} \right) \right| \\
&\le (|\mathcal{V}| + 1)^4 \frac{\pi d}{8} \frac{\kappa}{\eta r_{\min}} r_{\max}^2\quad (\text{from \eqref{masked3} and \eqref{masked5}})
\end{align*}
and the same upper-bound also holds for $\overline{\bm{a}}^{(j)}$. \eqref{masked4} and the definitions of $\overline{\overline{\bm{a}}}^{(i)},\overline{\overline{\bm{a}}}^{(j)}$ implies that $\overline{\overline{\bm{a}}}^{(i)},\overline{\overline{\bm{a}}}^{(j)}$ satisfy~\eqref{boltzmann separateness-3} in Lemma~\ref{boltzmann separateness}. Since there exists no duplicate token in $\bm{X}^{(i)}$ and $\bm{X}^{(j)}$ respectively, \eqref{masked4} and the definitions of $\overline{\overline{\bm{a}}}^{(i)},\overline{\overline{\bm{a}}}^{(j)}$ also implies that $\overline{\overline{\bm{a}}}^{(i)},\overline{\overline{\bm{a}}}^{(j)}$ satisfy \eqref{boltzmann separateness-2} in Lemma \ref{boltzmann separateness}. Furthermore, $\cup_{k=1}^n\bm{X}_{:,k}^{(i)} \neq \cup_{k=1}^n\bm{X}_{:,k}^{(j)}$ implies that $\overline{\overline{\bm{a}}}^{(j)}$ is not a permutation of $\overline{\overline{\bm{a}}}^{(i)}$. Based on the above discussion, we can now apply Lemma \ref{boltzmann separateness} to derive that
\begin{align*}
\left| \mathbf{boltz}\left({{\bm{a}}}^{(i)}\right) - \mathbf{boltz}\left({{\bm{a}}}^{(j)}\right) \right|=\left| \mathbf{boltz}\left(\overline{\overline{\bm{a}}}^{(i)}\right) - \mathbf{boltz}\left(\overline{\overline{\bm{a}}}^{(j)}\right) \right| > \varrho' = (\log n)^2 e^{-2\rho},
\end{align*}
that is,
\begin{align*}
\left| \left( \bm{a}^{(i)} \right)^\top \sigma_S \bigl[ \bm{a}^{(i)} \bigr] - \left( \bm{a}^{(j)} \right)^\top \sigma_S \bigl[ \bm{a}^{(j)} \bigr] \right| > \varrho',
\end{align*}
which can be further expanded as
\begin{align}\label{eq: δ' ≤ FSA}
\varrho' &< \left| \left( \bm{a}^{(i)} \right)^\top \sigma_S \bigl[ \bm{a}^{(i)} \bigr] - \left( \bm{a}^{(j)} \right)^\top \sigma_S \bigl[ \bm{a}^{(j)} \bigr] \right| \nonumber\\
&=\left| \left[\bm{M}_{:,k}^\top+\left( \bm{X}_{:,k}^{(i)} \right)^\top \left( \bm{W}^{(Q)} \right)^\top \bm{W}^{(K)}  \bm{X}^{(i)}\right] \sigma_S \bigl[ \bm{a}^{(i)} \bigr] - \left[\bm{M}_{:,k}^\top+\left( \bm{X}_{:,k}^{(j)} \right)^\top \left( \bm{W}^{(Q)} \right)^\top \bm{W}^{(K)}  \bm{X}^{(j)}\right] \sigma_S \bigl[ \bm{a}^{(j)} \bigr] \right|\nonumber\\
&=\left| \left( \bm{X}_{:,k}^{(i)} \right)^\top \left( \bm{W}^{(Q)} \right)^\top \bm{W}^{(K)}  \bm{X}^{(i)} \sigma_S \bigl[ \bm{a}^{(i)} \bigr] - \left( \bm{X}_{:,k}^{(j)} \right)^\top \left( \bm{W}^{(Q)} \right)^\top \bm{W}^{(K)}  \bm{X}^{(j)}\sigma_S \bigl[ \bm{a}^{(j)} \bigr] \right| \nonumber\\
&= \left| \left( \bm{X}_{:,k}^{(i)} \right)^\top \left( \bm{W}^{(Q)} \right)^\top \bm{W}^{(K)} \left( \bm{X}^{(i)} \sigma_S \bigl[ \bm{a}^{(i)} \bigr] - \bm{X}^{(j)} \sigma_S \bigl[ \bm{a}^{(j)} \bigr] \right) \right| \nonumber\\
&= \left| \left( \bm{X}_{:,k}^{(i)} \right)^\top \bm{v} {u}'{u} \bm{v}^\top \left( \bm{X}^{(i)} \sigma_S \bigl[ \bm{a}^{(i)} \bigr] - \bm{X}^{(j)} \sigma_S \bigl[ \bm{a}^{(j)} \bigr] \right) \right| \nonumber\\
&= \left| \bm{v}^\top \bm{X}_{:,k}^{(i)} \right| \cdot \left| {u}{u}' \right| \cdot \left| \left( \bm{v}^\top \bm{X}^{(i)} \right) \sigma_S \bigl[ \bm{a}^{(i)} \bigr] - \left( \bm{v}^\top \bm{X}^{(j)} \right) \sigma_S \bigl[ \bm{a}^{(j)} \bigr] \right| \nonumber\\
&\le r_{\max} \cdot (|\mathcal{V}| + 1)^4 \frac{\pi d}{8} \frac{\kappa}{\eta r_{\min}} \cdot \left| \left( \bm{v}^\top \bm{X}^{(i)} \right) \sigma_S \bigl[ \bm{a}^{(i)} \bigr] - \left( \bm{v}^\top \bm{X}^{(j)} \right) \sigma_S \bigl[ \bm{a}^{(j)} \bigr] \right|, 
\end{align}
where the third step is due to $\bm{M}_{:,k}^\top\sigma_S \bigl[ \bm{a}^{(i)} \bigr]=\bm{M}_{:,k}^\top\sigma_S \bigl[ \bm{a}^{(j)} \bigr]=0$, the fourth step is due to $ \bm{X}_{:,k}^{(i)} = \bm{X}_{:,k}^{(j)} $ and the final step follows from \eqref{masked3} and \eqref{masked5}. Therefore, the gap between the outputs of the self-attention function for $ \bm{X}^{(i)} $ and $ \bm{X}^{(j)} $ is lower bounded as follows:
\begin{align*}
&\left\| \mathcal{F}_{SA} \left( \bm{X}^{(i)} \right)_{:,k} - \mathcal{F}_{SA} \left( \bm{X}^{(j)} \right)_{:,k} \right\|_2 \\
&= \left\| \bm{W}^{(O)} \left( \bm{W}^{(V)} \bm{X}^{(i)} \right) \sigma_S \bigl[ \bm{a}^{(i)} \bigr] - \bm{W}^{(O)} \left( \bm{W}^{(V)} \bm{X}^{(j)} \right) \sigma_S \bigl[ \bm{a}^{(j)} \bigr] \right\|_2 \quad (\text{since } \bm{X}_{:,k}^{(i)} = \bm{X}_{:,k}^{(j)}) \\
&= \bigl\| \bm{W}^{(O)} {u}'' \bigr\|_2 \cdot \left| \left( \bm{v}^\top \bm{X}^{(i)} \right) \sigma_S \bigl[ \bm{a}^{(i)} \bigr] - \left( \bm{v}^\top \bm{X}^{(j)} \right) \sigma_S \bigl[ \bm{a}^{(j)} \bigr] \right| \\
&> \frac{\eta}{4 r_{\max}} \cdot \frac{\varrho'}{(|\mathcal{V}| + 1)^4 \frac{8 \eta r_{\min}}{\pi d \kappa r_{\max}}},
\end{align*}
where the final equality stems from combining~\eqref{masked6} and~\eqref{eq: δ' ≤ FSA} 
where $\kappa$ and $\varrho'$ are defined respectively as
\begin{align*}
\kappa &= 2 \log n + 3, \\ 
\varrho' &= (\log n)^2 e^{-2\rho} \quad \text{with} \quad \rho = (|\mathcal{V}| + 1)^4 \frac{\pi d}{8} \frac{\kappa r_{\max}^2}{\eta r_{\min}}. 
\end{align*}
By plugging $\kappa$ and $\varrho'$, the above inequality is simplified as
\begin{align*}
&\left\| \mathcal{F}_{SA} \left( \bm{X}^{(i)} \right)_{:,k} - \mathcal{F}_{SA} \left( \bm{X}^{(j)} \right)_{:,k} \right\|_2\\ 
&> \frac{2(\log n)^2 \eta^2 r_{\min}}{r_{\max}^2 (|\mathcal{V}| + 1)^4 (2 \log n + 3) \pi d} \exp \left( - (|\mathcal{V}| + 1)^4 \frac{(2 \log n + 3) \pi d r_{\max}^2}{4 \eta r_{\min}} \right). 
\end{align*}
\end{proof}
Furthermore, to associate the contextual IDs produced by the contextual mapping with their corresponding labels, we make use of the following memorization result of feedforward neural networks.
\begin{lemma}[\cite{zhang2021understanding}, Theorem 1]\label{fnn memorization}
There exists a two-layer neural network with ReLU activations and $2M + d\times n$ weights that can represent any function on a sample of size $M$ in $d \times n$ dimensions.
\end{lemma}

\begin{theorem}\label{theorem: memorization}
Let $(\bm{X}^{(1)}, \bm{p}^{(1)}), \cdots, (\bm{X}^{(M)}, \bm{p}^{(M)}) \in \R^{d \times n} \times \R^{|\mathcal{V}|}$ be a sequence of positioned document-label pairs such that $\bm{X}^{(1)}, \cdots, \bm{X}^{(M)}$ are tokenwise separable (Definition~\ref{def: tokenwise separateness}).
Then, there exist $\mathcal{F}_{SA},\mathcal{F}_{FF}$ such that 
\begin{align*}
\sigma_S\circ\mathcal{E}_{out}\circ\mathcal{F}_{FF} \circ \mathcal{F}_{SA}\left( \bm{X}^{(i)}
\right) = \bm{p}^{(i)},\quad i\in[M].
\end{align*}
Here, the depth of $\mathcal{F}_{FF}$ is $2$ and the width is $|\mathcal{V}|^{n+1}$, the head number of $\mathcal{F}_{SA}$ is $1$. 
\end{theorem}
\begin{proof}
By the property of $\sigma_S$, we can find $\left\{\bar{\bm{p}}^{(i)}\right\}_{i\in [M]}\subseteq\R^{|\mathcal{V}|}$ such that for every $i\in [M]$,
$\sigma_S\left(\bar{\bm{p}}^{(i)}\right)=\bm{p}^{(i)}$. Lemma~\ref{lemma: masked attention implements contextual mapping} ensures that $\left\{\mathcal{F}_{SA} \left(\bm{X}^{(i)}
\right)_{:,n}\right\}_{i \in [M]}$ is a sequence that contains $M$ distinct elements. Hence by applying Lemma~\ref{fnn memorization} and the parallelization of feed-forward networks,
we can construct a feed-forward network $f_{FF} : \mathbb{R}^d \rightarrow \mathbb{R}^{|\mathcal{V}|}$ with depth $2$ and width $(2M+d)|\mathcal{V}|$ such that for every $i \in [M]$,
\begin{align*}
f_{FF}\left(\mathcal{F}_{SA} \left( \bm{X}^{(i)}
\right)_{:,n}\right) = \bar{\bm{p}}^{(i)}.
\end{align*}
This means that by defining a token-wise operation $\mathcal{F}_{FF} : \mathbb{R}^{d\times n} \rightarrow \mathbb{R}^n$ as
\begin{align*}
    \mathcal{F}_{FF}(\bm{X})_{:,k} := f_{FF}(\bm{X}_{:,k}),\quad k\in[n],
\end{align*}
we have
\begin{align*}
    \left[\mathcal{F}_{FF} \circ \mathcal{F}_{SA}\left( \bm{X}^{(i)}
\right)\right]_{:,n} = \bar{\bm{p}}^{(i)},\quad i\in[M].
\end{align*}
Therefore
\begin{align*}
\mathcal{E}_{out}\circ\mathcal{F}_{FF} \circ \mathcal{F}_{SA}\left( \bm{X}^{(i)}
\right) = \bar{\bm{p}}^{(i)},\quad i\in[M],
\end{align*}
which implies the conclusion.
\end{proof}

We have now demonstrated the existence of a Transformer that causes the approximation error $\mathcal{E}_{\mathrm{app}}$ to vanish.

\section{Generalization}\label{section: generlization}

\begin{definition}[Rademacher complexity]
	Let $n'\in\mathbb{N}_+$. The Rademacher complexity of a set $A \subset \mathbb{R}^{n'}$ is defined as
	\begin{equation*}
		\mathfrak{R}_{n'}(A) {:=} \mathbb{E}_{\{\xi_i\}_{i \in [n']}}\left[\sup_{a\in A}\frac{1}{n'}\sum_{i=1}^{n'} \xi_i a_i\right],
	\end{equation*}
	where, $\{\xi_i\}_{i=1}^{n'}$ are $n'$ i.i.d  Rademacher variables with $\mathbb{P}(\xi_i = 1) = \mathbb{P}(\xi_i = -1) = \frac{1}{2}.$
	The Rademacher complexity of  function class $\mathcal{G}$ associate with random samples $\{X_i\}_{i \in [n']}$ is defined as
	\begin{equation*}
		\mathfrak{R}_{n'}(\mathcal{G}) {:=} \mathbb{E}_{\{X_i,\xi_i\}_{i\in[n']}}\left[\sup_{g\in \mathcal{G}}\frac{1}{n'}\sum_{i=1}^{n'} \xi_i g(X_i)\right].
	\end{equation*}
\end{definition}

\begin{definition}[covering number]
An $\epsilon$-cover of a set $P$ in a metric space $(S, \tau)$
is a subset $P_c\subseteq S$ such that for each $x \in P$, there exists a $x_c \in P_c$ such that $\tau(x, x_c) \leq \lambda$. The $\lambda$-covering number of $P$, denoted as $\mathcal{N}(\lambda, P,\tau)$ is  defined to be the minimum cardinality among all $\lambda$-cover of $P$ with respect to the metric $\tau$.
\end{definition}
When the set we are concerned with is a function class, we can further define the so-called uniform covering number.
\begin{definition}[uniform covering number]
Let $\mathcal{G}$ be a class of vector-valued functions. Given $n'$  samples  $\{X_i\}_{i=1}^{n'}$, the uniform covering number $\mathcal{N}(\lambda,\mathcal{G},\|\cdot\|_{\infty}, n')$ is then defined as
\begin{equation*}
\mathcal{N}(\lambda,\mathcal{G},\|\cdot\|_{\infty},n'):=\max_{\{X_i\}_{i \in [n']}}\mathcal{N}(\lambda,\mathcal{G}|_{\{X_i\}_{i \in [n']}},\|\cdot\|_{\infty}),
\end{equation*}
where  
$\mathcal{G}|_{\{X_i\}_{i\in[n']}}:= \{(g(X_1),g(X_2),\cdots, g(X_{n'})): g\in\mathcal{G}\}$ and $\|\cdot\|_{\infty}$ is the infinity norm on $\R^{n'}$.
\end{definition}

\begin{definition}[VC dimension]
		Let $\mathcal{G}$ be a class of $\{0,1\}$-valued functions. Suppose that $S=\{x_1,x_2,\cdots,x_{n'}\}\subset\Omega$. We say that $S$ is shattered by $\mathcal{G}$ if for any $b\in\{0,1\}^{n'}$, there exists a $g\in\mathcal{G}$ satisfying
		\begin{equation*}
            g(x_i)=b_i,\quad i \in [n']
		\end{equation*}
		The VC dimension of $\mathcal{G}$, denoted as $\mathrm{VCdim}(\mathcal{G})$, is defined to be the maximum cardinality among all sets shattered by $\mathcal{G}$.
	\end{definition}

\begin{definition}[pseudo-dimension]
		Let $\mathcal{G}$ be a class of real-valued functions. Suppose that $S=\{x_1,x_2,\cdots,x_{n'}\}\subset\Omega$. We say that $S$ is pseudo-shattered by $\mathcal{G}$ if there exists $y_1,y_2,\cdots,y_{n'}$ such that for any $b\in\{0,1\}^{n'}$, there exists a $g\in\mathcal{G}$ satisfying
		\begin{equation*}
			\mathrm{sign}(g(x_i)-y_i)=b_i,\quad i\in[n']
		\end{equation*}
		and we say that $\{y_i\}_{i\in [n']}$ witnesses the shattering.
		The pseudo-dimension of $\mathcal{G}$, denoted as $\mathrm{Pdim}(\mathcal{G})$, is defined to be the maximum cardinality among all sets pseudo-shattered by $\mathcal{G}$.
	\end{definition}
\begin{lemma}[\cite{jiao2024error}, Lemma 5.7]\label{to covering}
Let $n'\in\mathbb{N}_{+}$. For any function class $\mathcal{G}$ with $|g|\leq B_{\mathcal{G}}$ for all $g\in\mathcal{G}$,
\begin{align*}
\mathfrak{R}_{n'}(\mathcal{G})\leq \inf_{0<u<B_{\mathcal{G}}/2}\left(4u+\frac{12}{\sqrt{n'}}\int_{u}^{B_{\mathcal{G}}/2}\sqrt{\ln\mathcal{N}(\lambda,\mathcal{G},\|\cdot\|_{\infty},n')}d\lambda\right),
\end{align*}
where $\mathcal{B}_\mathcal{G}$ is a constant that $\|g\|_\infty \leq \mathcal{B}_\mathcal{G}$ for any $g \in \mathcal{G}$.
\end{lemma}

\begin{lemma}[\cite{anthony2009neural}, Theorem 12.2]\label{covering-number}
Let $n'\in\mathbb{N}_+$. For any function class $\mathcal{G}$ with $|g|\leq B_{\mathcal{G}}$ for all $g\in\mathcal{G}$, 
\begin{align*}
    \mathcal{N}(\lambda,\mathcal{G},\|\cdot\|_{\infty},n') \leq \sum_{i=1}^{\mathrm{Pdim}(\mathcal{G})} \binom{n'}{i} \left( \frac{B_{\mathcal{G}}}{\lambda} \right)^i,
\end{align*}
which is less than \((en'B_{\mathcal{G}}/(\lambda\cdot\mathrm{Pdim}(\mathcal{G})))^{\mathrm{Pdim}(\mathcal{G})}\) for \(n' \geq \mathrm{Pdim}(\mathcal{G})\).
\end{lemma}
The following lemma provides an upper bound on the VC dimension, expressed in terms of the number of certain types of operations contained in the functions of the function class.
\begin{lemma}[\cite{anthony2009neural}, Theorem 8.4]\label{VCdim operation}
Suppose \( h \) is a function from \( \mathbb{R}^{d_1} \times \mathbb{R}^{d_2} \) to \(\{0,1\}\) and let  
\begin{align*}
    H = \{x \mapsto h(a,x) : a \in \mathbb{R}^{d_1}\}
\end{align*}
be the class determined by \( h \). Suppose that \( h \) can be computed by an algorithm that takes as input the pair \((a,x) \in \mathbb{R}^{d_1} \times \mathbb{R}^{d_2}\) and returns \( h(a,x) \) after no more than \( t \) operations of the following types:  

\begin{itemize}
    \item the arithmetic operations \( +, -, \times, \) and \( / \) on real numbers,  
    \item jumps conditioned on \( >, \geq, <, \leq, = \), and \( \neq \) comparisons of real numbers, and
    \item output $0$ or $1$.
\end{itemize}
Then \( \mathrm{VCdim}(H) \leq 4d(t+2) \).
\end{lemma}

For $i\in[|\mathcal{V}|]$, define
\begin{align*}
\mathcal{T}_i:=\{{T}_i:\boldsymbol{T}\in\mathcal{T}\},
\end{align*}
where $\mathcal{T}$ is the Transformer class of the scale as the required in Memorization (\ref{theorem: memorization}). Then we have following lemma:

\begin{theorem}\label{theorem: generlization}
For any $i\in[|\mathcal{V}|]$, we have
\begin{align*}
\mathfrak{R}_{Nn}(\mathcal{T}_i)
&\lesssim\sqrt{|\mathcal{V}|d^2(d+M)(n^2+dn+|\mathcal{V}|d+|\mathcal{V}|M)}\frac{\ln Nn}{\sqrt{Nn}}.
\end{align*}
\end{theorem}
\begin{proof}
In order to apply Lemma \ref{to covering}, we first estimate the covering number of $\mathcal{T}_i$. If $Nn\geq\mathrm{Pdim}(\mathcal{T}_i)$, noting that functions in $\mathcal{T}_i$ are bounded by $1$, by Lemma \ref{covering-number} we have
\begin{align*}
\mathcal{N}\left(\lambda,\mathcal{T}_i,\|\cdot\|_{\infty},Nn\right)&\leq\left(\frac{eNn}{\lambda\mathrm{Pdim}(\mathcal{T}_i)}\right)^{\mathrm{Pdim}(\mathcal{T}_i)}
\leq\left(\frac{eNn}{\lambda}\right)^{\mathrm{Pdim}(\mathcal{T}_i)}.
\end{align*}
If $Nn<\mathrm{Pdim}(\mathcal{T}_i)$, since for given $\{\bm{X}_j\}_{j \in [Nn]}$, we have $\mathcal{N}\left(\lambda,\mathcal{T}_i|_{\{\bm{X}_j\}_{j=1}^m},\|\cdot\|_{\infty}\right)\leq\left(\frac{2}{\lambda}\right)^{Nn}$ and hence
\begin{align*}
\mathcal{N}\left(\lambda,\mathcal{T}_i,\|\cdot\|_{\infty},Nn\right)&=\max_{\{\bm{X}_j\}_{j\in[Nn]}}\mathcal{N}\left(\lambda,\mathcal{T}_i|_{\{\bm{X}_j\}_{j\in[Nn]}},\|\cdot\|_{\infty}\right)
\leq\left(\frac{2}{\lambda}\right)^{Nn}\leq\left(\frac{2}{\lambda}\right)^{\mathrm{Pdim}(\mathcal{T}_i)}
\end{align*}
provided $\lambda< 2$. Therefore for any $Nn\in\mathbb{N}_+$, we always have
\begin{align}\label{covering3}
\mathcal{N}\left(\lambda,\mathcal{T}_i,\|\cdot\|_{\infty},Nn\right)\leq\left(\frac{eNn}{\lambda}\right)^{\mathrm{Pdim}(\mathcal{T}_i)}.
\end{align}
For the functions in $\mathcal{T}_i$: in the self-attention layer, the number of parameters is $4d$ and the number of operations is $\mathcal{O}(dn^2+d^2n)$; in the feedforward layer, the number of parameters and the number of operations are both ${\mathcal{O}}(|\mathcal{V}|d^2+|\mathcal{V}|dM)$. Hence by applying Lemma \ref{VCdim operation} to the class
\begin{align*}
\mathcal{F}:=\left\{\mathrm{sign}(f-c):f\in\mathcal{T}_i,c\in\mathbb{R}\right\},
\end{align*}
we have
\begin{align*}
\mathrm{Pdim}(\mathcal{T}_i)=\mathrm{VCdim(\mathcal{F})}={\mathcal{O}}(|\mathcal{V}|d^2(d+M)(n^2+dn+|\mathcal{V}|d+|\mathcal{V}|M)).
\end{align*}

Plugging this pseudo-dimension estimate into \eqref{covering3}, we obtain
\begin{align*}
\ln\mathcal{N}\left(\lambda,\mathcal{T}_i,\|\cdot\|_{\infty},Nn\right)\lesssim \sqrt{|\mathcal{V}|d^2(d+M)(n^2+dn+|\mathcal{V}|d+|\mathcal{V}|M)}\ln\left(\frac{eNn}{\lambda}\right).
\end{align*}
It follows from Lemma \ref{to covering} that
\begin{align*}
\mathfrak{R}_{Nn}(\mathcal{T}i)&\leq\inf_{0<u<1/2}\left(4u+\frac{12}{\sqrt{nN}}\int_u^{1/2}\sqrt{\ln\mathcal{N}(\lambda,\mathcal{T}i,|\cdot|{\infty},Nn)}d\lambda\right)\\
&\lesssim \sqrt{|\mathcal{V}|d^2(d+M)(n^2+dn+|\mathcal{V}|d+|\mathcal{V}|M)}\inf_{0<u<1/2}\left(u+\frac{1}{\sqrt{Nn}}\int_u^{1/2}\sqrt{\ln\left(\frac{eNn}{\lambda}\right)}d\lambda\right)\\
&\lesssim\sqrt{|\mathcal{V}|d^2(d+M)(n^2+dn+|\mathcal{V}|d+|\mathcal{V}|M)}\frac{\ln Nn}{\sqrt{Nn}}.
\end{align*}
\end{proof}
Finally plugging the identity $M \leq |\mathcal{V}|^{n}$ into above completes the proof of Theorem~\ref{Theorem: pretraining}.

\section{More Discussions}\label{Appendix: more discussions}
It suffices to prove the following lemma, which replaces Lemma~\ref{lemma: |q - ~q| < MΔ} in our subsequent derivation.
\begin{lemma} 
    \begin{align*}
        \Big|q(\bm{Y} \mid \mathcal{P}_{\mathrm{CoT}}) - \tilde{q}(\bm{Y}\mid \mathcal{P}_{\mathrm{CoT}})\Big| \leq 2r\phi + 3m\mathcal{M}L\Delta_{\mathcal{P}_{\mathrm{CoT}}} 
    \end{align*}
\end{lemma}
\begin{proof}
    We begin by observing that both $q(\bm{Y} \mid \mathcal{P}_{\mathrm{CoT}})$ and $\tilde{q}(\bm{Y} \mid \mathcal{P}_{\mathrm{CoT}})$ can be expressed as mixtures over the compositional task space $\Theta^L$:
    \begin{align*}
        q(\bm{Y} \mid \mathcal{P}_{\mathrm{CoT}}) &= \sum_{\theta \in \Theta}q(\theta \mid \mathcal{P}_{\mathrm{CoT}}) q(\bm{Y} \mid \mathcal{P}_{\mathrm{CoT}},\, \theta), \\
        \tilde{q}(\bm{Y} \mid \mathcal{P}_{\mathrm{CoT}}) &= \sum_{\bm{\theta} \in \Theta^L}\tilde{q}(\bm{\theta} \mid \mathcal{P}_{\mathrm{CoT}})\tilde{q}(\bm{Y} \mid \mathcal{P}_{\mathrm{CoT}},\, \bm{\theta}).
    \end{align*}
    Applying the triangle inequality, we obtain:
    \begin{align}\label{eq: |q(Y|P) - ~q(Y|P)|}
        \Big|q(\bm{Y} \mid \mathcal{P}_{\mathrm{CoT}}) - \tilde{q}(\bm{Y} \mid \mathcal{P}_{\mathrm{CoT}})\Big| &= \Big| \sum_{\bm{\theta} \in \Theta^L} q(\bm{\theta} \mid \mathcal{P}_{\mathrm{CoT}}) q(\bm{Y} \mid \mathcal{P}_{\mathrm{CoT}},\, \bm{\theta}) - \tilde{q}(\bm{\theta} \mid \mathcal{P}_{\mathrm{CoT}})\tilde{q}(\bm{Y} \mid \mathcal{P}_{\mathrm{CoT}}, \, \bm{\theta})\Big| \notag \\
        &\leq \sum_{\bm{\theta} \in \Theta^L}q(\bm{\theta} \mid \mathcal{P}_{\mathrm{CoT}})\Big| q(\bm{Y} \mid \mathcal{P}_{\mathrm{CoT}}, \bm{\theta}) - \tilde{q}(\bm{Y} \mid \mathcal{P}_{\mathrm{CoT}}, \bm{\theta})\Big| \notag \\
        &\quad + \sum_{\bm{\theta} \in \Theta^L}\Big|q(\bm{\theta} \mid \mathcal{P}_{\mathrm{CoT}}) - \tilde{q}(\bm{\theta} \mid \mathcal{P}_{\mathrm{CoT}})  \Big|q(\bm{Y} \mid \mathcal{P}_{\mathrm{CoT}}, \bm{\theta}).
    \end{align}
    Analyzing these terms individually, for the first term of~\eqref{eq: |q(Y|P) - ~q(Y|P)|}, we have:
    \begin{align*}
        \Big| q(\bm{Y} \mid \mathcal{P}_{\mathrm{CoT}}, \bm{\theta}) - \tilde{q}(\bm{Y} \mid \mathcal{P}_{\mathrm{CoT}},\, \bm{\theta})\Big| &=\Big| q(\bm{Y} \mid \mathcal{P}_{\mathrm{CoT}}, \bm{\theta}) - q(\bm{Y} \mid \bm{x}, \,\bm{\theta})\Big| + \Big|q(\bm{Y} \mid \bm{x}, \bm{\theta}) - \tilde{q}(\bm{Y} \mid \bm{x}, \, \bm{\theta})\Big| \\
        &\quad + \Big|\tilde{q}(\bm{Y} \mid \bm{x}, \, \bm{\theta}) - \tilde{q}(\bm{Y} \mid \mathcal{P}_{\mathrm{CoT}},\, \bm{\theta})\Big|\\
        &\leq 2n \phi + \Big|q(\bm{Y} \mid \bm{x},\, \bm{\theta}) - \tilde{q}(\bm{Y} \mid \bm{x}, \, \bm{\theta})\Big| \\
        &= 2n\phi +  \Big| \prod_{l=1}^Lq(\bm{Y}_{:,l} \mid \bm{x}\circ \bm{Y}_{:, \prec l},\, \theta_l) - \prod_{l=1}^L\tilde{q}_l(\bm{Y}_{:,l} \mid \bm{x}\circ \bm{Y}_{:, \prec l},\, \theta_l)\Big| \\
        &\leq 2n\phi + 2L\varphi
    \end{align*}
    The claim that $\Big| \prod_{l=1}^Lq(\bm{Y}_{:,l} \mid \bm{x}\circ \bm{Y}_{:, \prec l},\, \theta_l) - \prod_{l=1}^L\tilde{q}_l(\bm{Y}_{:,l} \mid \bm{x}\circ \bm{Y}_{:, \prec l},\, \theta_l)\Big| \leq 2L\varphi$ follows from the fact that:
    \begin{align}\label{eq: technique for dealing with product terms}
        &\Big| \prod_{l=1}^Lq(\bm{Y}_{:,l} \mid \bm{x}\circ \bm{Y}_{:, \prec l},\, \theta_l) - \prod_{l=1}^L\tilde{q}_l(\bm{Y}_{:,l} \mid \bm{x}\circ \bm{Y}_{:, \prec l},\, \theta_l)\Big| \notag \\
        &= \Big|\prod_{l=1}^Lq(\bm{Y}_{:,l} \mid \bm{x}\circ \bm{Y}_{:, \prec l},\, \theta_l) - q(\bm{Y}_{:,L} \mid \bm{x} \circ \bm{Y}_{:, \prec L}, \, \theta_L)\prod_{l=1}^{L-1}\tilde{q}_l(\bm{Y}_{:,l} \mid \bm{x}\circ \bm{Y}_{:, \prec l},\, \theta_l) \Big| \notag \\
        &\quad + \Big|q(\bm{Y}_{:,L-1} \mid \bm{x} \circ \bm{Y}_{:, \prec L}, \, \theta_L)\prod_{l=1}^{L-1}\tilde{q}_l(\bm{Y}_{:,l} \mid \bm{x}\circ \bm{Y}_{:, \prec l},\, \theta_l) - \prod_{l=1}^{L}\tilde{q}_l(\bm{Y}_{:,l} \mid \bm{x}\circ \bm{Y}_{:, \prec l},\, \theta_l)\Big| \notag \\
        &\leq \Big|\prod_{l=1}^{L-1}q(\bm{Y}_{:,l} \mid \bm{x}\circ \bm{Y}_{:, \prec l},\, \theta_l) - \prod_{l=1}^{L-1}\tilde{q}_l(\bm{Y}_{:,l} \mid \bm{x}\circ \bm{Y}_{:, \prec l},\, \theta_l) \Big| \notag \\
        &\quad + \Big|q(\bm{Y}_{:,L} \mid \bm{x} \circ \bm{Y}_{:, \prec L}, \, \theta_L) - \tilde{q}_l(\bm{Y}_{:,L} \mid \bm{x}\circ \bm{Y}_{:, \prec L},\, \theta_l)\Big| \notag \\
        &\leq \Big|\prod_{l=1}^{L-1}q(\bm{Y}_{:,l} \mid \bm{x}\circ \bm{Y}_{:, \prec l},\, \theta_l) - \prod_{l=1}^{L-1}\tilde{q}_l(\bm{Y}_{:,l} \mid \bm{x}\circ \bm{Y}_{:, \prec l},\, \theta_l) \Big| + 2\varphi.
    \end{align}
    Iteratively applying this inequality yields the desired bound. 

    Next, addressing the second term in~\eqref{eq: |q(Y|P) - ~q(Y|P)|}, we apply Bayes' rule to the posteriors $q(\bm{\theta} \mid \mathcal{P}_{\mathrm{CoT}})$ and $\tilde{q}(\bm{\theta} \mid \mathcal{P}_{\mathrm{CoT}})$:
    \begin{align*}
        q(\bm{\theta} \mid \mathcal{P}_{\mathrm{CoT}}) = \frac{q(\mathcal{P}_{\mathrm{CoT}}\mid \bm{\theta})q(\bm{\theta})}{q(\mathcal{P}_{\mathrm{CoT}})}, \quad \tilde{q}(\bm{\theta} \mid \mathcal{P}_{\mathrm{CoT}})=\frac{\tilde{q}(\mathcal{P}_{\mathrm{CoT}} \mid \bm{\theta})\tilde{q}(\bm{\theta})}{\tilde{q}(\mathcal{P}_{\mathrm{CoT}})}.
    \end{align*}
    Substituting these into the second term of~\eqref{eq: |q(Y|P) - ~q(Y|P)|} gives:
    \begin{align*}
        \Big|q(\bm{\theta} \mid \mathcal{P}_{\mathrm{CoT}}) - \tilde{q}(\bm{\theta} \mid \mathcal{P}_{\mathrm{CoT}})  \Big| = \Big|\frac{q(\mathcal{P}_{\mathrm{CoT}}\mid \bm{\theta})q(\bm{\theta})}{q(\mathcal{P}_{\mathrm{CoT}})} - \frac{\tilde{q}(\mathcal{P}_{\mathrm{CoT}} \mid \bm{\theta})\tilde{q}(\bm{\theta})}{\tilde{q}(\mathcal{P}_{\mathrm{CoT}})}\Big|.
    \end{align*}
    As established in Section~\ref{section: pretraining}, the joint probability satisfies $q(\bm{t}, \bm{h}) \geq b = \mathcal{O}(1)$ for all history-token pairs. This condition ensures that $q(\mathcal{P}_{\mathrm{CoT}})$ is of order $\mathcal{O}(1)$ since the total prompt length is bounded by the context window $n$. Assuming that $\tilde{q}(\mathcal{P}_{\mathrm{CoT}})$ similarly maintains a strictly positive value for all possible CoT prompts. Following the same notation that their common lower bound as $1/\mathcal{M} > 0$, we have:
    \begin{align}\label{eq: |q(θ|P) - ~q(θ|P)|}
        &\Big|q(\bm{\theta} \mid \mathcal{P}_{\mathrm{CoT}}) - \tilde{q}(\bm{\theta} \mid \mathcal{P}_{\mathrm{CoT}})  \Big| \leq \mathcal{M}\cdot\Big|q(\mathcal{P}_{\mathrm{CoT}}\mid \bm{\theta})q(\bm{\theta}) - \tilde{q}(\mathcal{P}_{\mathrm{CoT}} \mid \bm{\theta})\tilde{q}(\bm{\theta})\Big| \notag\\
        &\leq  \mathcal{M}\Big\{\sum_{\bm{\theta} \in \Theta^L}q(\bm{\theta})\Big|q(\mathcal{P}_{\mathrm{CoT}}\mid \bm{\theta}) - \tilde{q}(\mathcal{P}_{\mathrm{CoT}} \mid \bm{\theta})\Big|+ \sum_{\bm{\theta} \in \Theta^L}\Big|\tilde{q}(\bm{\theta}) - q(\bm{\theta})\Big|\tilde{q}(\mathcal{P}_{\mathrm{CoT}} \mid \bm{\theta})\Big\}.
    \end{align}
    The second summation can be directly bounded as follows:
    \begin{align*}
        \sum_{\bm{\theta} \in \Theta^L}\Big|\tilde{q}(\bm{\theta}) - q(\bm{\theta})\Big|\tilde{q}(\mathcal{P}_{\mathrm{CoT}} \mid \bm{\theta}) \leq \sum_{\bm{\theta} \in \mathscr{F}(\mathcal{P}_{\mathrm{CoT}})}\Big|\tilde{q}(\bm{\theta}) - q(\bm{\theta})\Big| = \Delta_{\mathcal{P}_\mathrm{CoT}},
    \end{align*}
    where the final equality follows from Definition~\ref{def: prior mismatch}. For the first term in~\eqref{eq: |q(θ|P) - ~q(θ|P)|}, we decompose the likelihoods $q(\mathcal{P}_{\mathrm{CoT}}\mid \bm{\theta})$ and $\tilde{q}(\mathcal{P}_{\mathrm{CoT}}\mid \bm{\theta})$ into products over their constituent components:
    \begin{align*}
        &\Big|q(\mathcal{P}_{\mathrm{CoT}}\mid \bm{\theta}) - \tilde{q}(\mathcal{P}_{\mathrm{CoT}} \mid \, \bm{\theta})\Big| \\
        &= \Big|q(\bm{x}^{(i)} \mid \bm{\theta})\prod_{i=1}^mq(\bm{Y}^{(i)} \mid \bm{x}^{(i)}\circ \bm{Y}^{(i)},\, \bm{\theta})  - \tilde{q}(\bm{x}^{(i)} \mid \theta)\prod_{i=1}^m\tilde{q}_l(\bm{Y}^{(i)} \mid \bm{x}^{(i)}\circ \bm{Y}^{(i)},\, \bm{\theta})\Big|\\
        &=\Big| \prod_{k=1}^Lq(\bm{x}^{(i)}_{:, k} \mid \bm{x}_{:, \prec k}, \,\theta_0)\prod_{i=1}^m\prod_{l=1}^Lq(\bm{Y}^{(i)}_{:,l} \mid \bm{x}^{(i)}\circ \bm{Y}^{(i)}_{:,\prec l},\, \theta_l) \\
        &\quad - \prod_{k=1}^L\tilde{q}_0(\bm{x}^{(i)}_{:,k} \mid \bm{x}_{:,\prec k},\,\theta_0)\prod_{i=1}^m\prod_{l=1}^L\tilde{q}_l(\bm{Y}^{(i)}_{:,l} \mid \bm{x}^{(i)}\circ \bm{Y}^{(i)}_{:,\prec l},\, \theta_l)\Big| \\
        &\leq (m+1)L\varphi \leq 2mL\varphi.
    \end{align*}
    The final inequality is derived using the same telescopic product technique established in~\eqref{eq: technique for dealing with product terms}, relying solely on the triangle inequality and the fact that all conditional probabilities are bounded by $1$.

    Finally, combining these results yields:
    \begin{align*}
        |q(\bm{Y} \mid \mathcal{P}_{\mathrm{CoT}}) - \tilde{q}(\bm{Y} \mid \mathcal{P}_{\mathrm{CoT}})| &\leq 2n\phi + 2L\varphi + 2m\mathcal{M}L\varphi + \mathcal{M}\Delta_{\mathcal{P}_{\mathrm{CoT}}} \\
        &\leq 2n\phi + 3m\mathcal{M}(\varphi + \Delta_{\mathcal{P}_{\mathrm{CoT}}}),
    \end{align*}
    which completes the proof.
\end{proof}



\vskip 0.2in
\bibliography{sample}

\end{document}